\documentclass{lmcs}
\pdfoutput=1

\usepackage[UKenglish]{babel}
\usepackage[dvipsnames]{xcolor}
\usepackage{tikz}
\usetikzlibrary{arrows.meta,automata,positioning,shadows,trees}
\usepackage{bbm}
\usepackage{amssymb}
\usepackage[shortlabels]{enumitem}
\usepackage{thmtools}
\usepackage{subcaption}
\input{macros}

\begin{document}

\title[Approximating Reachability in Posterior-Deterministic POMDPs]{Approximating the Reachability Value of Posterior-Deterministic POMDPs}

\author[N. Fijalkow]{Nathana\"el Fijalkow\lmcsorcid{0000-0002-6576-4680}}[a]
\author[A. Ghosh]{Arka Ghosh\lmcsorcid{0000-0003-3839-8459}}[a]
\author[R. Kniazev]{Roman Kniazev}[a]
\author[G. A. P\'erez]{Guillermo A. P\'erez\lmcsorcid{0000-0002-1200-4952}}[b]
\author[P. Vandenhove]{Pierre Vandenhove\lmcsorcid{0000-0001-5834-1068}}[c]

\address{LaBRI, CNRS, University of Bordeaux, France}
\email{nathanael.fijalkow@gmail.com, a.ghosh@uw.edu.pl, roman@knzv.me}

\address{University of Antwerp -- Flanders Make, Belgium}
\email{guillermo.perez@uantwerpen.be}

\address{UMONS -- Universite de Mons, Belgium}
\email{pierre.vandenhove@umons.ac.be}

\thanks{Arka Ghosh was supported by the Polish National Science Centre (NCN) grant "Linear algebra in orbit-finite dimension" (2022/45/N/ST6/03242) and the SAIF project, funded by the "France 2030" government investment plan managed by the French National Research Agency, under the reference ANR-23-PEIA-0006.}
\thanks{Guillermo A. P\'erez was supported by the Belgian FWO "SynthEx" (G0AH524N) project.}
\thanks{Pierre Vandenhove was supported by the Belgian Fonds de la Recherche Scientifique -- FNRS through the CDR funding "INSTAP" n$^\circ$ J.0228.26.}


\begin{abstract}
Partially observable Markov decision processes (POMDPs) are a standard model for sequential decision-making under uncertainty. Many verification and synthesis problems for POMDPs are undecidable or intractable. In particular, Madani et al. (2003) showed that no algorithm can, given a POMDP and a set of target states, compute the maximal probability of reaching the target states, or even approximate it up to a non-trivial constant. For fully observable Markov decision processes (MDPs), in contrast, this reachability value can be computed in polynomial time.

We introduce posterior-deterministic POMDPs, and we show that for this class of POMDPs, the maximal probability of reaching a given set of states can be approximated up to arbitrary precision.

A POMDP is posterior-deterministic if the next state can be uniquely determined by the current state, the action taken, and the observation received. In particular, once the current state is known, it remains known forever. The class contains, for instance, multiple-environment MDPs equipped with a prior distribution over the environments, quasi-deterministic POMDPs, and the episodic version of the classical Tiger POMDP (Kaelbling et al. 1998).
\end{abstract}

\maketitle

\section{Introduction}

Partially observable Markov decision processes (POMDPs) are the standard model
for sequential decision-making under uncertainty.  They were introduced by
Åström~\cite{Astrom1965}, studied algorithmically by Smallwood and
Sondik~\cite{SmallwoodSondik1973}, and became widely used in artificial
intelligence after the work of Kaelbling, Littman, and
Cassandra~\cite{KaelblingLittmanCassandra1998}.  A POMDP models a
decision-making agent that must act in a stochastic environment whose true
state is hidden; the agent receives only noisy observations and must maintain
a \emph{belief} (a probability distribution over possible states) which they
update as new observations arrive.

POMDPs are used whenever an agent has to act on incomplete information.  In
\emph{robotics}, POMDPs model navigation under sensor noise, manipulation with
uncertain object poses, and collision avoidance for autonomous
vehicles~\cite{LauriHsuPajarinen2022,HsiaoKaelblingLozanoPerez2007,TemizerKochenderfer2010}.
In \emph{healthcare}, they are used for treatment planning when diagnoses are
uncertain and patient states evolve
stochastically~\cite{HauskrechtFraser2000,SchaeferBaileyShechter2005}.
\emph{Dialogue systems} use POMDPs to model the inference of user intent from
noisy speech recognition~\cite{WilliamsYoung2007,YoungGasicThomson2013}.  In
\emph{verification and controller synthesis}, POMDPs arise when
synthesizing controllers for systems with imperfect sensors or hidden
adversarial behavior~\cite{ArnoldVincentWalukiewicz2003,ChatterjeeDoyen2013}.
In all these domains, the problem is to compute optimal or near-optimal
policies, although the set of beliefs is infinite.
In most of the applications above, the objective is a discounted or a finite-horizon sum of rewards.
As we recall below, such objectives behave very differently from the undiscounted objectives used in verification and controller synthesis, such as maximising the probability of eventually reaching a set of target states.
This paper is concerned with the latter objective.

POMDPs are, however, hard to analyse.
The set of beliefs is a continuous simplex, and an optimal policy maps this infinite set to actions; moreover, the value function, while piecewise-linear and convex for finite-horizon problems~\cite{SmallwoodSondik1973}, can have exponentially many pieces.

The complexity of POMDPs depends on the objective.
As a point of comparison, in (fully observable) Markov decision processes (MDPs), the maximal probability of reaching a set of states can be computed in polynomial time, for instance by linear programming~\cite{CY95,BK08}.
For finite-horizon POMDPs, even
determining whether there is a strategy that achieves a given sum of rewards is
PSPACE-complete~\cite{PapadimitriouTsitsiklis1987}.
For discounted objectives, the value of a POMDP can be approximated up to any precision by value iteration, since the contribution of the rewards collected after $n$ steps decreases geometrically with~$n$; the difficulty is then scalability.
For infinite-horizon undiscounted
objectives, the situation is much worse: Madani, Hanks, and
Condon~\cite{DBLP:journals/ai/MadaniHC03} established that the maximal
probability of reaching a given set of states (the reachability value)
cannot be approximated, even up to a fixed error
of~$\frac{1}{3}$.  More precisely, it follows from Condon and
Lipton~\cite{CL89} that deciding whether the reachability value exceeds
$\frac{2}{3}$ or is below $\frac{1}{3}$ is undecidable.  Almost-sure
reachability (determining whether a strategy reaches the target with
probability~$1$) is decidable~\cite{BGB12}, but this positive result does not
extend to approximating or computing the exact value.

This raises the question of which natural classes of POMDPs admit an algorithm that approximates the reachability value.

Several subclasses of POMDPs have been studied; they differ in where the uncertainty about the current state comes from.

\emph{Deterministic POMDPs}~\cite{LittmanThesis,Bon09} require transitions and observations to be deterministic functions of the current state and action, so that all the uncertainty comes from the initial belief.
Only finitely many beliefs are then reachable from the initial belief, hence the reachability value can be computed by analysing a finite (exponentially large) MDP.
Deterministic POMDPs are a probabilistic counterpart of \emph{contingent planning} with deterministic actions and sensing and an uncertain initial state (as in, e.g.,~\cite{APG09}; contingent planning also allows non-deterministic action effects, see, e.g.,~\cite{HB05}), and, in the absence of observations, of \emph{conformant planning}.
\emph{Quasi-deterministic POMDPs}~\cite{BC09} keep deterministic transitions but allow stochastic observations, so that infinitely many beliefs may be reachable.

\emph{Multiple-environment MDPs} (MEMDPs)~\cite{RS14} consist of finitely many MDPs, called environments, over the same set of states: the agent observes the current state, but does not know which environment it interacts with.
Most works on MEMDPs consider the worst case over the environments, for instance asking for a strategy that is good in all environments simultaneously~\cite{RS14,VJJ23,CDRS25}.
If the environments are instead equipped with a prior probability distribution and the objective is the expected value, then an MEMDP is a POMDP whose hidden component (the environment) never changes.
This point of view is taken by Chatterjee et al.~\cite{CCKNR20}, who exploit it to speed up belief updates for discounted objectives, and who show that the entropy of the belief over the environments does not increase in expectation.
Such models also appear as hidden-model MDPs, a subclass of mixed-observability MDPs~\cite{OPHL10} used for adaptive management in conservation biology~\cite{CCMNSB12}, and as multi-model MDPs, proposed for medical decision-making under ambiguity about the parameters of the model~\cite{SKD21} (Steimle et al.\ consider Markov policies; with history-dependent policies, multi-model MDPs are MEMDPs with a prior).
Raskin and Sankur~\cite{RS14}, who focus on two environments, observed that the environment becomes known as soon as a transition that exists in only one environment is taken.
They introduced \emph{double end components}, which are end components in both environments, and called them distinguishing if some action inside them has different distributions in the two environments.
Inside a distinguishing double end component, the agent can learn the environment with arbitrarily high confidence by sampling such actions; a non-distinguishing one can be replaced by a gadget in which only the ways of leaving it remain.
Our analysis of support end components (see below) generalises this dichotomy from MEMDPs to posterior-deterministic POMDPs.

\emph{MDPs with delayed observations}~\cite{KE03}, in which the state is observed with a constant delay, reduce to fully observable MDPs over an augmented state space, and are therefore easier from the point of view of value approximation.
Their natural encoding as POMDPs is in general not posterior-deterministic (in the sense defined below), since the observation received concerns a past state and does not determine the new state.
A one-step delay also underlies the \emph{fast informed bound}~\cite{Hau00}, a classical upper bound on the value of a POMDP, which can be interpreted as the value obtained when the previous state is revealed to the agent at each step.

\emph{Revealing POMDPs}~\cite{BFGHPV25} require that the agent almost surely eventually gets full information about the current state; for these, almost-sure satisfaction of $\omega$-regular objectives is decidable, by reduction to a finite belief-support MDP.
For \emph{$\omega$-regular objectives}, the decidable and undecidable fragments of the qualitative problems for POMDPs are described in detail in~\cite{DBLP:journals/jcss/ChatterjeeCT16}.
For safety objectives (the dual of reachability), the value can be approximated in all POMDPs~\cite{BBBFV25}, whereas this is impossible for reachability.

\emph{Point-based value iteration}~\cite{PineauGordonThrun2003} and \emph{heuristic search value iteration} (HSVI)~\cite{SmithSimmons04} are among the most successful practical algorithms for POMDPs, and they come with guarantees: the error of point-based value iteration is bounded in terms of the density of the set of sampled beliefs, and HSVI maintains lower and upper bounds on the value, whose difference is guaranteed to become arbitrarily small.
These guarantees are established for discounted objectives, and rely on the contraction induced by discounting.
Heuristic search algorithms have also been designed for undiscounted goal POMDPs~\cite{BG09,HBC18}, in which the aim is to minimise the expected cost of reaching a target under positive costs.
Goal-HSVI~\cite{HBC18} comes with convergence guarantees for this expected-cost objective, which is however different from the probability of reaching the target.
In verification, the probabilistic model checkers PRISM and Storm compute sound lower and upper bounds on undiscounted reachability probabilities in POMDPs, using finite abstractions of the belief MDP obtained by discretising the belief space and by cutting off its exploration~\cite{NPZ17,BJKQ20,BKQ22}; in general, there is no guarantee that these bounds converge to the value.
None of these algorithms is a decision procedure for approximating the undiscounted reachability value, and by the result of Madani et al.\ no algorithm can approximate it for all POMDPs.
Our algorithm is close in spirit to these approaches: it explores the belief MDP forward from the initial belief and maintains a lower and an upper bound on the value, and its cut operation, which removes small probabilities from beliefs, is reminiscent of the belief clipping of~\cite{BKQ22}; its specific unfolding rules (split and exit operations on support end components) guarantee that the two bounds converge to the value for posterior-deterministic POMDPs.

In this paper, we introduce \emph{posterior-deterministic POMDPs}, a class that
strictly generalises quasi-deterministic and deterministic POMDPs.
A POMDP is posterior-deterministic if, when the current state is known, the successor state after any action-observation pair is uniquely determined.
Equivalently, the posterior distribution over states, conditioned on knowing the current state and observing $(a, o)$, is always a point mass.
As a consequence, the belief support (the set of states with positive probability at a given moment) never increases in size.

This property distinguishes posterior-deterministic POMDPs from general
POMDPs, where a single observation can ``spread'' probability mass to new
states.  It also clarifies the relationship to quasi-deterministic POMDPs:
posterior-determinism is a strictly weaker condition. Quasi-deterministic POMDPs
require that successor states be determined by the current state and the
action taken, yet they allow for stochastic observations.
Posterior-deterministic POMDPs also allow for a stochastic transition
function, requiring only that each observation resolves
the successor state uniquely given the current state and the action taken.
The class contains the following models, in all of which infinitely many beliefs may be reachable from the initial belief.
This is what we mean when we say that the class exhibits \emph{genuine} partial observability, in contrast with deterministic POMDPs, for which the analysis reduces to a finite MDP.
\begin{itemize}
    \item MEMDPs equipped with a prior distribution over the environments, viewed as POMDPs as described above: given the current state (environment and visible state), the action, and the observation (the next visible state), the next state is determined, since the environment does not change.
      More generally, a POMDP whose states are pairs made of a visible and a hidden component is posterior-deterministic as soon as the observation determines the next visible component and the next hidden component is a function of the current state, the action, and the observation (for instance, when the hidden component is static).
    \item All quasi-deterministic POMDPs, hence all deterministic POMDPs.
    \item The \emph{Tiger POMDP}~\cite{KaelblingLittmanCassandra1998}, a standard (if simplistic) benchmark, in its episodic version where opening a door ends the game (the natural version for a reachability objective): the hidden position of the tiger does not change while the agent listens.
\end{itemize}
The class also contains all MDPs, but this is not a meaningful motivation in itself, since MDPs are much simpler to analyse.

Our main result is that posterior-deterministic POMDPs admit decidable value approximation for reachability objectives.

\begin{quote}
\textbf{Main theorem.} For any posterior-deterministic POMDP $\pdp$, initial
belief $\bel{b}$, and a tolerance $\epsilon > 0$, one can compute a value $v \in [0,1]$ such that~$|\preach^\pdp(\bel{b}) - v| \leq \epsilon$.
\end{quote}
The algorithm unfolds the belief MDP into a tree, and the proof that it is correct and terminates uses an extension of \emph{end components} from MDPs to POMDPs and martingale arguments.
In MDPs, almost every run eventually stays in an end component, and the analysis of reachability and $\omega$-regular objectives relies on this fact.
For general POMDPs, it is well understood that undecidability, even for qualitative $\omega$-regular problems, is closely related to the impossibility, for the agent, of identifying the end component in which the run eventually stays~\cite{BGB12,DBLP:journals/jcss/ChatterjeeCT16}.
Our \emph{support end components} (SECs) are the end components of the belief-support MDP, and we use them to address this issue.
We show that, in posterior-deterministic POMDPs, an agent staying long enough in an SEC can learn, with arbitrarily high confidence, the class of the current state for an indistinguishability relation (distinguishing SECs), and that otherwise only finitely many beliefs are reachable inside the SEC (non-distinguishing SECs).
The fact that information about the hidden part of the state is never lost, which is a direct consequence of posterior-determinism, plays the same role as in the analysis of MEMDPs, and the distinction between distinguishing and non-distinguishing SECs generalises the double end components of Raskin and Sankur~\cite{RS14}.
The new difficulties are that SECs are made of belief supports rather than states, that infinitely many beliefs may be reachable inside an SEC, that a belief support may split into several classes of indistinguishable states (rather than the environment being either learned or not), and that the objective is the expected value rather than the worst case over environments.

\noindent

Definitions are given in \Cref{sec:prelim}, including background on end components of MDPs and on the belief MDP.
\Cref{sec:deterministicPOMDPs} introduces posterior-deterministic POMDPs and the belief-support MDP.
\Cref{sec:descriptionAlgo} describes our algorithm, a tree unfolding of the belief MDP, and explains its rules on examples (most proofs are deferred).
\Cref{sec:correctness} proves that the algorithm is correct and terminates, and \Cref{sec:ec} studies support end components.

\section{Preliminaries} \label{sec:prelim}
If $X$ is a finite set, we denote by \emphdef{$\Dist(X)$} the set of all probability distributions over $X$.
The \emphdef{support} of a distribution $\mu\in\Dist(X)$ is the set $\supp{\mu} = \setof{x\in X}{\mu(x) > 0}$.

\subparagraph{POMDPs.}
A \emphdef{partially observable Markov decision process} (POMDP) is a tuple $\pdp = \pdpFull$ where~$\states$ is a finite set of \emphdef{states}, $\actions$ is a finite set of \emphdef{actions}, $\sig$ is a finite set of \emphdef{observations}, and $\ptran\colon \states \times \actions \to \Dist(\sig \times \states)$ is a \emphdef{transition function}.
We frequently write $\ptran(\obs, \state' \mid \state, \action)$ to denote the probability of receiving observation $\obs$ and going to state $\state'$ when taking action $\action$ in state $\state$ (formally, $\ptran(\obs, \state' \mid \state, \action) = \ptran(\state, \action)(\obs, \state')$).
We also write $\ptran(\state' \mid \state, \action)$ to denote the marginal probability of going to state $\state'$ when taking action $\action$ in state $\state$ (formally, $\ptran(\state' \mid \state, \action) = \sum_{\obs\in\sig} \ptran(\obs, \state' \mid \state, \action)$) and, similarly, $\ptran(\obs \mid \state, \action)$ to denote the marginal probability of receiving observation $\obs$ when taking action $\action$ in state $\state$ (formally, $\ptran(\obs \mid \state, \action) = \sum_{\state'\in\states} \ptran(\obs, \state' \mid \state, \action)$).

An agent interacting with a POMDP takes actions sequentially, which updates the state probabilistically based on function~$\ptran$.
Yet, the agent only observes the observations, not the actual states (and of course, the agent knows which actions they have taken).
In particular, they do not know in general which state the POMDP is in at a given moment, which is why it is called \emph{partially observable}.

Let $\pdp = \pdpFull$ be a POMDP.
A \emphdef{run} of $\pdp$ is an infinite sequence of states, actions, and observations of the form $\run = \state_0 \action_1 \obs_1 \state_1 \action_2 \obs_2 \ldots \in (\states \times \ac \times \sig)^\omega$ such that $\ptran(\state_i, \action_{i+1})(\obs_{i+1}, \state_{i+1}) > 0$ for all $i\ge 0$.
We write $\Runs(\pdp)$ for the set of runs on $\pdp$.

A \emphdef{history} is a finite prefix of a run ending in a state; it is an element of $\states \times (\ac \times \sig \times \states)^*$.
We write \emphdef{$\last(\history)$} for the last state of the history.
We also define an \emphdef{observable history} (resp.\ \emphdef{observable run}) as the projection of a history (resp.\ run) to $(\actions \times \sig)^*$ (resp.\ $(\actions \times \sig)^\omega$), corresponding to the information an agent has access to while taking actions in the POMDP.
We write $\projObs(\cdot)$ for this projection.
The \emphdef{length} $\length{\history}$ of an (observable) history $\history$ is the number of
actions in it.
If $\run$ is an (observable) run, we write $\partialRun{n}$ for its prefix of length $n$, i.e., the (observable) history containing the first $n$ action-observation pairs.

\subparagraph{Beliefs and belief supports.}
A \emphdef{belief} of $\pdp$ is a distribution $\bel{b}\in\Dist(\states)$.
Beliefs are denoted as boldface letters.
The set of all beliefs of $\pdp$ is denoted as \emphdef{$\bsp{\pdp}$}.
Beliefs are used to represent the most accurate information about the current state that the agent has.
We usually assume that the agent has an \emphdef{initial belief} about its starting position in the POMDP.
This belief is updated as actions are taken and observations are received.
If $\bel{b}$ is the current belief, after taking action $\action$ and observing $\obs$, the new belief $\bel{b}_{\action, \obs}$ is given, for $\state\in\states$, by
\begin{equation}
    \bel{b}_{\action, \obs}(\state) = \frac{\sum_{\state'\in\states} \bel{b}(\state') \cdot \ptran(\obs, \state \mid \state', \action)}{\sum_{\state''\in\states} \sum_{\state'\in\states} \bel{b}(\state') \cdot \ptran(\obs, \state'' \mid \state', \action)} \ . \label{eq:belief update}
\end{equation}
In matrix notation, we have \(\bel{b}_{\action,\obs} = \frac{\bel{b}\cdot\ptran_{\action,\obs}}{\lVert\bel{b}\cdot\ptran_{\action,\obs}\rVert_1}\), where \(\ptran_{\action,\obs}\) is the transition matrix defined as \(\ptran_{\action,\obs}(\state,\state') = \ptran(\state,\action)(\obs,\state')\).

The \emphdef{support} of a belief $\bel{b}$ is the set $\supp{\bel{b}}$.
We write \emphdef{$\bsup{\pdp}$} for the set of all belief supports of $\pdp$ (corresponding to the set $2^{\states}\setminus \{\emptyset\}$).
We will sometimes write a belief~$\bel{b}$ as
$
\sum_{\state\in\states} \bel{b}(\state)\cdot \state
$, where we identify a state $\state$ with the Dirac belief $\dirac{\state}$ assigning probability~$1$ to~$\state$; that is, $\sum_{\state\in\states} \bel{b}(\state)\cdot \state$ stands for $\sum_{\state\in\states} \bel{b}(\state)\cdot \dirac{\state}$.
The Dirac belief $\dirac{\state}$ should not be confused with the transition function~$\transition$ of \Cref{sec:deterministicPOMDPs}, which always has two or three arguments.

It will sometimes be convenient to restrict the support of a belief. This
gives rise to \emphdef{sub-beliefs}, \label{page:subbel}%
functions $\bel{b}$ from $\states$ to $[0,1]$ such that
\(
\sum_{\state\in\states} \bel{b}(\state) \leq 1
\).
We write $\oneNorm{\bel{b}} = \sum_{\state\in\states} \bel{b}(\state)$ for the \emphdef{mass} of a sub-belief~$\bel{b}$, so that beliefs are exactly the sub-beliefs of mass~$1$, and, when $\bel{b}$ is non-zero, $\normalised{\bel{b}} = \frac{1}{\oneNorm{\bel{b}}}\cdot\bel{b}$ for the \emphdef{normalised} belief.
For a belief $\bel{b}$ and subset of states $\bsupp\subseteq \states$,
the restriction \emphdef{$\res{\bel{b}}{\bsupp}$} of $\bel{b}$ to $\bsupp$ is
the sub-belief defined so that
\[
\res{\bel{b}}{\bsupp}(q) =
\begin{cases}
\bel{b}(\state) & \text{if $\state\in \bsupp$} \\
0 & \text{otherwise.}
\end{cases}
\]
We frequently extend functions defined on states to beliefs in the natural way.
For instance, we write
$
\ptran(\obs \mid \bel{b},\action)
= \sum_{\state\in\supp{\bel{b}}} \ptran(\obs \mid \state,\action)\cdot\bel{b}(\state)
$
for the probability of observing~$\obs$ after taking action $\action$ from belief $\bel{b}$.
The belief update~\eqref{eq:belief update} is defined only when $\ptran(\obs \mid \bel{b},\action) > 0$, i.e., when observation~$\obs$ is \emphdef{possible} after action~$\action$ from~$\bel{b}$.

\noindent\textit{Belief update of sub-beliefs.}
We extend the belief update to non-zero sub-beliefs so that it \emph{preserves the mass}: for a non-zero sub-belief~$\bel{b}$, an action~$\action$, and an observation~$\obs$ such that $\ptran(\obs \mid \bel{b},\action) > 0$, we let
\[
\bel{b}_{\action,\obs} = \oneNorm{\bel{b}}\cdot \left(\normalised{\bel{b}}\right)_{\action,\obs}\ .
\]
This is consistent with~\eqref{eq:belief update} when $\bel{b}$ is a belief, and $\oneNorm{\bel{b}_{\action,\obs}} = \oneNorm{\bel{b}}$.
Moreover, for every $\lambda > 0$ such that $\oneNorm{\lambda\bel{b}} \le 1$, we have $(\lambda\bel{b})_{\action,\obs} = \lambda\cdot\bel{b}_{\action,\obs}$ and $\ptran(\obs\mid\normalised{\lambda\bel{b}},\action) = \ptran(\obs\mid\normalised{\bel{b}},\action)$, since $\normalised{\lambda\bel{b}} = \normalised{\bel{b}}$.
We use this convention throughout the paper; in particular, the mass of a sub-belief is only changed by restrictions and by the cut operation of \Cref{sec:descriptionAlgo}, and never by belief updates (the split operation of \Cref{sec:descriptionAlgo} restricts a sub-belief and then rescales it, so that the mass is preserved).
For an observable history $\history = \action_1\obs_1\cdots\action_n\obs_n$, we define $\bel{b}_\history$ inductively by $\bel{b}_{\emptyhist} = \bel{b}$ (where $\emptyhist$ denotes the empty history) and $\bel{b}_{\history'\action\obs} = (\bel{b}_{\history'})_{\action,\obs}$.
We say that $\history$ is \emphdef{possible} from~$\bel{b}$ if $\bel{b}_\history$ is defined, i.e., if $\ptran(\obs_k \mid \bel{b}_{\action_1\obs_1\cdots\action_{k-1}\obs_{k-1}}, \action_k) > 0$ for all $k\in\{1,\dots,n\}$.

\subparagraph{Strategies.}
A \emphdef{strategy} in $\pdp$ is a function $\strategy\colon (\actions \times \sig)^* \to \Dist(\actions)$ that sequentially takes an action (possibly in a randomised fashion) given the current \emph{observable} history.
Such strategies are sometimes called \emph{observation-based} in the literature because they only look at the sequence of observations (as well as the actions they have previously taken) to make their decisions.
In this paper, all strategies are assumed to be observation-based.
Strategies are defined for the POMDP itself; the initial belief is supplied separately when evaluating probabilities and values.

\subparagraph{Probability measure on runs.}
For a history $\history$, the \emphdef{cylinder $\cyl{\history}$} is the set of all runs that start with $\history$.
Let \emphdef{$\filtr{}$} be the \sigal{} on the set $\Runs(\pdp)$ generated by the set $\setof{\cyl{\history}}{\history\in \states\times (\actions\times \sig \times \states)^*}$.
For an initial (sub-)belief $\bel{b}$ and a strategy $\strategy$ in $\pdp$, we define a function \emphdef{$\pdist{\bel{b}}{\strategy}$} such that for $\history\in\states\times (\actions\times \sig \times \states)^*$, $\pdist{\bel{b}}{\strategy}(\cyl{\history})$ is the probability of getting the history~$\history$ starting from~$\bel{b}$ using strategy $\strategy$.
We give a formal inductive definition: if $\history = \state$ contains a single state, then $\pdist{\bel{b}}{\strategy}(\cyl{\history}) = \bel{b}(\state)$.
If $\history = \history' \action \obs \state$, then
\[
\pdist{\bel{b}}{\strategy}(\cyl{\history}) = \pdist{\bel{b}}{\strategy}(\cyl{\history'}) \cdot \strategy(\projObs(\history'))(\action) \cdot \ptran(\last(\history'), \action)(\obs, \state) \ .
\]
 
If $\bel{b}$ is a belief, the function $\pdist{\bel{b}}{\strategy}$ %
can be uniquely extended to a probability distribution on the full measurable space $(\Runs(\pdp), \filtr{})$ using Ionescu-Tulcea extension theorem~\cite[Section~14.3]{Kle20}.
For a non-zero sub-belief~$\bel{b}$, the same inductive definition yields $\pdist{\bel{b}}{\strategy} = \oneNorm{\bel{b}}\cdot\pdist{\normalised{\bel{b}}}{\strategy}$ (and $\pdist{\bel{b}}{\strategy} = 0$ if $\bel{b} = 0$): it is a measure of total mass $\oneNorm{\bel{b}}$, which is a probability measure only when $\bel{b}$ is a belief.
By construction, $\pdist{\bel{b}}{\strategy} = \sum_{\state\in\states}\bel{b}(\state)\cdot\pdist{\dirac{\state}}{\strategy}$ for every sub-belief~$\bel{b}$.
We write \emphdef{$\expect{\bel{b}}{\strategy}{\cdot}$} to denote the expectation with respect to the probability distribution $\pdist{\bel{b}}{\strategy}$.

\noindent\textit{Observable cylinders.}
For an observable history $\history\in(\actions\times\sig)^*$, we also write $\cyl{\history}$ for the set of runs whose projection to $(\ac \times \sig)^\omega$ starts with~$\history$.
Let $\bel{b}$ be a belief, $\strategy$ be a strategy, and $\history = \action_1\obs_1\cdots\action_n\obs_n$ be an observable history; we write $\history_{<k} = \action_1\obs_1\cdots\action_{k-1}\obs_{k-1}$.
If $\history$ is possible from $\bel{b}$, then
\begin{equation}\label{eq:prob obs hist}
\pdist{\bel{b}}{\strategy}(\cyl{\history}) = \prod_{k=1}^{n} \strategy(\history_{<k})(\action_k)\cdot\ptran(\obs_k\mid\bel{b}_{\history_{<k}},\action_k)\ ,
\end{equation}
and otherwise $\pdist{\bel{b}}{\strategy}(\cyl{\history}) = 0$.
Moreover, if $\history$ is possible from $\bel{b}$, then for every measurable set $E\subseteq\Runs(\pdp)$,
\begin{equation}\label{eq:shift}
\pdist{\bel{b}}{\strategy}(\cyl{\history}\cap\sigma^{-n}(E)) = \pdist{\bel{b}}{\strategy}(\cyl{\history})\cdot\pdist{\bel{b}_\history}{\strategy[\history]}(E)\ ,
\end{equation}
where $\sigma^n(\state_0\action_1\obs_1\state_1\action_2\obs_2\ldots) = \state_n\action_{n+1}\obs_{n+1}\state_{n+1}\ldots$ removes the first $n$ steps of a run, and $\strategy[\history]$ is the \emphdef{shifted} strategy defined by $\strategy[\history](\history') = \strategy(\history\history')$.
Both identities are standard consequences of the definitions of $\pdist{\bel{b}}{\strategy}$ and of the belief update, and are proved by induction on~$n$ (for~\eqref{eq:shift}, it suffices to consider cylinders~$E$, since both sides are measures in~$E$).
Identity~\eqref{eq:shift} states that, conditionally on the observable history~$\history$, the run continues as a run from belief~$\bel{b}_\history$ under strategy~$\strategy[\history]$; in particular, the belief $\bel{b}_\history$ is the conditional distribution of the current state given~$\cyl{\history}$.

\subparagraph{Value of a POMDP.}
A \emphdef{reachability objective} on $\pdp$ is specified by a set of target states $\targetSet\subseteq \states$.
Formally, it is the measurable set
$\reach{\targetSet} = \setof{\state_0 \action_1 \obs_1 \state_1 \action_2 \obs_2 \ldots \in \Runs(\pdp)}{\exists i\ge 0, \state_i \in \targetSet}$.
Given an initial (sub-)belief $\bel{b}$, the \emphdef{(reachability) value} of $\pdp$ with respect to
$\targetSet$ is
\[\preach^\pdp_\targetSet(\bel{b}) = \sup_{\strategy} \pdist{\bel{b}}{\strategy}(\reach{\targetSet})
\]
where the supremum is taken over all (observation-based) strategies $\strategy$ in $\pdp$.
When $\pdp$ and $\targetSet$ are clear from the context, we simply write
$\preach(\bel{b})$.
Observe that $\preach(\bel{b})$ is upper-bounded by $\oneNorm{\bel{b}} = \sum_{\state\in\states} \bel{b}(\state)$, which gives a better upper bound than $1$ when $\bel{b}$ is a sub-belief.
The following properties allow us to transfer results from beliefs to sub-beliefs.
\begin{lemma}\label{lem:subbelief val}
For all sub-beliefs $\bel{b}, \bel{b}'$ and all $\lambda\ge 0$ such that $\oneNorm{\lambda\bel{b}}\le 1$:
\begin{enumerate}
\item $\preach(\lambda\bel{b}) = \lambda\cdot\preach(\bel{b})$; in particular, $\preach(\bel{b}) = \oneNorm{\bel{b}}\cdot\preach(\normalised{\bel{b}})$ if $\bel{b}\neq 0$; \label{item:val homogeneous}
\item if $\bel{b}(\state)\le\bel{b}'(\state)$ for all $\state\in\states$, then $\preach(\bel{b})\le\preach(\bel{b}')$; \label{item:val monotone}
\item $|\preach(\bel{b}) - \preach(\bel{b}')| \le \oneNorm{\bel{b}-\bel{b}'}$. \label{item:val lipschitz}
\end{enumerate}
\end{lemma}
\begin{proof}
All three items follow from the identity $\pdist{\bel{b}}{\strategy} = \sum_{\state\in\states}\bel{b}(\state)\cdot\pdist{\dirac{\state}}{\strategy}$.
For~\eqref{item:val homogeneous}, $\pdist{\lambda\bel{b}}{\strategy} = \lambda\cdot\pdist{\bel{b}}{\strategy}$ for every strategy~$\strategy$, and we take the supremum over~$\strategy$.
For~\eqref{item:val monotone}, $\pdist{\bel{b}}{\strategy}(\reach{\targetSet})\le\pdist{\bel{b}'}{\strategy}(\reach{\targetSet})$ for every strategy~$\strategy$.
For~\eqref{item:val lipschitz}, for every strategy~$\strategy$,
\[
\pdist{\bel{b}}{\strategy}(\reach{\targetSet}) - \pdist{\bel{b}'}{\strategy}(\reach{\targetSet})
\le \sum_{\state\in\states}|\bel{b}(\state)-\bel{b}'(\state)|\cdot\pdist{\dirac{\state}}{\strategy}(\reach{\targetSet})
\le \oneNorm{\bel{b}-\bel{b}'}\ ,
\]
hence $\pdist{\bel{b}}{\strategy}(\reach{\targetSet})\le \preach(\bel{b}') + \oneNorm{\bel{b}-\bel{b}'}$.
Taking the supremum over~$\strategy$ and exchanging the roles of $\bel{b}$ and $\bel{b}'$ yields the claim.
\end{proof}
A strategy $\strategy$ is called \emphdef{$\epsilon$-optimal} from belief $\bel{b}$ if
\[
\pdist{\bel{b}}{\strategy}(\reach{\targetSet}) \geq \preach(\bel{b}) - \epsilon \ .
\]

We also define the value of a non-zero (sub-)belief $\bel{b}$ and an action $\action$ as%
\begin{equation}\label{eq:val bel act}
\preach(\bel{b},\action) =
\sum_{\obs\in\sig,\ \ptran(\obs \mid \bel{b},\action)>0}
\ptran(\obs \mid \normalised{\bel{b}},\action)\cdot \preach(\bel{b}_{\action,\obs}) \ , 
\end{equation}
where $\bel{b}_{\action,\obs}$ is 
obtained from $\bel{b}$ after taking action $\action$ and observing $\obs$
(see \Cref{eq:belief update} and the mass-preserving convention above).
For a belief~$\bel{b}$, we have $\normalised{\bel{b}} = \bel{b}$; for a sub-belief, \Cref{lem:subbelief val} gives $\preach(\bel{b},\action) = \oneNorm{\bel{b}}\cdot\preach(\normalised{\bel{b}},\action)$.

Assume that all target states are absorbing, i.e., that $\ptran(\state\mid\state,\action) = 1$ for all $\state\in\targetSet$ and $\action\in\actions$ (this will be the case under the assumptions made below).
Then, for every non-zero sub-belief~$\bel{b}$,
\begin{equation} \label{eq:val bel upd}
\preach(\bel{b}) =
\max_{\action\in\ac}\
\preach(\bel{b},\action) \ .
\end{equation}
By \Cref{lem:subbelief val}, it suffices to prove~\eqref{eq:val bel upd} for beliefs.
Since target states are absorbing, a run visits $\targetSet$ if and only if it visits $\targetSet$ at some step $i\ge 1$, i.e., $\reach{\targetSet} = \sigma^{-1}(\reach{\targetSet})$.
Hence, by~\eqref{eq:prob obs hist} and~\eqref{eq:shift} with $n = 1$, for every strategy~$\strategy$,
\[
\pdist{\bel{b}}{\strategy}(\reach{\targetSet}) = \sum_{\action\in\actions}\strategy(\emptyhist)(\action)\sum_{\obs\in\sig,\ \ptran(\obs \mid \bel{b},\action)>0}\ptran(\obs\mid\bel{b},\action)\cdot\pdist{\bel{b}_{\action,\obs}}{\strategy[\action\obs]}(\reach{\targetSet}) \le \max_{\action\in\actions}\preach(\bel{b},\action)\ .
\]
Conversely, for every action~$\action$ and $\eta > 0$, the strategy that first plays~$\action$ and then, after observing~$\obs$, follows an $\eta$-optimal strategy from $\bel{b}_{\action,\obs}$, ensures a probability at least $\preach(\bel{b},\action) - \eta$ of reaching~$\targetSet$.
The assumption that target states are absorbing cannot be dropped: if $\state\in\targetSet$ and every action leads from~$\state$ to an absorbing non-target state, then $\preach(\dirac{\state}) = 1$, whereas $\preach(\dirac{\state},\action) = 0$ for every action~$\action$.

Iterating~\eqref{eq:val bel upd}, we obtain that, for every belief $\bel{b}$, strategy~$\strategy$, and $n\in\N$,
\begin{equation}\label{eq:val iterated}
\preach(\bel{b}) \ge \sum_{\history\in(\actions\times\sig)^n \text{ possible from } \bel{b}} \pdist{\bel{b}}{\strategy}(\cyl{\history})\cdot\preach(\bel{b}_\history)\ .
\end{equation}
Indeed, \eqref{eq:val iterated} is trivial for $n = 0$, and, assuming it for~$n$ (for all beliefs and strategies), we have
$\preach(\bel{b}) \ge \sum_{\action}\strategy(\emptyhist)(\action)\cdot\preach(\bel{b},\action) = \sum_{\action,\obs}\strategy(\emptyhist)(\action)\cdot\ptran(\obs\mid\bel{b},\action)\cdot\preach(\bel{b}_{\action,\obs})$,
and we conclude by applying the induction hypothesis to $\bel{b}_{\action,\obs}$ and $\strategy[\action\obs]$, and by~\eqref{eq:prob obs hist}.

\subparagraph{Assumptions.}
A state is \emphdef{absorbing} if, whenever it is reached, the agent remains in that state forever, no matter what actions are taken.
A state is \emphdef{observable} if, whenever it is reached, the agent always
receives a specific observation that uniquely identifies the state, thus
collapsing the current belief to a singleton set.
Formally, a state~$p$ is absorbing if $\ptran(p\mid p,\action) = 1$ for all $\action\in\actions$, and observable if there is an observation $\obs_p$ such that, for all $\state,\state'\in\states$, $\action\in\actions$, and $\obs\in\sig$ with $\ptran(\obs,\state'\mid\state,\action)>0$, we have $\obs = \obs_p$ if and only if $\state' = p$.
We may assume, without loss of generality, that
the target set~$\targetSet$ contains a single state~$\target$ and that there
is a unique state $\bad$ with value $0$. Moreover, both $\target$ and $\bad$ are absorbing and observable.
We assume that all POMDPs we consider later in the paper satisfy these properties.
The statement appears to be folklore. %
For completeness, we include a proof below.
\begin{lemma} \label{lem:single target}
  Let $\pdp$ be a POMDP with target set $\targetSet$ and initial belief $\bel{b}$.
  We can, in time linear in the size of $\pdp$, build a POMDP~$\pdp'$ with two new states $\target$ and $\bad$, and a belief $\bel{b}'$ such that
  \begin{itemize}
    \item both $\target$ and $\bad$ are absorbing and observable,
    \item $\bad$ is the only state with value $0$ in $\pdp'$,
    \item the reachability value is preserved, i.e., $\preach^{\pdp}_\targetSet(\bel{b}) = \preach^{\pdp'}_{\{\target\}}(\bel{b}')$, and
    \item if $\pdp$ is posterior-deterministic (\Cref{sec:deterministicPOMDPs}), then so is~$\pdp'$.
  \end{itemize}
\end{lemma}%
  \begin{proof}
    Our proof is in two steps: first, we show how to reduce to the case of a single target state, and then we show how to reduce to the case of a single state with value $0$.
    These two steps can then be composed to obtain the desired result.

    We construct a new POMDP $\pdp' = (\states', \actions, \sig', \ptran')$ with $\states' = (\states \setminus \targetSet) \cup \{\target\}$, where $\target$ is a fresh state not in~$\states$, $\sig' = \sig \cup \{\signal_\target\}$, where $\signal_\target$ is a fresh observation not in~$\sig$, and the transition function $\ptran'$ defined as follows:
    \begin{itemize}
      \item for all $\state,\state'\in \states\setminus \targetSet$, $\action\in\ac$, and $\obs\in\sig$,
      \[
      \ptran'(\obs, \state' \mid \state, \action) = \ptran(\obs, \state' \mid \state, \action),
      \]
      \item for all $\state\in \states\setminus \targetSet$ and $\action\in\ac$,
      \[
      \ptran'(\signal_\target, \target \mid \state, \action) = \sum_{\obs\in\sig, \state'\in \targetSet} \ptran(\obs, \state' \mid \state, \action),
      \]
      \item for all $\action\in\ac$,
      \[
      \ptran'(\signal_\target, \target \mid \target, \action) = 1.
      \]
    \end{itemize}
    In particular, state $\target$ is absorbing and observable in $\pdp'$.
    Moreover, this construction can be done in linear time in the size of $\pdp$.

    We also define the initial belief $\bel{b}'$ in~$\pdp'$ as $\bel{b}'(\state) = \bel{b}(\state)$ for all $\state\in\states\setminus \targetSet$ and $\bel{b}'(\target) = \sum_{\state\in\targetSet} \bel{b}(\state)$.
    We prove that for any strategy $\strategy$ in~$\pdp$, there is a corresponding strategy $\strategy'$ in~$\pdp'$ such that the probability of reaching $\targetSet$ in~$\pdp$ under $\strategy$ from $\bel{b}$ is equal to the probability of reaching $\{\target\}$ in $\pdp'$ under $\strategy'$ from $\bel{b}'$, and vice versa.
    This implies in particular that the values of $\pdp$ and $\pdp'$ with respect to the reachability objectives defined by $\targetSet$ and $\{\target\}$, respectively, are equal (up to a natural change in the initial belief).

    Let $\strategy$ be a strategy in $\pdp$.
    We define a strategy $\strategy'$ in $\pdp'$ as follows.
    Let $\action\in\ac$ be an arbitrary action.
    Intuitively, we make strategy $\strategy'$ mimic strategy $\strategy$ until the target state $\target$ is reached, after which $\strategy'$ takes the arbitrary action $\action$ forever.
    Notice that if an observable history $\history' \in (\ac \times \sig')^*$ contains no occurrence of $\signal_\target$, then it can be seen as an observable history $\history \in (\ac \times \sig)^*$ in $\pdp$.
    We can therefore define $\strategy'$ as
    \[
    \strategy'(\history') =
    \begin{cases}
    \strategy(\history') & \text{if $\history'$ contains no occurrence of $\signal_\target$} \\
    \dirac{\action} & \text{otherwise,}
    \end{cases}
    \]
    where $\dirac{\action}$ is the Dirac distribution on $\action$.

    Observe that $\reach{\targetSet}$ is the disjoint union of the cylinders $\cyl{\history}$ for all histories $\history\in\states \times (\ac \times \sig \times \states)^*$ that visit $\targetSet$ a single time at their last state.
    For such an $\history$, let $\history'$ be the corresponding history in $\pdp'$ obtained by replacing the last observation with $\signal_\target$ and the last state with $\target$.
    Several histories~$\history$ of~$\pdp$ are mapped to the same history~$\history'$ of~$\pdp'$: they differ only in their last observation and their last state.
    For a history $\history'$ of $\pdp'$ that visits~$\target$ a single time at its last state, let $H(\history')$ be the set of histories of~$\pdp$ mapped to~$\history'$.
    By construction, both $\strategy$ and $\strategy'$ coincide exactly up to the first visit of $\targetSet$ or $\target$; moreover, the probability of a transition from a state $\state\notin\targetSet$ to~$\target$ in~$\pdp'$ is the sum of the probabilities of the transitions from~$\state$ to~$\targetSet$ in~$\pdp$, and $\bel{b}'(\target) = \sum_{\state\in\targetSet}\bel{b}(\state)$.
    Hence
    \[\pdist{\bel{b}'}{\strategy'}(\cyl{\history'}) = \sum_{\history\in H(\history')}\pdist{\bel{b}}{\strategy}(\cyl{\history}) \ .\]
    Since $\reach{\{\target\}}$ is the disjoint union of the cylinders $\cyl{\history'}$ for the histories~$\history'$ of~$\pdp'$ visiting~$\target$ a single time at their last state, summing over all such~$\history'$, we obtain
    \[\pdist{\bel{b}}{\strategy}(\reach{\targetSet}) = \pdist{\bel{b}'}{\strategy'}(\reach{\{\target\}}) \ .\]

    Conversely, let $\strategy'$ be a strategy in $\pdp'$.
    We define a strategy~$\strategy$ in $\pdp$ as follows.
    Let $\action\in\ac$ be an arbitrary action.
    Let $\history \in (\ac \times \sig)^*$ be an observable history in $\pdp$.
    We distinguish whether $\history$ is a possible observable history in $\pdp'$ or not.
    One can show that $\history$ is a possible observable history in $\pdp'$ if and only if there is a history $\overline{\history} \in \states \times (\ac \times \sig \times \states)^*$ in $\pdp$ such that its projection to $(\ac \times \sig)^*$ is $\history$ and none of the states in $\overline{\history}$ belong to~$\targetSet$.

    We now define $\strategy$ on an observable history $\history$ as:
    \begin{itemize}
      \item if $\history$ is a possible observable history in $\pdp'$ from $\bel{b}'$, we define $\strategy(\history) = \strategy'(\history)$;
      \item if not, it means that all histories $\overline{\history}$ with projection $\history$ have necessarily reached $\targetSet$ in $\pdp$, so we may define $\strategy(\history) = \dirac{\action}$.
    \end{itemize}
    Once again, for every history $\history \in \states \times (\ac \times \sig \times \states)^*$ that visits $\targetSet$ a single time at its last state, let $\history'$ be the corresponding history in $\pdp'$ obtained by replacing the last observation with $\signal_\target$ and the last state with $\target$.
    By construction, both $\strategy$ and $\strategy'$ coincide exactly up to the first visit of $\targetSet$ or $\target$, hence, with the same notation $H(\history')$ as above,
    \[\pdist{\bel{b}'}{\strategy'}(\cyl{\history'}) = \sum_{\history\in H(\history')}\pdist{\bel{b}}{\strategy}(\cyl{\history}) \ .\]
    Summing over all histories~$\history'$ of~$\pdp'$ visiting~$\target$ a single time at their last state, we obtain
    \[\pdist{\bel{b}}{\strategy}(\reach{\targetSet}) = \pdist{\bel{b}'}{\strategy'}(\reach{\{\target\}}) \ .\]

    We now turn to the second part of the lemma.
    We prove that we can assume without loss of generality that there is a single state $\bad$ with value $0$, which is absorbing and observable.
    Let $\states_\bad = \setof{\state\in\states}{\preach^\pdp_\targetSet(\dirac{\state}) = 0}$ be the set of states with value $0$.
    They can be easily identified by standard graph algorithms, since they are the states from which there is no path to $\targetSet$ in the underlying graph.
    We construct a new POMDP $\pdp' = (\states', \actions, \sig', \ptran')$ with $\states' = (\states \setminus \states_\bad) \cup \{\bad\}$, where $\bad$ is a fresh state not in $\states$, $\sig' = \sig \cup \{\signal_\bad\}$, where $\signal_\bad$ is a fresh observation not in $\sig$, and the transition function $\ptran'$ defined as follows:
    \begin{itemize}
      \item for all $\state,\state'\in \states\setminus \states_\bad$, $\action\in\ac$, and $\obs\in\sig$,
      \[
      \ptran'(\obs, \state' \mid \state, \action) = \ptran(\obs, \state' \mid \state, \action),
      \]
      \item for all $\state\in \states\setminus \states_\bad$ and $\action\in\ac$,
      \[
      \ptran'(\signal_\bad, \bad \mid \state, \action) = \sum_{\obs\in\sig, \state'\in \states_\bad} \ptran(\obs, \state' \mid \state, \action),
      \]
      \item for all $\action\in\ac$,
      \[
      \ptran'(\signal_\bad, \bad \mid \bad, \action) = 1.
      \]
    \end{itemize}
    By construction, state $\bad$ is absorbing, observable, and the only state with value $0$ in $\pdp'$.
    Moreover, this construction can be done in linear time in the size of $\pdp$ since $\states_\bad$ can be identified in linear time (they are the states that cannot reach $\targetSet$ in the underlying graph).

    The constructions are then the same as in the proof for the single target state: we build a correspondence between strategies in $\pdp$ and strategies in $\pdp'$ that makes them coincide on all observable histories that have not surely reached $\states_\bad$ or $\bad$ yet.
    Since histories that reach $\states_\bad$ or $\bad$ have value $0$ in both POMDPs, the probabilities of reaching the target sets are preserved (the probability of a history of~$\pdp'$ ending in~$\bad$ is again the sum of the probabilities of the histories of~$\pdp$ that are mapped to it).
    This shows in particular that the values of $\pdp$ and $\pdp'$ with respect to the reachability objectives defined by $\targetSet$ and $\{\target\}$, respectively, are equal.

    Finally, both constructions preserve posterior-determinism: from every state, the new observation~$\signal_\target$ (resp.~$\signal_\bad$) can only lead to~$\target$ (resp.~$\bad$), and the other transitions of~$\pdp'$ are transitions of~$\pdp$.
  \end{proof}

\subparagraph{MDPs and end components.}
A (finite) \emphdef{Markov decision process} (MDP) is a POMDP in which every state is observable; equivalently, it is a tuple $\mathcal{M} = (V,\actions,P)$ where $V$ is a finite set of states and $P\colon V\times\actions\to\Dist(V)$, and the agent observes the current state.
We write $P(v'\mid v,\action)$ for the probability of going from~$v$ to~$v'$ with action~$\action$.
An \emphdef{end component} (EC) of~$\mathcal{M}$~\cite{ThesisAlfaro},~\cite[Section~10.6.3]{BK08} is a partial function $g\colon V\rightharpoonup 2^{\actions}$ with non-empty domain such that
(i) $g(v)\neq\emptyset$ for all $v\in\dom{g}$;
(ii) for all $v\in\dom{g}$, $\action\in g(v)$, and $v'\in V$ with $P(v'\mid v,\action)>0$, we have $v'\in\dom{g}$; and
(iii) the directed graph with vertices $\dom{g}$ and an edge from~$v$ to~$v'$ whenever $P(v'\mid v,\action)>0$ for some $\action\in g(v)$ is strongly connected.
In other words, an EC is a set of states together with a non-empty set of actions for each of them, such that the agent can stay forever in the set by playing these actions, and can move between any two states of the set while doing so.
An EC~$g$ is \emphdef{contained} in an EC~$g'$ if $\dom{g}\subseteq\dom{g'}$ and $g(v)\subseteq g'(v)$ for all $v\in\dom{g}$, and \emphdef{maximal} (a MEC) if it is not contained in any other EC.
The union of two ECs whose domains intersect is an EC; consequently, the domains of distinct MECs are disjoint, and the MECs of an MDP can be computed in polynomial time~\cite{CY95,ThesisAlfaro}.
The main property of ECs is that, under any strategy, almost every run eventually stays forever in the domain of some EC, and the set of states and actions it visits infinitely often forms an EC~\cite{ThesisAlfaro},~\cite[Section~10.6.3]{BK08}.
This is why reachability and $\omega$-regular objectives in MDPs can be analysed through MECs.
In this paper, we use the ECs of a finite MDP built from a POMDP, the belief-support MDP (\Cref{sec:deterministicPOMDPs}); its ECs are called \emph{support end components} (\Cref{def:sec}).

\subparagraph{The belief MDP.}
A POMDP~$\pdp$ induces a fully observable MDP with infinitely many states, the \emphdef{belief MDP}~\cite{Astrom1965,KaelblingLittmanCassandra1998}: its states are the beliefs of~$\pdp$, and from a belief~$\bel{b}$, action~$\action$ leads to the belief $\bel{b}_{\action,\obs}$ with probability $\ptran(\obs\mid\bel{b},\action)$, for every observation~$\obs$ such that $\ptran(\obs\mid\bel{b},\action)>0$.
Since the current belief is a function of the initial belief and of the observable history, observation-based strategies of~$\pdp$ correspond to strategies of the belief MDP, and by~\eqref{eq:shift} the sequence of beliefs $(\bel{b}_{\partialRun{n}})_{n\in\N}$ along a run of~$\pdp$ is distributed as a run of the belief MDP.
Equations~\eqref{eq:val bel act} and~\eqref{eq:val bel upd} are the optimality (Bellman) equations of the belief MDP for the reachability objective.
The tree unfolding used by our algorithm (\Cref{sec:descriptionAlgo}) is an exploration of the belief MDP forward from the initial belief, which maintains a lower and an upper bound on the value; we compare it with existing planning and verification algorithms in the introduction.

\subparagraph{Problem.}
We focus on the following %
problem. %
It is undecidable
for general
POMDPs~\cite{DBLP:journals/ai/MadaniHC03,Fij17}.
\begin{description}
  \item[Value approximation.] Given a POMDP~$\pdp$, a belief~$\bel{b}$, a
    target state~$\target$, and a tolerance~$\epsilon > 0$, compute a number $v \in [0,1]$ such that 
    \( \left| \preach(\bel{b}) - v \right| \leq \epsilon.
    \)
\end{description}

\section{\Deterministic POMDPs} \label{sec:deterministicPOMDPs}
We
now present the class of POMDPs we call \emph{\deterministic}.

\subsection{POMDPs and automata}
A \emphdef{finite automaton} is a tuple $\autSolo = (\autStates, \letters, \autTrans)$ where $\autStates$ is a finite set of \emphdef{states}, $\letters$ is a finite \emphdef{alphabet}, and $\autTrans \subseteq \autStates \times \letters \times \autStates$ is a \emphdef{transition relation}.
It is \emphdef{deterministic} if for every $\autState \in \autStates$ and
$\letter \in \letters$, there is at most one state $\autState' \in \autStates$ such that $(\autState, \letter, \autState') \in \autTrans$.
For deterministic finite automata, we replace $\autTrans$ by a (partial) \emphdef{transition function} $\transition\colon \autStates \times \letters \rightharpoonup \autStates$ defined as $\transition(\autState, \letter) = \autState'$ if $(\autState, \letter, \autState') \in \autTrans$.
We extend the transition function to words in the usual way and write $\transitionIter(\autState, w)$ for the state, if any, reachable from $\autState$ by reading the word $w\in \letters^*$.

From a POMDP~$\pdp = \pdpFull$, we construct a (possibly non-deterministic)
finite automaton~\emphdef{$\aut{\pdp}$} on the alphabet $\ac \times \sig$ with
$\states$ as the set of states and the transition relation~$\autTrans$ so that
$(\state,(\action,\obs),\state') \in \autTrans$
if there is a non-zero probability to go from the state $\state$ to the state $\state'$ after taking the action $\action$ and receiving the observation $\obs$, i.e., if $\ptran(\obs, \state' \mid \state, \action) > 0$.%

We extend the transition relation $\autTrans$ of $\aut{\pdp}$ to all subsets of states, i.e., to $2^{\states}$, in the natural way:
\[
\transition(\bsupp,(\action,\obs)) =
\{\state' \in \states \mid \exists\state\in \bsupp,\ (\state,(\action,\obs),\state') \in \autTrans\} =
\{\state' \in \states \mid \exists\state\in \bsupp,\ \ptran(\obs, \state' \mid \state, \action) > 0\} \ .
\]
For brevity, we write $\transition(\bsupp,\action,\obs)$ instead of $\transition(\bsupp,(\action,\obs))$.
In particular, $\transition(\emptyset,\action,\obs) = \emptyset$, and $\transition(\bsupp,\action,\obs) = \emptyset$ when observation~$\obs$ cannot be received after action~$\action$ from any state of~$\bsupp$.
We extend $\transition$ to observable histories as usual: $\transitionIter(\bsupp,\emptyhist) = \bsupp$ and $\transitionIter(\bsupp,\history\action\obs) = \transition(\transitionIter(\bsupp,\history),\action,\obs)$; thus $\transitionIter(\bsupp,\history)$ is defined (and possibly empty) for every $\bsupp\subseteq\states$ and every $\history\in(\actions\times\sig)^*$.
We can then define a \emph{deterministic} finite automaton \emphdef{$\baut{\pdp}$} with set of states $2^{\states}$ (in which $\emptyset$ is a sink state), alphabet $\ac\times\sig$, and transition function $\transition$.

The \emphdef{belief-support MDP} of~$\pdp$ is the finite MDP whose states are the belief supports $\bsup{\pdp}$, whose actions are~$\actions$, and in which action~$\action$ leads from~$\bsupp$ to $\transition(\bsupp,\action,\obs)$ for every observation~$\obs$ such that $\transition(\bsupp,\action,\obs)\neq\emptyset$ (for instance with uniform probability over these observations; only the set of possible successors matters in this paper).
Its underlying graph is the restriction of $\baut{\pdp}$ to non-empty sets of states.
The belief-support MDP abstracts the belief MDP of \Cref{sec:prelim}, as shown by \Cref{lem:possible hist} below; it is the standard tool for the qualitative analysis of POMDPs (see, e.g.,~\cite{BGB12,BFGHPV25}).

We also extend the transition function of $\pdp$ to $\bsp{\pdp}$ (and to sub-beliefs) in the natural way:
\[
\transition(\bel{b},\action,\obs) = \bel{b}_{\action, \obs} \ .
\]
where $\bel{b}_{\action, \obs}$ is defined in~\Cref{eq:belief update}, and $\transitionIter(\bel{b},\history) = \bel{b}_\history$ for observable histories~$\history$ possible from~$\bel{b}$. Note
that this gives rise to a possibly infinite automaton
since the set of beliefs is infinite.

\begin{lemma}\label{lem:possible hist}
Let $\pdp$ be a POMDP, $\bel{b}$ be a non-zero sub-belief, and $\history$ be an observable history.
Then $\history$ is possible from~$\bel{b}$ if and only if $\transitionIter(\supp{\bel{b}},\history)\neq\emptyset$, and in this case $\supp{\bel{b}_\history} = \transitionIter(\supp{\bel{b}},\history)$.
In particular, for every action~$\action$ and observation~$\obs$, $\ptran(\obs\mid\bel{b},\action)>0$ if and only if $\transition(\supp{\bel{b}},\action,\obs)\neq\emptyset$.
\end{lemma}
\begin{proof}
By induction on the length of~$\history$, it suffices to prove the case of a single action-observation pair $\action\obs$.
We have $\ptran(\obs\mid\bel{b},\action) = \sum_{\state\in\supp{\bel{b}}}\bel{b}(\state)\sum_{\state'\in\states}\ptran(\obs,\state'\mid\state,\action)$, which is positive if and only if there are $\state\in\supp{\bel{b}}$ and $\state'\in\states$ with $\ptran(\obs,\state'\mid\state,\action)>0$, i.e., if and only if $\transition(\supp{\bel{b}},\action,\obs)\neq\emptyset$.
In this case, by~\eqref{eq:belief update}, $\bel{b}_{\action,\obs}(\state')>0$ if and only if $\ptran(\obs,\state'\mid\state,\action)>0$ for some $\state\in\supp{\bel{b}}$, i.e., if and only if $\state'\in\transition(\supp{\bel{b}},\action,\obs)$.
\end{proof}

\subsection{\Deterministic POMDPs}
\begin{definition}
A POMDP $\pdp$ is called \emphdef{\deterministic} if $\aut{\pdp}$ is deterministic.
\end{definition}
In other words, a POMDP $\pdp = \pdpFull$ is \deterministic if for every state $\state\in\states$, action $\action\in\ac$, and observation $\obs\in\sig$, there exists at most one state $\state'\in\states$ such that $\ptran(\obs, \state' \mid \state, \action) > 0$.
In a \deterministic POMDP, if the current state $\state$ is known, then taking an action $\action$ and receiving an observation~$\obs$ leads to at most one possible successor state~$\state'$.
We use $\transition\colon \states \times \ac \times \sig \rightharpoonup \states$ to denote the (partial) \emph{transition function} of $\aut{\pdp}$.
It is defined on $(\state,\action,\obs)$ if and only if $\ptran(\obs\mid\state,\action)>0$, and then $\ptran(\obs,\transition(\state,\action,\obs)\mid\state,\action) = \ptran(\obs\mid\state,\action)$.
Its extension $\transitionIter(\state,\history)$ to observable histories is also a partial function, and, by induction on~$\history$, for every $\bsupp\subseteq\states$,
\begin{equation}\label{eq:supp image}
\transitionIter(\bsupp,\history) = \setof{\transitionIter(\state,\history)}{\state\in\bsupp \text{ and } \transitionIter(\state,\history) \text{ is defined}}\ .
\end{equation}

In a \deterministic POMDP, once the current state is known, it remains known forever, since every subsequent action-observation pair leads to a unique successor state.
More generally, the size of the current belief support can never increase when taking actions and receiving observations, as stated in the following lemma.

\begin{lemma}\label{lem:supp dec}
    Let $\pdp = \pdpFull$ be a \deterministic POMDP.
    For every set of states $\bsupp\subseteq\states$ and observable history
    $\history \in (\ac \times \sig)^*$ we have
    \(
    |\transitionIter(\bsupp,\history)| \leq |\bsupp| \ .
    \)
\end{lemma}
\begin{proof}
By~\eqref{eq:supp image}, $\transitionIter(\bsupp,\history)$ is the image of a subset of~$\bsupp$ under the function $\state\mapsto\transitionIter(\state,\history)$.
\end{proof}

In \deterministic POMDPs, the origin of the partial observability is in the initial belief: given the initial state and the observable history, the current state is determined.
However, even when the current state is known, the observation received after taking an action may be random, in contrast to deterministic POMDPs~\cite{LittmanThesis,Bon09}, and the subsequent state may be random as well, in contrast to quasi-deterministic POMDPs~\cite{BC09}.

Our main result is that the value-approximation problem is decidable for \deterministic POMDPs with reachability objectives.
\begin{theorem}
    There is an algorithm that solves the value-approximation problem for
    \deterministic POMDPs. Moreover, given a \deterministic POMDP $\pdp$ with rational transition probabilities, a rational belief
    $\bel{b}$, a rational tolerance $\varepsilon > 0$, and a rational threshold $v \in [0,1]$ (all numbers being encoded in binary) as
    input, deciding which case
    applies from (i) $\preach(\bel{b}) \geq v + \epsilon$ or (ii) $\preach(\bel{b}) < v$
    is in 3EXPTIME.
\end{theorem}

\begin{remark}[Memory of $\epsilon$-optimal strategies]
    For general POMDPs with reachability objectives, \emph{finite-memory strategies} suffice for $\epsilon$-optimality (with $\epsilon > 0$).
    Indeed, let $\pdp$ be a POMDP and let $\epsilon > 0$.
    Let $\strategy$ be an $\frac{\epsilon}{2}$-optimal strategy in~$\pdp$.
    Since the value of $\pdp$ under $\strategy$ is the limit of the values of the events
    \[
        \reach{\target}^{\le k} = \setof{\state_0\action_1\obs_1\state_1\ldots \in \Runs(\pdp)}{\text{there exists } 0 \leq i \leq k \text{ such that } \state_i = \target} \ ,
    \]
    as $k$ goes to infinity, there exists $k \in \N$ such that
    \[
    \pdist{\bel{b}}{\strategy}(\reach{\target}^{\le k}) \geq \pdist{\bel{b}}{\strategy}(\reach{\target}) - \frac{\epsilon}{2}
    \geq \sup_{\strategy'} \pdist{\bel{b}}{\strategy'}(\reach{\target}) - \epsilon \ .
    \]%
    This means that a finite-memory strategy that mimics~$\strategy$ for $k$ steps and then takes arbitrary actions is $\epsilon$-optimal.
    This holds for
    all POMDPs, not just \deterministic ones.
\end{remark}

\section{Description of the algorithm} \label{sec:descriptionAlgo}
\newcommand{\belof}{\mathsf{belief}}

In this section, we motivate and describe the main features of our algorithm to approximate the reachability value in \deterministic POMDPs.
In \Cref{sec:intuition}, we informally present a reasonable but incomplete algorithm based on a tree unfolding.
We then argue for the need for three more sophisticated unfolding operations in \Cref{sec:justifying SECs} through examples.
Finally, we present our complete algorithm in \Cref{sec:complete algo}.
Except for \Cref{lem:sec basic} and the complexity analysis at the end of \Cref{sec:complete algo}, proofs are deferred to \Cref{sec:correctness,sec:ec}.
Let $\pdp = \pdpFull$ be a \deterministic POMDP for the section.

\subsection{A naive tree unfolding} \label{sec:intuition}
\subparagraph{A first attempt at an unfolding.}
Using \Cref{lem:single target},
we can assume without loss of generality that the POMDP $\pdp$ contains a unique absorbing and observable target state~$\target$ and a unique absorbing and observable state~$\bad$ with value $0$.
Then, for any belief $\bel{b}$,
$\preach(\bel{b}) = 1$ if $\supp{\bel{b}} = \set{\target}$,
$\preach(\bel{b}) = 0$ if $\supp{\bel{b}} = \set{\bad}$,
and otherwise
we apply \Cref{eq:val bel act} and \Cref{eq:val bel upd} to unfold the tree.
\begin{figure*}[t!]
\centering
\begin{minipage}{0.45\textwidth}
\begin{center}
\begin{tikzpicture}[
    level 1/.style={sibling distance=1.6cm, level distance=1.6cm},
    level 2/.style={sibling distance=2.4cm, level distance=2.75cm},
    ->
]

\node {$\bel{b}$}
    child {
        node {$(\bel{b}, \action_1)$}
        child {
            node {$\transition(\bel{b}, \action_1, \obs_1)$}
            edge from parent node[left, pos=0.1, xshift=-10pt, rotate=45] {\footnotesize $\ptran(\obs_1 \mid \bel{b}, \action_1)$}
        }
        child {
            node {$\transition(\bel{b}, \action_1, \obs_2)$}
            edge from parent node[left, pos=0.1, xshift=-8pt, rotate=90] {\footnotesize $\ptran(\obs_2 \mid \bel{b}, \action_1)$}
        }
        child {
            node[xshift=-1.2cm,yshift=1.375cm] {$\cdots$}
            edge from parent[draw=none]
        }
        edge from parent node[left, pos=0.5, xshift=-2pt] {\small $1$}
    }
    child {
        node {$(\bel{b}, \action_2)$}
        edge from parent node[left, pos=0.5, xshift=2pt] {\small $1$}
    }
    child {
        node[xshift=-.8cm,yshift=.8cm] {$\cdots$}
        edge from parent[draw=none]
    };

\end{tikzpicture}
\caption{Tree-like depiction of the naive unfolding.}
\label{fig:expand}
\end{center}
\end{minipage}%
\hfill
\begin{minipage}{0.45\textwidth}
\begin{center}
\begin{tikzpicture}[
state/.style={circle, draw=black, very thick, minimum size=20pt},
wait/.style={circle, fill=black},
pointer/.style={circle}
]
\node[pointer] (mid) {};
\node[state]   (q1) [left=of mid]  {$\state_1$};
\node[state]   (q2) [right=of mid] {$\state_2$};
\node[wait]    (wq1) [below of=mid] {};
\node[wait]    (wq2) [above of=mid] {};
\node[pointer] (low) [below of=wq1] {};
\node[state]   (r) [below of=low] {$\state_3$};
\draw[->] (q1.south) to [out=270,in=180] node[below]{$\action$} (wq1.west);
\draw[->,red] (wq1.south) to node[left]{$\obs_2 : \frac{1}{2}$} (r.north);
\draw[->,blue] (wq1.east) to [out=0,in=270] node[below] {$\obs_1 : \frac{1}{2}$} (q2.south);
\draw[->] (q2.north) to [out=90,in=0] node[below] {$\action$} (wq2.east);
\draw[->,blue] (wq2.west) to [out=180,in=90] node[above] {$\obs_1 : 1$} (q1.north);
\end{tikzpicture}
\caption{\Deterministic POMDP with an infinite $\tree(\bel{b})$.}
\label{fig:3 st aodscc}
\end{center}
\end{minipage}
\end{figure*}
As depicted in \Cref{fig:expand},
this %
can be seen as evaluating a tree whose nodes are labelled with either a belief or a (belief, action) pair.
This tree is the unfolding (computation tree) of the belief MDP of \Cref{sec:prelim} from~$\bel{b}$: belief nodes correspond to states of the belief MDP, in which the agent chooses an action, and (belief, action) nodes correspond to the random choice of the next observation.
From a belief node $\bel{b}$, there is an edge with weight~$1$ to child nodes $(\bel{b},\action)$ for all actions $\action\in\actions$.
From a (belief, action) node $(\bel{b},\action)$, there is an edge with weight $\ptran(\obs\mid \bel{b},\action)$ to child nodes $\transition(\bel{b},\action,\obs)$ for all observations $\obs\in\sig$ such that $\ptran(\obs\mid \bel{b},\action) > 0$.

Let \emphdef{$\tree(\bel{b})$} denote the tree whose root is labelled by $\bel{b}$ and
whose nodes and edges are weighted as described above.
Note that the tree is finitely branching and its branches correspond to (possibly infinite, see \Cref{fig:3 st aodscc}) observable
runs of $\pdp$ starting from~$\bel{b}$.
We use \emphdef{$\rootof(T)$} to denote the root of a tree $T$.
Henceforth, we use \emphdef{$\belof(x)$} to denote the belief appearing in the label of $x$, that is, $\belof(x) = \bel{b}$ if the label of $x$ is $\bel{b}$, and $\belof(x) = \bel{b}$ if the label of $x$ is $(\bel{b},\action)$ for some action~$\action$.
In general,
$\tree(\bel{b})$ is not finite,
even if we assume $\pdp$ to be \deterministic{}.
For example, say the structure in \Cref{fig:3 st aodscc} is a part of $\pdp$. 
If we have a belief node $x$ with $\supp{\belof(x)} = \set{\state_1,\state_2}$,
then the subtree rooted at $x$ has an infinite branch whose nodes alternately have labels of the form $\bel{b}$ and $(\bel{b},\action)$ with $\supp{\bel{b}} = \set{\state_1,\state_2}$.

One could also try to guess the initial state according to the initial belief and solve the MDP obtained for this guess. This does not work in general, since the observations received later give information about the initial state, and the optimal actions depend on the whole current belief.

\subparagraph{Bounding the value from below and above.}
A reasonable idea could then be to approximate the value of $\bel{b}$ by taking a finite truncation of $\tree(\bel{b})$ (say, cut at depth~$n$).
We could then compute a value, which we write $\eval_n(\rootof(\tree(\bel{b})))$, corresponding to the maximal probability to reach $\target$ \emph{within $\lfloor n/2\rfloor$ steps} (belief nodes and (belief, action) nodes alternate along each branch), which is a lower bound on $\preach(\bel{b})$.
This value could be obtained by assuming that the value of a leaf node is $1$ if its label support is $\{\target\}$ and $0$ otherwise, and then evaluating the tree by backward induction (taking the $\max$ of the children for belief nodes and the weighted sum of the children for (belief, action) nodes).

Similarly, one could compute an upper bound on $\preach(\bel{b})$ by
evaluating the maximal possible ``error'' made at depth $n$, which we write
$\eerror_n(\rootof(\tree(\bel{b})))$ (in \Cref{sec:correctness}, the statistic $\erank_n$ plays this role for the smarter unfolding). 
For a leaf $x$ of the truncated tree, we could set this error to $0$
if $\supp{\belof(x)}$ is $\set{\target}$ or $\set{\bad}$ (because we know the exact value in this case),
and to $1$ otherwise (an upper bound on the error).
We could then define the error of all the inner nodes also using backward induction
to obtain \[\eval_n(\rootof(\tree(\bel{b}))) \le \preach(\bel{b}) \leq
\eval_n(\rootof(\tree(\bel{b}))) + \eerror_n(\rootof(\tree(\bel{b})))\ .\]
If $\eerror_n(\rootof(\tree(\bel{b})))$ goes to $0$ as $n$ goes to infinity, we are done.
Unfortunately, this is not the case in general, even for posterior-deterministic POMDPs, for two reasons, which are illustrated by the two scenarios of \Cref{sec:justifying SECs}.
\begin{enumerate}
\item Under some actions, the belief support may remain forever, with positive probability, in a strongly connected set of belief supports. Then, for every~$n$, the leaves at depth~$n$ whose value is unknown have a total weight bounded away from~$0$ (Scenario~1).
\item Even when the current belief support can be left with positive probability at every step, this probability may tend to~$0$ along a branch, because the probability of the states from which the support can be left tends to~$0$. Again, the leaves whose value is unknown may then have a total weight bounded away from~$0$ (Scenario~2).
\end{enumerate}
Our algorithm handles these two problems by adding three operations to the naive unfolding.

\subsection{The need for more complex operations} %
We present \emph{two scenarios} that illustrate the need for a more fine-grained approach: they correspond to the two problems of the naive unfolding described at the end of \Cref{sec:intuition}, and lead to three new unfolding operations, called split, exit, and cut.

\label{sec:justifying SECs}
\subparagraph{Scenario 1: Trapped in a strongly connected component of belief supports.}
Assume the structure in \Cref{fig:3 st adscc} is part of $\pdp$.  For any
belief $\bel{b}$ with $\supp{\bel{b}}$ being either $\set{\state,\state_1}$ or
$\set{\state,\state_2}$, $\supp{\transition(\bel{b},\action,\obs_1)}$ and
$\supp{\transition(\bel{b},\action,\obs_2)}$ are either
$\set{\state,\state_1}$ or $\set{\state,\state_2}$. Thus, $\tree(\bel{b})$
contains an infinite subtree rooted at $\rootof(\tree(\bel{b}))$ and whose nodes
$x$ all satisfy that $\supp{\belof(x)}$ is either $\set{\state,\state_1}$ or
$\set{\state,\state_2}$.
\begin{figure}
\begin{minipage}{0.46\textwidth}
\begin{center}
\begin{tikzpicture}[
state/.style={circle, draw=black, very thick, minimum size=20pt},
wait/.style={circle, fill=black},
pointer/.style={circle}
]
\node[pointer] (mid) {};
\node[state] (q1)  [above=of mid] {$\state_1$};
\node[state] (q2)  [below=of mid] {$\state_2$};
\node[wait] (wq1) [left of=mid] {};
\node[wait] (wq2) [right of=mid] {};
\node[state] (q) [left of=wq1] {$\state$};
\node[wait] (wq) [left of=q] {};
\draw[->] (q.west) to node[above]{$\action$} (wq.east);
\draw[->,red] (wq.north) to [out=60,in=120] node[above]{$\obs_1:\frac{1}{3}$}(q.north);
\draw[->,blue] (wq.south) to [out=-60,in=-120] node[below]{$\obs_2:\frac{2}{3}$} (q.south);
\draw[->] (q1.south) to [out=270,in=0] node[right]{$\action$} (wq1.east);
\draw[->,red] (wq1.north) to [out=90,in=180] node[left]{$\obs_1:\frac{1}{2}$} (q1.west);
\draw[->,blue] (wq1.south) to [out=270,in=180] node[left]{$\obs_2:\frac{1}{2}$} (q2.west);
\draw[->] (q2.north) to [out=90,in=180] node[left]{$\action$} (wq2.west);
\draw[->,red] (wq2.south) to [out=270,in=0] node[right]{$\obs_1:\frac{1}{2}$} (q2.east);
\draw[->,blue] (wq2.north) to [out=90,in=0] node[right]{$\obs_2:\frac{1}{2}$} (q1.east);
\end{tikzpicture}
\caption{Scenario 1-type gadget with a non-vanishing $\eerror$.}
\label{fig:3 st adscc}
\end{center}
\end{minipage}
\hfill
\begin{minipage}{0.46\textwidth}
\begin{center}
\begin{tikzpicture}[
state/.style={circle, draw=black, very thick, minimum size=20pt},
wait/.style={circle, fill=black},
pointer/.style={circle}
]
\node[pointer] (mid) {};
\node[state]   (q1) [left=of mid]  {$\state_1$};
\node[state]   (q2) [right=of mid] {$\state_2$};
\node[wait]    (wq1) [below of=q1] {};
\node[wait]    (wq2) [below of=q2] {};
\node[pointer] (low) [below of=wq2] {};
\node[state]   (r) [below of=low] {$\state_3$};
\draw[->] (q1.east) to [out=0,in=0] node[right]{$\action$} (wq1.east);
\draw[->,blue] (wq1.west) to [out=180,in=180] node[left] {$\obs_1 : 1$} (q1.west);
\draw[->] (q2.west) to [out=180,in=180] node[left] {$\action$} (wq2.west);
\draw[->,blue] (wq2.east) to [out=0,in=0] node[right] {$\obs_1 : \frac{1}{2}$} (q2.east);
\draw[->,red] (wq2.south) to [out=270,in=90] node[left] {$\obs_2 : \frac{1}{2}$} (r.north);
\end{tikzpicture}
\caption{Scenario 2-type gadget with a non-vanishing $\eerror$.}
\label{fig:vand}
\end{center}
\end{minipage}
\end{figure}
The issue is that, in our unfolding, for every truncation depth $n$ there are leaves $x$ at depth $n$ with $\supp{\belof(x)}$ being either
$\set{\state,\state_1}$ or $\set{\state,\state_2}$, whose error will thus be
$1$. In this case, the error of the root does not converge to $0$ as $n$
goes to infinity.  We need to detect that we are trapped in such a strongly
connected component of belief supports and treat it with special care in the
unfolding.

Below, we formalise such scenarios. Proofs of the statements
for them are in
\Cref{sec:ec}.
\begin{definition} \label{def:sec}
Let $\ascc \colon \bsup{\pdp} \rightharpoonup 2^{\ac}$ be a partial function and $S\in\bsup{\pdp}$.
An observable history $\history = \action_1 \obs_1 \dots \action_n \obs_n$ is \emphdef{inside $\ascc$ from $S$} if $\transitionIter(S,\history)\neq\emptyset$ and, for all $k\in\set{1,\dots,n}$, the belief support $S_{k-1} = \transitionIter(S,\action_1 \obs_1 \dots \action_{k-1} \obs_{k-1})$ belongs to $\dom{\ascc}$ and $\action_k\in \sccFunc(S_{k-1})$.
An observable run is inside $\ascc$ from $S$ if all its finite prefixes are.
We let \emphdef{$\reach{\sccFunc}(S)$} be the set of belief supports $\transitionIter(S,\history)$ for all observable histories $\history$ inside $\ascc$ from $S$.
For a non-zero sub-belief $\bel{b}$, an observable history is inside $\ascc$ from $\bel{b}$ if it is inside $\ascc$ from $\supp{\bel{b}}$; such a history is possible from~$\bel{b}$ by \Cref{lem:possible hist}, and we let \emphdef{$\reach{\sccFunc}(\bel{b})$} be the set of sub-beliefs $\bel{b}_\history$ for all observable histories $\history$ inside $\ascc$ from $\bel{b}$.

A \emphdef{support end component (SEC)} of $\pdp$ is a partial function $\ascc
\colon \bsup{\pdp} \rightharpoonup 2^{\ac}$ with non-empty domain such that:
\begin{enumerate}
\item for all $S \in\dom{\ascc}$, the set $\ascc(S)$ is non-empty;
\item for all $S \in\dom{\ascc}$, $\action\in \ascc(S)$, and $\obs\in\sig$ such that $\beltr{S}{\action,\obs}\neq\emptyset$, the belief support $\beltr{S}{\action,\obs}$ is in $\dom{\ascc}$;
\item for all pairs of belief supports $S$ and $T$ in $\dom{\ascc}$, we have $T\in\reach{f}(S)$.
  \end{enumerate}
\end{definition}
The empty history is inside $\ascc$ from every belief support~$S$, so $S\in\reach{\ascc}(S)$.
Condition~(2) only concerns the observations that are possible from~$S$, i.e., by \Cref{lem:possible hist}, the observations~$\obs$ such that $\ptran(\obs\mid\bel{b},\action)>0$ for some (equivalently, every) belief~$\bel{b}$ with support~$S$; impossible observations do not lead to any belief support.
If $\ascc$ is an SEC and $\history$ is inside $\ascc$ from $S\in\dom{\ascc}$, then, by condition~(2) and induction, the belief supports $\transitionIter(S,\history')$ for all prefixes $\history'$ of~$\history$, including $\history$ itself, belong to~$\dom{\ascc}$.

SECs are exactly the end components (\Cref{sec:prelim}) of the belief-support MDP of~$\pdp$ (\Cref{sec:deterministicPOMDPs}): conditions (1)--(3) above are conditions (i)--(iii) of the definition of end components, since the successors of a belief support~$S$ under action~$\action$ in the belief-support MDP are the non-empty sets $\transition(S,\action,\obs)$.
This definition makes sense for arbitrary POMDPs.

\begin{example}\label{eg:adscc}
In
\Cref{fig:3 st adscc}, there are at least three \adscc{s}.
Namely,
\begin{enumerate}
\item $f_1$ with domain $\set{\set{\state,\state_1,\state_2}}$ which maps every element in the domain to $\set{\action}$,
\item $f_2$ with domain $\set{\set{\state,\state_1},\set{\state,\state_2}}$ which maps every element in the domain to $\set{\action}$, and
\item $f_3$ with domain $\set{\set{\state_1,\state_2}}$ and which maps every element in the domain to $\set{\action}$.
\end{enumerate}
The partial function $f_4$ with domain $\set{\set{\state,\state_1},\set{\state,\state_2},\set{\state_1,\state_2}}$ which maps every element in the domain to $\set{\action}$ is not an \adscc{}.
\end{example}
We say that an \adscc $\sccFunc$ \emphdef{is contained in} another \adscc{} $\sccFunc'$ if $\dom{\sccFunc} \subseteq \dom{\sccFunc'}$ and for all $\bsupp\in \dom{\sccFunc}$, $\sccFunc(\bsupp) \subseteq \sccFunc'(\bsupp)$.
An \adscc{} $\sccFunc$ is called \emphdef{maximal} if for all \adsccs{} $\sccFunc'$,
if $\sccFunc$ is contained in $\sccFunc'$, then $\sccFunc = \sccFunc'$.
We use the same terminology for arbitrary partial functions from $\bsup{\pdp}$ to $2^{\actions}$; observe that if $\sccFunc$ is contained in $\sccFunc'$, then every history inside $\sccFunc$ from a belief support $S$ is also inside $\sccFunc'$ from~$S$.
The \emphdef{union} of finitely many partial functions $g_1,\dots,g_m\colon\bsup{\pdp}\rightharpoonup 2^{\actions}$ is the partial function $g$ with domain $\dom{g} = \bigcup_{i=1}^m\dom{g_i}$ defined by
\[
g(S) = \bigcup\setof{g_i(S)}{i\in\set{1,\dots,m} \text{ such that } S\in\dom{g_i}}\ .
\]
Each $g_i$ is contained in their union.
\begin{restatable}{lemma}{maximalDisjoint} \label{lem:maximal disjoint}
\begin{enumerate}
\item Every \adscc{} is contained in a maximal \adscc.
\item For every belief support $S$ that belongs to the domain of some \adscc, there is a unique maximal \adscc{} whose domain contains~$S$; we denote it by~$\sccFunc_S$. Every \adscc{} whose domain contains~$S$ is contained in~$\sccFunc_S$.
\item The domains of two distinct maximal \adsccs{} are disjoint.
\end{enumerate}
Consequently, $\bsup{\pdp}$ is partitioned into the domains of the maximal \adsccs{} and the (possibly empty) set of belief supports that belong to the domain of no \adscc.
\end{restatable}
In terms of the belief-support MDP, this is the classical decomposition of an MDP into maximal end components (\Cref{sec:prelim}).

\noindent\textit{Histories inside SECs.}
In general, observable histories cannot be concatenated: for instance, with the assumptions of \Cref{lem:single target}, both $\action\obs_\bad$ and $\action\obs_\target$ may be possible from a belief, whereas $\action\obs_\bad\action\obs_\target$ is possible from no belief.
Histories inside an SEC behave better, as shown by the following lemma; its first item holds for arbitrary partial functions, and the others rely on posterior-determinism.
\begin{lemma}\label{lem:sec basic}
Let $\sccFunc\colon\bsup{\pdp}\rightharpoonup 2^{\actions}$ be a partial function, $S\in\bsup{\pdp}$, and $\history$ be an observable history inside $\sccFunc$ from~$S$.
\begin{enumerate}
\item For every observable history $\history'$, the history $\history\history'$ is inside $\sccFunc$ from~$S$ if and only if $\history'$ is inside $\sccFunc$ from $\transitionIter(S,\history)$; in this case, $\transitionIter(S,\history\history') = \transitionIter(\transitionIter(S,\history),\history')$. In particular, if $T\in\reach{\sccFunc}(S)$ and $U\in\reach{\sccFunc}(T)$, then $U\in\reach{\sccFunc}(S)$. \label{item:concat}
\end{enumerate}
Assume now that $\sccFunc$ is an \adscc{} and that $S\in\dom{\sccFunc}$.
\begin{enumerate}[resume]
\item All belief supports in $\dom{\sccFunc}$ have the same cardinality. \label{item:same size}
\item For every $\state\in S$, the state $\transitionIter(\state,\history)$ is defined, and the map $\state\mapsto\transitionIter(\state,\history)$ is a bijection from $S$ onto $\transitionIter(S,\history)$. Moreover, $\transitionIter(\state,\history\history') = \transitionIter(\transitionIter(\state,\history),\history')$ for all $\history'$ inside $\sccFunc$ from $\transitionIter(S,\history)$. In particular, for every $\action\in\sccFunc(S)$ and $\obs\in\sig$ such that $\transition(S,\action,\obs)\neq\emptyset$, we have $\ptran(\obs\mid\state,\action)>0$ for all $\state\in S$. \label{item:bijection}
\item If $\transitionIter(S,\history) = S$, then there is $k\ge 1$ such that the concatenation $\history^k$ of $k$ copies of $\history$ is inside $\sccFunc$ from~$S$ and $\transitionIter(\state,\history^k) = \state$ for all $\state\in S$. \label{item:power}
\end{enumerate}
\end{lemma}
\begin{proof}
Item~\eqref{item:concat} follows directly from the definitions, since $\transitionIter(S,\history\history'') = \transitionIter(\transitionIter(S,\history),\history'')$ for every prefix $\history''$ of $\history'$.

For~\eqref{item:same size}, let $S,T\in\dom{\sccFunc}$. Since $T\in\reach{\sccFunc}(S)$, there is $\history''$ such that $T = \transitionIter(S,\history'')$, hence $|T|\le|S|$ by \Cref{lem:supp dec}. By symmetry, $|S| = |T|$.

For~\eqref{item:bijection}, we have $\transitionIter(S,\history)\in\dom{\sccFunc}$ (see the remark after \Cref{def:sec}), so $|\transitionIter(S,\history)| = |S|$ by~\eqref{item:same size}.
By~\eqref{eq:supp image}, $\transitionIter(S,\history)$ is the image, under the map $\state\mapsto\transitionIter(\state,\history)$, of the set of states $\state\in S$ for which $\transitionIter(\state,\history)$ is defined.
Hence this set is $S$ itself, and the map is injective.
The equality $\transitionIter(\state,\history\history') = \transitionIter(\transitionIter(\state,\history),\history')$ holds by definition of $\transitionIter$ on words, both sides being defined by the first part of~\eqref{item:bijection} and by~\eqref{item:concat}.
For the last statement, the history $\action\obs$ is inside $\sccFunc$ from~$S$, so $\transition(\state,\action,\obs)$ is defined for every $\state\in S$, i.e., $\ptran(\obs\mid\state,\action)>0$.

For~\eqref{item:power}, by~\eqref{item:bijection}, the map $\state\mapsto\transitionIter(\state,\history)$ is a permutation of~$S$; let $k\ge 1$ be such that its $k$-th power is the identity.
By~\eqref{item:concat} and induction, $\history^j$ is inside $\sccFunc$ from~$S$ for every $j\ge 1$, and by~\eqref{item:bijection}, $\transitionIter(\state,\history^k) = \state$ for all $\state\in S$.
\end{proof}
\adscc{s} can be of two kinds.
For illustration, consider \Cref{fig:3 st adscc} once more. Note that after taking the action $\action$,
we observe $\obs_1$ and $\obs_2$ with same probabilities from $\state_1$ and
$\state_2$. These are different from the probabilities with which we observe them from $\state$.
In other words, $\state$ is distinguishable from both $\state_1$ and $\state_2$,
whereas $\state_1$ and $\state_2$ are indistinguishable.
\begin{definition} \label{def:indistinguishable}
    Let $\sccFunc$ be an \adscc of $\pdp$.
    Let $\state, \state'$ be two states from a same belief
    support $\bsupp \in \dom{\sccFunc}$.
    We say that $\state$ and $\state'$ are \emphdef{$(\sccFunc, \bsupp)$-indistinguishable} (denoted $\state \indistinguishable \state'$) if for all observable histories $\history$ inside $\sccFunc$ from $\bsupp$, if we write $\bsupp_\history = \transitionIter(\bsupp, \history)$, $\state_\history = \transitionIter(\state, \history)$, and $\state_{\history}' = \transitionIter(\state', \history)$ (the states $\state_\history$ and $\state'_\history$ are well defined by \Cref{lem:sec basic}), then for all actions $\action\in \sccFunc(\bsupp_\history)$ and observations $\obs\in\sig$,
  \[
    \ptran(\obs \mid \state_\history, \action) = \ptran(\obs \mid
    \state_\history', \action)\ .
  \]
\end{definition}
For all $\bsupp \in \dom{f}$,
the $(\sccFunc, \bsupp)$-indistinguishability relation is an equivalence relation on~$\bsupp$.
Two states $\state$ and $\state'$ are \emphdef{$(\sccFunc, \bsupp)$-distinguishable} if they are not $(\sccFunc, \bsupp)$-indistinguishable.
\begin{definition}
 An \adscc $f$ is called \emphdef{distinguishing} if there exists some belief support $S$ in the domain of $f$ and states $\state,\state'\in S$ such that $\state$ and $\state'$ are $(\sccFunc, \bsupp)$-distinguishable.
\end{definition}
In \Cref{eg:adscc}, both $f_1$ and $f_2$ are distinguishing,
and $f_3$ is not.
Support end components and the distinction between distinguishing and non-distinguishing ones are inspired by the double end components of MEMDPs introduced by Raskin and Sankur~\cite{RS14}, which are end components in both environments of a two-environment MEMDP and are called distinguishing if some action inside them has different distributions in the two environments.
As in their setting, the hidden information can be learned with arbitrarily high confidence inside a distinguishing component (here, by playing uniformly at random, see \Cref{thm:dis adscc}), whereas only the exits matter inside a non-distinguishing one (\Cref{thm:nd adscc}).
In posterior-deterministic POMDPs, however, the components are made of belief supports rather than states, several equivalence classes of indistinguishable states may coexist in a belief support, and infinitely many beliefs may be reachable inside a distinguishing \adscc.

SECs will be used to unfold the tree in two different ways,
depending on whether they are distinguishing or not.

\subparagraph{Scenario~1.1: Distinguishing SECs.}
Our main result about \emph{distinguishing SECs} is that, by staying in such an SEC, it is possible to eventually determine with arbitrarily high confidence the equivalence class of the current state with respect to the indistinguishability relation.
In practice, it suffices to take actions remaining in the SEC uniformly at random for a sufficiently long time.
If $\bsupp/{\indistinguishable}$ stands
for the set of equivalence classes of $\indistinguishable$, then:
\begin{restatable}{theorem}{distSECs}\label[theorem]{thm:dis adscc}
Let $\bel{b}$ be a sub-belief such that $\bsupp = \supp{\bel{b}}$ belongs to the domain of a distinguishing maximal \adscc{} $\sccFunc$.
Then,
\(
\preach(\bel{b}) =
\sum_{%
C \in \bsupp/{\indistinguishable}} 
\preach(\res{\bel{b}}{C}).
\)
\end{restatable}

\noindent
In this result, the ``$\ge$'' direction is the non-trivial one: from a belief
$\bel{b}$ in a distinguishing \adscc, it is possible to obtain a value as if we
knew the equivalence class of the current state. This is strictly more information than is contained in the belief $\bel{b}$.

Thanks to this result, we can unfold a belief node $x$ with
$\supp{\belof(x)}$ in a distinguishing \adscc{} with a \emph{split operation}: we add one
child per equivalence class of the indistinguishability relation. %
This strictly decreases the
size of the belief supports being considered in the subtree, since, in a distinguishing \adscc, every belief support of the domain contains at least two equivalence classes (\Cref{lem:dist all supports}).
\subparagraph{Scenario~1.2: Non-distinguishing \adscc{s}.}
Unfortunately, in non-distinguishing \adsccs, it is not possible to gain more information about the current state by staying in the \adscc.
However, they are still amenable to an algorithmic analysis thanks to the following result.
\begin{restatable}{lemma}{finBeliefSpace}\label{cor:non dist fin reach}
For all non-distinguishing \adsccs $\sccFunc$ of $\pdp$ and all non-zero sub-beliefs $\bel{b}$ such that $\supp{\bel{b}}\in \dom{\sccFunc}$,
the set $\reach{\sccFunc}(\bel{b})$ is finite.
\end{restatable}

One can thus explore all beliefs reachable from $\bel{b}$ while staying in the
\adscc. To reach the target, one must eventually leave the \adscc, but there
are only finitely many ways to do so: we just have to find the ``best exit''.
An \adscc{}~$\sccFunc$ is \emphdef{bottom} if $\sccFunc(S) = \actions$ for all $S\in\dom{\sccFunc}$, i.e., if all actions remain in the \adscc{}, and \emphdef{non-bottom} otherwise.
Such an exit exists when the SEC~$\sccFunc$ is non-bottom, i.e., when $\sccFunc(S)\neq\actions$ for some $S\in\dom{\sccFunc}$; thanks to \Cref{lem:single target}, every maximal SEC whose domain is not $\{\{\target\}\}$ or $\{\{\bad\}\}$ is non-bottom (\Cref{lem:trivial secs}).
\begin{restatable}{theorem}{ndSECs}\label[theorem]{thm:nd adscc}
Let $\bel{b}$ be a sub-belief such that $\supp{\bel{b}}$ belongs to the domain of some non-distinguishing, non-bottom, maximal \adscc{} $\sccFunc$.
Then
\[
\preach(\bel{b}) =
\max_{\bel{b}'\in\reach{\sccFunc}(\bel{b})}\
\max_{\action\notin \sccFunc(\supp{\bel{b}'})}
\preach(\bel{b}',\action)\ .
\]
\end{restatable}
Using this result, we can unfold a belief node $x$ with $\supp{\belof(x)}$
in the domain of a non-distinguishing \adscc{} into one child per pair $(\bel{b}',\action)$ of a sub-belief $\bel{b}'$ in
$\reach{\sccFunc}(\belof(x))$ (these can be enumerated effectively)
and an action $\action$ not in $\sccFunc(\supp{\bel{b}'})$.
We call such unfoldings \emph{exit operations}.
\subparagraph{Scenario 2: An observation that remains possible but may never occur.}
Assume the structure presented in \Cref{fig:vand} is a part of $\pdp$.
The observation $\obs_1$ can be observed from both $q_1$ and $q_2$ after taking the action $a$.
However, $\obs_2$ can only be observed from~$q_2$.
Therefore, $\tree(\bel{b})$ with $\supp{\bel{b}} = \set{q_1,q_2}$ contains an infinite branch that alternates between taking the action $\action$ and observing $\obs_1$.
Note that, even though $\obs_2$ has a positive probability of being observed from belief support $\{\state_1, \state_2\}$ after taking action $\action$, there is a branch with positive probability that observes an infinite number of observations of $\obs_1$ and no observation of $\obs_2$.
In other words, observation~$\obs_2$ remains \emph{possible} at every step along this branch (in the sense of \Cref{sec:prelim}: $\ptran(\obs_2\mid\bel{c},\action)>0$ for every belief~$\bel{c}$ on the branch), but it never \emph{occurs}: if $\bel{b}(\state_1) = p$, the probability of observing $\obs_1$ forever is~$p > 0$.
This prevents the error of the root from converging to $0$ as we go deeper in the tree.
Yet, along this branch, the probability to be in $\state_2$ converges to $0$,
without actually reaching $0$, as we go deeper in the tree. This property should intuitively allow us to approximate the value of $\bel{b}$ by the value of $\state_1$.

To tackle this scenario, we simply fix a threshold $\thr > 0$ and truncate beliefs by removing any probability that is smaller than $\thr$.

\begin{definition}\label[definition]{def:delmin}
Let $\bel{b}$ be a sub-belief and $\thr > 0$ be a threshold.
Define \emphdef{$\delmin_\thr(\bel{b})$} to be the sub-belief such that
$
\delmin_\thr(\bel{b})(q) =
0$ if $\bel{b}(q) < \thr$, and $\delmin_\thr(\bel{b})(q) = \bel{b}(q)$ otherwise.
\end{definition}
The value of $\thr$ will depend on the allowed tolerance $\epsilon > 0$
specified by the input to the value-approximation algorithm. Indeed, such a
\emph{cut operation} introduces an error but
it helps us avoid infinite branches such as the one described above: after a
bounded number of steps, the belief will necessarily be truncated, thereby
strictly decreasing the size of its support.
If we apply the split, exit, and cut operations whenever possible in an unfolding, split and cut operations can be applied at most $|\states| - 1$ times along each branch, since they strictly decrease the size of the belief support, which never increases (\Cref{lem:supp dec}).
For the exit operation (Scenario~1.2), the situation is more involved: the operation sends a probability mass outside of the current non-distinguishing \adscc, but some mass may come back to it, and exit operations may be applied infinitely often along a branch (of probability~$0$).
In particular, an exit operation does not necessarily decrease the rank defined in \Cref{sec:correctness}, even with positive probability: the belief supports reached after an exit may be equivalent to the current one.
Nevertheless, we show in \Cref{sec:correctness} that the unfolding still makes progress: whatever the choices made at the nodes where an action or an exit is chosen, within $2^{|\states|+1}$ unfolding steps the tree contains a node of non-negligible weight at which the rank decreases (a leaf of known value, a cut, or a split); otherwise, the actions chosen in the tree would form an \adscc{} contradicting the maximality of the \adsccs{} used by exit operations.
As a consequence, the weighted rank of the leaves decreases geometrically with the depth (\Cref{clm:err dec}).
Since the total mass removed by cut operations along a branch is at most $|\states|\cdot\thr$, the error introduced by cutting can be controlled with the right choice of $\thr$ computed as a function of $\epsilon$ and the size of $\pdp$.

\subsection{The smarter tree unfolding} \label{sec:complete algo}
\Cref{thm:dis adscc}, \Cref{thm:nd adscc}, and \Cref{def:delmin} give us three new rules to unfold the tree.
We now use them to describe our main algorithm.
First, we assume we are given a tolerance parameter $\epsilon > 0$ and fix
$\thr = \frac{\epsilon}{2\cdot |\states|}$.
\begin{definition}[Rules for unfolding the tree]\label{def:unfold}
A node is unfolded
using the first applicable rule amongst the following ones.
The cases are split according to the label of the node.
\begin{reflist}{def:unfold}
\item (Case: $(\bel{b},\action)$) Applying \eqref{eq:val bel act},
we create one child per observation $\obs$ such that $\ptran(\obs\mid\bel{b},\action)>0$, which is labelled by~$\bel{b}_{a,\obs}$ (with the mass-preserving update of \Cref{sec:prelim}),
and the edge connecting it to its parent has weight $\ptran(\obs \mid \normalised{\bel{b}},\action)$.
\label{unr:act}
\item (Case: $\bel{b}$ such that $\supp{\bel{b}}$ is either $\set{\target}$ or $\set{\bad}$, or $\bel{b} = 0$) In this case, this node is a leaf.
\label{unr:bot}
\item (Case: $\bel{b}$ such that $\min\setof{\bel{b}(\state)}{\state\in\states,\ \bel{b}(\state)\neq 0} < \thr$)
Using \Cref{def:delmin},
we create a child which is labelled by $\delmin_\thr(\bel{b})$,
and the edge connecting it to its parent has weight~$1$.%
\label{unr:delmin}%
\item (Case: $\bel{b}$ such that $\supp{\bel{b}}$ belongs to the domain of some \adscc{} and the maximal \adscc{} $\sccFunc$ containing $\supp{\bel{b}}$ (denoted $\sccFunc_{\supp{\bel{b}}}$ in \Cref{lem:maximal disjoint}) is distinguishing.)
Applying \Cref{thm:dis adscc}, we create one child per equivalence class $[\state]$ of $\indistinguishable$ (with $\bsupp = \supp{\bel{b}}$).
The child is labelled by
$\frac{\oneNorm{\bel{b}}}{\oneNorm{\res{\bel{b}}{[\state]}}}\cdot\res{\bel{b}}{[\state]}$, and the edge connecting to its parent has weight~$\frac{\oneNorm{\res{\bel{b}}{[\state]}}}{\oneNorm{\bel{b}}}$.
\label{unr:dist comp}
\item (Case: $\bel{b}$ such that $\supp{\bel{b}}$ belongs to the domain of some \adscc{} and the maximal \adscc{} $\sccFunc$ containing $\supp{\bel{b}}$ (denoted $\sccFunc_{\supp{\bel{b}}}$) is non-distinguishing.)
Applying \Cref{thm:nd adscc}, we create children labelled by pairs $(\bel{b}',a)$ for $\bel{b}'\in \reach{\sccFunc}(\bel{b})$ and $\action\notin \sccFunc(\supp{\bel{b}'})$.
The edge connecting them to their parent has weight~$1$.
\label{unr:undist comp}
\item (Case: All remaining cases, i.e., $\supp{\bel{b}}$ belongs to the domain of no \adscc.)
\label{unr:rest}
Applying \eqref{eq:val bel upd}, for all actions $\action$ we create a child labelled by the pair $(\bel{b},\action)$.
The edge connecting them to their parent has weight~$1$.
\end{reflist}
\end{definition}
By \Cref{lem:maximal disjoint}, the cases of \ref{unr:dist comp}, \ref{unr:undist comp}, and \ref{unr:rest} are mutually exclusive and cover all non-zero sub-beliefs.
When \ref{unr:undist comp} is applied, $\supp{\bel{b}}\notin\set{\set{\target},\set{\bad}}$ (otherwise \ref{unr:bot} would apply), so the \adscc~$\sccFunc$ is non-bottom by \Cref{lem:trivial secs}, and the node has at least one child.
The weights of the edges from a node unfolded by \ref{unr:act} or \ref{unr:dist comp} to its children sum to~$1$: for \ref{unr:act} since $\sum_{\obs}\ptran(\obs\mid\normalised{\bel{b}},\action) = 1$, and for \ref{unr:dist comp} since the equivalence classes partition $\supp{\bel{b}}$.
Finally, all children of a node are labelled by sub-beliefs with the same mass as the label of the node, except the child of a node unfolded by \ref{unr:delmin}.
\newcommand{\treex}{\widehat{\tree}}
Let \emphdef{$\treex(\bel{b})$} denote the (possibly infinite) tree whose root is labelled by $\bel{b}$ and is unfolded using the above rules.
Let \emphdef{$\treex_n(\bel{b})$} denote its truncation at depth $n$. 
The value-approximation algorithm consists in computing statistics on
$\treex_n(\bel{b})$ for a large value of $n$ (depending on $\pdp$ and
$\epsilon$), which allows to bound the error of the approximation
of~$\preach(\bel{b})$.

To conclude, we can bound the depth of $\treex_n(\bel{b})$ by a double
exponential on the size of $\pdp$, thus the size of $\treex_n(\bel{b})$ is
triply exponential. Hence, using \Cref{eq:val bel act} and \Cref{eq:val bel upd} in a
backward-induction fashion yields the advertised complexity of 3EXPTIME.
In more detail, the depth~$n_\epsilon$ that guarantees an $\epsilon$-approximation is given in \Cref{thm:err dec}; it is doubly exponential in the size of~$\pdp$ and in the size of the binary encoding of~$\epsilon$.
The branching factor of the tree is bounded as follows: a node unfolded by \ref{unr:act} has at most $|\sig|$ children, by \ref{unr:dist comp} at most $|\states|$, by \ref{unr:rest} exactly $|\actions|$, and by \ref{unr:undist comp} at most $|\actions|\cdot|\states|!$.
For the last bound, every belief in $\reach{\sccFunc}(\bel{b})$ is obtained from~$\bel{b}$ by renaming the states of $\supp{\bel{b}}$ through an injective map from $\supp{\bel{b}}$ to $\states$ (\Cref{lem:non-dist reach}), and there are at most $|\states|!$ such maps.
The branching factor is thus at most $\max(|\sig|, |\actions|\cdot|\states|!)$, which is exponential in the size of~$\pdp$, and the tree $\treex_{n_\epsilon}(\bel{b})$ has at most $\max(|\sig|, |\actions|\cdot|\states|!)^{n_\epsilon+1}$ nodes, a triply exponential number.
The labels are sub-beliefs with rational entries.
We represent a label $\bel{c}$ by its mass $\oneNorm{\bel{c}}$ together with a non-negative rational vector $\vec{v}$ such that $\bel{c} = \oneNorm{\bel{c}}\cdot\frac{\vec{v}}{\oneNorm{\vec{v}}}$.
\ref{unr:act} multiplies $\vec{v}$ by a matrix $\ptran_{\action,\obs}$; \ref{unr:dist comp} restricts $\vec{v}$ to an equivalence class; \ref{unr:undist comp} renames the entries of $\vec{v}$ through an injective map (\Cref{lem:non-dist reach}); and all these rules keep the mass unchanged.
Finally, \ref{unr:delmin} restricts $\vec{v}$ to a subset $K$ of states and multiplies the mass by $\oneNorm{\res{\vec{v}}{K}}/\oneNorm{\vec{v}}$; this happens at most $|\states|$ times along each branch, since it strictly decreases the size of the support, which never increases.
Hence the bit size of $\vec{v}$ grows by an amount polynomial in the size of~$\pdp$ at each unfolding step, and the mass changes at most $|\states|$ times along each branch, each time by a multiplication with a rational number whose bit size is linear in that of~$\vec{v}$; both bit sizes are therefore at most doubly exponential.
Finally, the maximal \adsccs{}, the indistinguishability relations, and the sets $\reach{\sccFunc}(\bel{b})$ only depend on the belief-support MDP (for the latter, on the pairs made of a belief support and an injective renaming of $\supp{\bel{b}}$), which has exponentially many states; they can be computed in exponential time by standard graph algorithms.
Hence the tree and its statistics (\Cref{sec:correctness}) can be computed in triply exponential time.

\section{Correctness and termination of the algorithm} \label{sec:correctness}
In this section, we complete the description of the algorithm by formalising how we provide concrete bounds from the tree $\treex_n(\bel{b})$ presented in \Cref{sec:complete algo}, and we then prove the correctness and termination of the algorithm.
We take here \Cref{thm:dis adscc,thm:nd adscc} as given; their proofs are in \Cref{sec:ec}.

\noindent\textit{Roadmap and intuition.}
The algorithm computes two statistics on the finite tree $\treex_n(\bel{b})$: $\eval_n$, the value of the tree when all leaves of unknown value are evaluated pessimistically, and $\erank_n$, which measures how much is still unknown at the leaves.
\Cref{thm:correctness} (correctness) states that, for every~$n$, $\eval_n$ and $\eval_n + \erank_n$ (plus the error introduced by cuts) are lower and upper bounds on the value.
\Cref{thm:err dec} (termination) states that $\erank_n$ of the root tends to~$0$ as $n$ grows, at an explicit rate.
Its main step, \Cref{clm:err dec}, shows that a fixed number of additional unfolding steps decreases $\erank$ by a constant factor.
The rank of a node is the mass of its belief multiplied by the height of its belief support in a partial order, in which the belief support decreases when its size decreases, and along transitions that cannot be undone.
To prove \Cref{clm:err dec}, we fix the choices made at the max nodes (intuitively, a strategy), and we show that, within $2^{|\states|+1}$ unfolding steps, the resulting subtree contains a node of non-negligible weight at which the rank has decreased: a leaf of known value, a cut, or a split.
If this were not the case, the actions chosen in the subtree would define an \adscc{} contradicting the maximality of the \adsccs{} used by the exit operation (\Cref{clm:smaller node}).

\noindent\textit{Generality.}
The belief MDP, the belief-support MDP, \adsccs, the cut operation, and the naive unfolding of \Cref{sec:intuition} are defined for arbitrary POMDPs, but the smarter unfolding $\treex_n(\bel{b})$, and hence the statistics $\eval_n$ and $\erank_n$, are not: the indistinguishability relation of \Cref{def:indistinguishable}, which is used by \ref{unr:dist comp} and \ref{unr:undist comp}, relies on the states $\transitionIter(\state,\history)$, which are uniquely defined only in posterior-deterministic POMDPs.
Given the tree, the proof of \Cref{thm:correctness} only uses \eqref{eq:val bel act}, \eqref{eq:val bel upd}, \Cref{lem:subbelief val}, \Cref{thm:dis adscc,thm:nd adscc}, and \Cref{lem:supp dec}.
Besides the definition of indistinguishability, posterior-determinism is used in three places:
(a) in the proofs of \Cref{thm:dis adscc,thm:nd adscc} and \Cref{cor:non dist fin reach} (\Cref{sec:ec}), which rely on the fact that, inside an \adscc{}, every observable history $\history$ acts as a bijection from a belief support $S$ onto $\transitionIter(S,\history)$ (\Cref{lem:sec basic});
(b) through the fact that belief supports never grow (\Cref{lem:supp dec}), which bounds the total error introduced by cuts along a branch (\Cref{thm:correctness}) and makes the order $\ttot$ below a partial order compatible with the unfolding; and
(c) in the proof of \Cref{clm:err dec}, through the fact that the beliefs reachable inside a non-distinguishing \adscc{} are obtained by renaming states (\Cref{lem:non-dist reach}), so that no small probability is created without a cut.

Let $\pdp = \pdpFull$ be a \deterministic POMDP and $\treex_n(\bel{b})$ be the tree defined as above from a belief $\bel{b}$.
We call inner nodes of $\treex_n(\bel{b})$ unfolded using \ref{unr:act} or \ref{unr:dist comp} \emphdef{sum nodes}.
The inner nodes unfolded using \ref{unr:delmin} are called \emphdef{cut nodes}.
Finally, the inner nodes unfolded using \ref{unr:undist comp} or \ref{unr:rest} are called \emphdef{max nodes}.
Nodes labelled by a sub-belief are called belief nodes, and nodes labelled by a pair $(\bel{c},\action)$ are called belief-action nodes; the latter are exactly the nodes unfolded by \ref{unr:act} (or leaves at depth~$n$).

\subsection{Rank and tree statistics}
\subparagraph{Rank.}
We now define the \emph{rank}, which we use as an error term.
We show that it gives an upper bound on the value (\Cref{thm:correctness}), and that the resulting error at the root of the tree goes to $0$ when $n$ goes to infinity (\Cref{thm:err dec}).%
\begin{definition}\label{def:reach qo}
Define a quasi-order \emphdef{$\tord$} on $\bsup{\pdp}$ as
\[
T \tord S
\iff
\text{there exists $\history\in (\ac \times \sig)^*$ such that
$\transitionIter(S,\history) = T$}\ .
\]
This quasi-order induces the equivalence relation \emphdef{$\sim$} defined as:
\begin{equation}\label{eq:equiv odscc}
S \sim T \quad\iff\quad S \tord T \text{ and } T\tord S\ .
\end{equation}
The equivalence class of a belief support $S$ is denoted by $[S]_{\sim}$.
The quasi-order $\preceq$ defines a partial order on the equivalence classes of $\sim$,
which is also denoted by $\preceq$.
\end{definition}
\begin{lemma}\label{fact:size dec}
If $T \tord S$, then $|T| \leq |S|$.
\end{lemma}
\begin{proof}
Follows from $\pdp$ being \deterministic{} and \Cref{lem:supp dec}.
\end{proof}
Keeping this lemma in mind, we extend the quasi-order $\tord$ as follows.
\begin{definition}
Define a quasi-order \emphdef{$\ttot$} on $\bsup{\pdp}$ as
\[
S \ttot T \quad\iff\quad
S \tord T \text{ or } |S| < |T| \ .
\]
\end{definition}
\Cref{fact:size dec} implies that, for every $S\sim T$, we have $|S| = |T|$.
Therefore, $\ttot$ induces the same equivalence classes as $\tord$.
Moreover, it induces a partial order on the equivalence classes of $\sim$ that extends $\tord$.
\begin{definition}
The \emphdef{height} of an element~$x$ of a finite partially ordered set is the maximal length~$k$ of a chain $x_0 < x_1 < \dots < x_k = x$; in particular, minimal elements have height~$0$, and if $y < x$, then the height of $y$ is smaller than the height of~$x$.
The \emphdef{$\rank$} of an equivalence class of $\sim$ is its height in the partial order induced by $\ttot$ on the equivalence classes of $\sim$.
The $\rank$ of a belief support is the rank of its equivalence class for $\sim$.
The $\rank$ of a non-zero sub-belief $\bel{b}$ is defined as $\oneNorm{\bel{b}}\cdot\rank(\supp{\bel{b}})$, i.e., the rank of its support weighted by the total probability mass of the belief; we let $\rank(0) = 0$.
\end{definition}
In particular, relying on the assumptions of \Cref{lem:single target}, $\rank(\bel{b}) = 0$ if and only if $\supp{\bel{b}}$ is either $\set{\target}$ or $\set{\bad}$ (or $\bel{b} = 0$).
Indeed, the classes of $\set{\target}$ and $\set{\bad}$ are singletons and minimal, since only $\set{\target}$ (resp.\ $\set{\bad}$) is reachable from $\set{\target}$ (resp.\ $\set{\bad}$); every other belief support~$S$ is strictly above $\set{\target}$, either because $|S|\ge 2$, or because $S = \set{\state}$ for a state $\state\neq\bad$ with positive value, from which $\target$ is reachable.
The rank is upper-bounded by the height of the partial order induced by~$\ttot$, and is thus at most $2^{|\states|} - 1$ (which is the cardinality of $\bsup{\pdp}$).

\subparagraph{Tree statistics.}
We define two statistics, $\eval_n$ and $\erank_n$, for all nodes of $\treex_n(\bel{b})$.
These are used to lower and upper bound the value of the belief associated with the node (upcoming formal statement in \Cref{thm:correctness}).
These statistics are defined inductively on the structure of the tree as follows.
Recall that nodes in the tree are sometimes labelled by \emph{sub-}beliefs due to the cut moves, hence the use of $\oneNorm{\belof(x)}$ (which is not always~$1$) in the definitions below.

For a leaf $x$ in $\treex_n(\bel{b})$,
\begin{align*}
\eval_n(x) &= 
\begin{cases}
\oneNorm{\belof(x)} & \text{if $\supp{\belof(x)} = \set{\target}$}\\
0 & \text{otherwise,}
\end{cases}
\quad\text{ and }\\
\erank_n(x) &= 
\rank(\belof(x))\ . %
\end{align*}
Observe that $\erank_n(x)$ is $0$ if $\supp{\belof(x)}$ is either $\set{\target}$ or $\set{\bad}$ (or $\belof(x) = 0$), and is at least $\oneNorm{\belof(x)}$ otherwise.

For a sum node $x$,
\begin{align*}
\eval_n(x) &= \sum_{\text{child $y$ of $x$}} p_y\cdot\eval_n(y)\quad\text{ and }\\
\erank_n(x) &= \sum_{\text{child $y$ of $x$}} p_y\cdot\erank_n(y)\ ,
\end{align*}
where $p_y$ is the weight of the edge connecting the child $y$ to its parent $x$.

For a cut node $x$ with $y$ being the only child of $x$, $\eval_n(x)$ (resp.\ $\erank_n(x)$) is equal to $\eval_n(y)$ (resp.\ $\erank_n(y)$).

For a max node $x$,
\begin{align*}
\eval_n(x)  &= \max_{\text{child $y$ of $x$}} \eval_n(y) \quad\text{ and }\\
\erank_n(x) &= \max_{\text{child $y$ of $x$}} \erank_n(y) \ .
\end{align*}
\subsection{Correctness of the algorithm}\label{subsec:all good}
We now prove that we can obtain lower and upper bounds on the value of the belief associated with any node of $\treex_n(\bel{b})$ using $\eval_n$ and $\erank_n$, which means our algorithm is sound.
We extend the value to the nodes of $\treex_n(\bel{b})$: for a belief node~$x$, we let $\preach(x) = \preach(\belof(x))$, and for a belief-action node~$x$ labelled by $(\bel{c},\action)$, we let $\preach(x) = \preach(\bel{c},\action)$.
\begin{theorem}[Correctness] \label{thm:correctness}
For all $n\in\N$ and all nodes $x$ of $\treex_n(\bel{b})$,
\[
\eval_n(x) \le \preach(x) \leq \eval_n(x) + \erank_n(x) + |\supp{\belof(x)}|\cdot\thr\ .
\]
\end{theorem}
For the root, $\preach(\rootof(\treex_n(\bel{b}))) = \preach(\bel{b})$.
Note that, for a belief-action node $x$ labelled by $(\bel{c},\action)$, the value $\preach(\belof(x)) = \preach(\bel{c})$ may be larger than $\preach(\bel{c},\action)$ when $\action$ is not an optimal action; this is why the statement uses $\preach(x)$.
\begin{proof}
  The proof of this theorem is a simple induction on the structure of the tree.

  Let $x$ be a leaf of $\treex_n(\bel{b})$.
  Assume first that $x$ is a belief node.
  If $\supp{\belof(x)}$ is $\set{\target}$ (resp.~$\set{\bad}$, or $\belof(x) = 0$), then $\eval_n(x) = \preach(\belof(x)) = \oneNorm{\belof(x)}$ (resp.\ $0$), so the claim holds.
  Otherwise, let $\bel{b} = \belof(x)$.
  We have $\eval_n(x) = 0$ and $0 \le \preach(\bel{b}) \le \oneNorm{\bel{b}} \le \oneNorm{\bel{b}}\cdot\rank(\supp{\bel{b}}) = \erank_n(x)$, so the claim holds as well.
  This covers case \ref{unr:bot} of \Cref{def:unfold}.
  If $x$ is a belief-action node labelled by $(\bel{c},\action)$, then $\supp{\bel{c}}\notin\set{\set{\target},\set{\bad}}$: indeed, $x$ was created by \ref{unr:undist comp} or \ref{unr:rest} from a belief node to which \ref{unr:bot} does not apply, and in the case of \ref{unr:undist comp}, $\supp{\bel{c}}$ belongs to the domain of a maximal \adscc{} different from the two trivial ones (\Cref{lem:trivial secs}). Hence $\eval_n(x) = 0 \le \preach(\bel{c},\action)\le\oneNorm{\bel{c}}\le\rank(\bel{c}) = \erank_n(x)$.

  Let $x$ be an inner node of $\treex_n(\bel{b})$.
  Assume by induction that the claim holds for all children of $x$.
  We now distinguish cases according to its type: sum, cut, or max node.
  \begin{itemize}
  \item If $x$ is a sum node, it is unfolded using \ref{unr:act} (resp.\ \ref{unr:dist comp}).
  In this case, $\eval_n(x)$, $\preach(x)$, and $\erank_n(x)$ are obtained as weighted sums of the corresponding values of the children of $x$, with weights summing to~$1$: by definition for $\eval_n$ and $\erank_n$, and by \eqref{eq:val bel act} (resp.\ \Cref{thm:dis adscc} together with \Cref{lem:subbelief val}, since, writing $\bel{c} = \belof(x)$, we have $\preach(\res{\bel{c}}{C}) = \frac{\oneNorm{\res{\bel{c}}{C}}}{\oneNorm{\bel{c}}}\cdot\preach(\frac{\oneNorm{\bel{c}}}{\oneNorm{\res{\bel{c}}{C}}}\cdot\res{\bel{c}}{C})$ for every equivalence class~$C$) for $\preach$.
  Moreover, for all children $y$ of $x$, we have $|\supp{\belof(y)}| \le |\supp{\belof(x)}|$ by \Cref{lem:supp dec}, \Cref{lem:possible hist}, and the definition of \ref{unr:dist comp}.
    The claim then follows using the induction hypothesis on all children of $x$.
  \item If $x$ is a cut node, it is unfolded using \ref{unr:delmin}.
  Let $y$ be its unique child. Let $k$ be the number of states $\state \in \supp{\belof(x)}$ such that $\belof(x)(\state) < \thr$, i.e., the number of states removed by~$\delmin_\thr$.
    We have by definition that $\eval_n(x) = \eval_n(y)$ and $\erank_n(x) = \erank_n(y)$.

    For the left inequality, observe that $\preach(\belof(y)) \le \preach(\belof(x))$ since $\belof(y)$ is obtained from $\belof(x)$ by removing some probability mass (\Cref{lem:subbelief val}).
    Hence, by induction hypothesis on $y$,
    \[\eval_n(x) = \eval_n(y) \le \preach(\belof(y)) \le \preach(\belof(x))\ .\]

    For the right inequality, observe that, by \Cref{lem:subbelief val} and the definition of $\delmin_\thr$,
    $\preach(\belof(x)) \le \preach(\belof(y)) + k \cdot \thr$.
    Using the induction hypothesis on $y$, we thus obtain
    \begin{align*}
    \preach(\belof(x))
    &\le \eval_n(y) + \erank_n(y) + |\supp{\belof(y)}| \cdot \thr + k \cdot \thr\\
    &= \eval_n(x) + \erank_n(x) + |\supp{\belof(x)}| \cdot \thr\ .
    \end{align*}
  \item If $x$ is a max node, it is unfolded using \ref{unr:undist comp} or \ref{unr:rest}.
  In this case, $\eval_n(x)$, $\preach(x)$, and $\erank_n(x)$ are defined as maxima of the corresponding values of the children of $x$: by definition for $\eval_n$ and $\erank_n$, and by \Cref{thm:nd adscc} (resp.\ \eqref{eq:val bel upd}) for $\preach$, since the children of $x$ are belief-action nodes.
    Moreover, for all children $y$ of $x$, we have $|\supp{\belof(y)}| \le |\supp{\belof(x)}|$ by \Cref{lem:supp dec}.
    The claim then follows using the induction hypothesis on all children of~$x$.
  \qedhere
  \end{itemize}
\end{proof}

\subsection{Termination of the algorithm}
We now show that the upper and lower bounds provided by \Cref{thm:correctness} converge to an $\epsilon$-approximation of the actual value of the belief as $n$ goes to infinity.
Fix~$N = 2^{|\states|+1}$.
Let~\emphdef{$\pmin$} be the smallest non-zero probability occurring in the syntactic description of~$\pdp$, i.e., $\min \set{\ptran(\obs, \state' \mid \state, \action) \mid \ptran(\obs, \state' \mid \state, \action) > 0}$.
Take $c = \frac{(\pmin\cdot\thr)^N}{2^{|\states|}}$.
\begin{theorem}[Termination]\label{thm:err dec}
For all $\epsilon > 0$ and every natural number $n_{\epsilon}$ such that
\[
n_{\epsilon} \ge (N+1)\cdot\left\lceil\frac{1}{c}\cdot\left((|\states|+1)\cdot\ln(2) + \ln\left(\frac{1}{\epsilon}\right)\right)\right\rceil,
\]
we have
\[
  \eval_{n_{\epsilon}}(\rootof(\treex_{n_{\epsilon}}(\bel{b}))) \le \preach(\bel{b}) \le \eval_{n_{\epsilon}}(\rootof(\treex_{n_{\epsilon}}(\bel{b}))) + \epsilon\ .
\]
\end{theorem}
To prove this theorem, our main technical tool is the following lemma, which shows that the $\erank$ of a node always decreases by the multiplicative factor $1-c$ after $N$ additional steps of the tree unfolding.
\begin{lemma}\label{clm:err dec}
Let $\bel{b}$ be a sub-belief of $\pdp$.
Then,
\[
\erank_N(\rootof(\treex_N(\bel{b}))) \leq (1-c)\cdot\rank(\bel{b})\ .
\]
\end{lemma}
We also use the following basic properties of the unfolding.
For $m \leq n$, we consider $\treex_m(\bel{b})$ to be a subtree of $\treex_n(\bel{b})$ with the same root.
Since the rule used to unfold a node only depends on its label, the subtree of $\treex_n(\bel{b})$ rooted at a node~$x$ of depth~$d$ is a copy of the tree $\treex_{n-d}(\ell)$, where $\ell$ is the label of~$x$ and $\treex_{m}(\ell)$ denotes the truncation at depth~$m$ of the tree whose root is labelled by~$\ell$ and is unfolded using the rules of \Cref{def:unfold} (for a label $\ell = (\bel{c},\action)$, the root of this tree is a belief-action node).
In particular, the statistics of~$x$ in $\treex_n(\bel{b})$ coincide with the statistics of the root of $\treex_{n-d}(\ell)$.
\begin{lemma}\label{lem:rank children}
Let $n\in\N$, let $x$ be an inner node of $\treex_n(\bel{b})$, and let $y$ be a child of~$x$.
Then $\oneNorm{\belof(y)}\le\oneNorm{\belof(x)}$, with equality unless $x$ is a cut node, and $\rank(\belof(y))\le\rank(\belof(x))$.
Moreover, if $x$ is a sum node, then the weights of the edges from~$x$ to its children sum to~$1$.
\end{lemma}
\begin{proof}
The statements on masses and weights were observed after \Cref{def:unfold}.
If $\belof(y) = 0$, then $\rank(\belof(y)) = 0$.
Otherwise, it suffices to show that $[\supp{\belof(y)}]_\sim \ttot [\supp{\belof(x)}]_\sim$, which implies $\rank(\supp{\belof(y)})\le\rank(\supp{\belof(x)})$.
If $x$ is unfolded by \ref{unr:act}, then $\supp{\belof(y)} = \transition(\supp{\belof(x)},\action,\obs)$ for some $\action$ and $\obs$ (\Cref{lem:possible hist}), so $\supp{\belof(y)}\tord\supp{\belof(x)}$.
If $x$ is unfolded by \ref{unr:delmin} or \ref{unr:dist comp}, then $\supp{\belof(y)}\subseteq\supp{\belof(x)}$, so either the two supports are equal or $|\supp{\belof(y)}|<|\supp{\belof(x)}|$.
If $x$ is unfolded by \ref{unr:undist comp}, then $\supp{\belof(y)}$ and $\supp{\belof(x)}$ belong to the domain of the same \adscc, so they are $\sim$-equivalent.
If $x$ is unfolded by \ref{unr:rest}, then $\belof(y) = \belof(x)$.
\end{proof}

\begin{lemma}\label{lem:erank monotone}
Let $n\in\N$.
\begin{enumerate}
\item For every node $x$ of $\treex_n(\bel{b})$, we have $\erank_n(x)\le\rank(\belof(x))$. \label{item:erank le rank}
\item For every $m\le n$ and every node $x$ of $\treex_m(\bel{b})$, we have $\erank_n(x)\le\erank_m(x)$. \label{item:erank monotone}
\end{enumerate}
\end{lemma}
\begin{proof}
For~\eqref{item:erank le rank}, we proceed by bottom-up induction on~$x$.
For a leaf, both sides are equal.
For a sum node, $\erank_n(x) = \sum_y p_y\cdot\erank_n(y)\le\sum_y p_y\cdot\rank(\belof(y))\le\rank(\belof(x))$ by induction hypothesis and \Cref{lem:rank children}, since the weights $p_y$ sum to~$1$.
The cases of cut and max nodes are similar.

For~\eqref{item:erank monotone}, we proceed by bottom-up induction on the nodes~$x$ of $\treex_m(\bel{b})$.
If $x$ is a leaf of $\treex_m(\bel{b})$, then $\erank_m(x) = \rank(\belof(x))\ge\erank_n(x)$ by~\eqref{item:erank le rank}.
If $x$ is an inner node of $\treex_m(\bel{b})$, it has the same children in $\treex_m(\bel{b})$ and in $\treex_n(\bel{b})$, and the statistics of~$x$ are obtained from those of its children by weighted sums, copies, or maxima, which are monotone operations.
\end{proof}

We show how to use \Cref{clm:err dec} to prove \Cref{thm:err dec} before proving the lemma itself.
\begin{proof}[Proof of \Cref{thm:err dec} using \Cref{clm:err dec}]
Let $r = \rootof(\treex_{n_{\epsilon}}(\bel{b}))$.
By \Cref{thm:correctness}, we know that 
\[
\eval_{n_{\epsilon}}(r) \le \preach(\bel{b}) \le \eval_{n_{\epsilon}}(r) + \erank_{n_{\epsilon}}(r) + |\supp{\bel{b}}|\cdot\thr\ .
\]
Recall that we defined $\thr = \frac{\epsilon}{2\cdot |\states|}$.
Thus,
\[
|\supp{\bel{b}}|\cdot\thr \leq |\states|\cdot\thr \le \frac{\epsilon}{2} \ .
\]
To obtain \Cref{thm:err dec}, it now suffices to show that $\erank_{n_{\epsilon}}(r) \leq \frac{\epsilon}{2}$.

Let $n\in\N$ and let $x$ be a leaf of $\treex_n(\bel{b})$.
We show that
\begin{equation}\label{eq:leaf dec}
\erank_{n+N+1}(x)\le(1-c)\cdot\erank_n(x) = (1-c)\cdot\rank(\belof(x))\ .
\end{equation}
If $x$ is unfolded by \ref{unr:bot}, then both sides are~$0$.
Otherwise, $x$ is at depth~$n$, and the subtree of $\treex_{n+N+1}(\bel{b})$ rooted at~$x$ is a copy of $\treex_{N+1}(\ell)$, where $\ell$ is the label of~$x$.
If $x$ is a belief node, with $\ell = \bel{d}$, then by \Cref{lem:erank monotone} and \Cref{clm:err dec},
\begin{align*}
\erank_{n+N+1}(x) &= \erank_{N+1}(\rootof(\treex_{N+1}(\bel{d})))\\
&\le \erank_{N}(\rootof(\treex_{N}(\bel{d}))) \le (1-c)\cdot\rank(\bel{d})\ .
\end{align*}
If $x$ is a belief-action node, with $\ell = (\bel{d},\action)$, then its children $y_\obs$ are labelled by the sub-beliefs $\bel{d}_{\action,\obs}$, with weights $p_\obs$ summing to~$1$, and the subtree rooted at $y_\obs$ is a copy of $\treex_N(\bel{d}_{\action,\obs})$; hence, by \Cref{clm:err dec} and \Cref{lem:rank children},
\begin{align*}
\erank_{n+N+1}(x) &= \sum_{\obs} p_\obs\cdot\erank_{N}(\rootof(\treex_{N}(\bel{d}_{\action,\obs})))\\
&\le (1-c)\cdot\sum_{\obs}p_\obs\cdot\rank(\bel{d}_{\action,\obs}) \le (1-c)\cdot\rank(\bel{d})\ .
\end{align*}
This proves~\eqref{eq:leaf dec}.
By definition of $\erank$ for inner nodes (obtained using $\max$, copies, or weighted sums of children), observe that if the $\erank$ of all leaves is multiplied by a factor $1 - c$ (or smaller), then this change is propagated to all nodes in the tree.
Hence, by bottom-up induction on the nodes of $\treex_n(\bel{b})$, we obtain that for all nodes~$x$ of $\treex_{n}(\bel{b})$,
\[
\erank_{n+N+1}(x) \leq (1-c)\cdot\erank_n(x) \ .
\]
Thus, for every $k\in\N$, we have
\[
\erank_{k(N+1)}(\rootof(\treex_{k(N+1)}(\bel{b}))) \leq (1-c)^k \cdot \rank(\bel{b}) \leq (1-c)^k \cdot 2^{|\states|}\ .
\]
Using that $1 - c \le e^{-c}$~\cite[Inequality~142]{ineqHardy}, we obtain that
\[
(1-c)^k \cdot 2^{|\states|} \leq e^{-c\cdot k} \cdot 2^{|\states|} \leq
\frac{\epsilon}{2}
\]
when $k \geq \frac{1}{c}\cdot\left((|\states|+1)\cdot\ln(2) + \ln(\frac{1}{\epsilon})\right)$.
Taking $k = \left\lceil\frac{1}{c}\cdot\left((|\states|+1)\cdot\ln(2) + \ln(\frac{1}{\epsilon})\right)\right\rceil$, we have $k(N+1)\le n_\epsilon$, so \Cref{lem:erank monotone} yields $\erank_{n_\epsilon}(r)\le\erank_{k(N+1)}(r)\le\frac{\epsilon}{2}$.
\end{proof}
\newcommand{\leaves}[1]{\textnormal{\textsf{Leaves}}_{#1}}
The remainder of this section is dedicated to a proof of \Cref{clm:err dec}, the only missing piece to prove the termination of our algorithm.

\begin{proof}[Proof of \Cref{clm:err dec}]
If $\rank(\bel{b}) = 0$, the statement follows from \Cref{lem:erank monotone}.
We assume from now on that $\rank(\bel{b})>0$, so that $\bel{b}\neq 0$ and $\supp{\bel{b}}\notin\set{\set{\target},\set{\bad}}$, hence $\rank(\supp{\bel{b}})\ge 1$.

We define a \emphdef{contributing subtree} of $\treex_n(\bel{b})$ (or of any of its subtrees) to be a subtree $T$ of $\treex_n(\bel{b})$ (resp.\ of the given subtree) that contains:
\begin{enumerate}
\item the root of $\treex_n(\bel{b})$ (resp.\ of the given subtree),
\item all children of sum nodes and cut nodes of~$T$, and
\item exactly one child of every max node of~$T$.
\end{enumerate}
A contributing subtree thus corresponds to fixing the choices of a deterministic strategy up to depth $n$ in the unfolding $\treex_n(\bel{b})$.

\newcommand{\influ}{\mathsf{infl}}
The \emphdef{influence $\influ(x,y)$} of a node $y$ on its ancestor $x$ is defined to be the product of the weights of all edges in the path from $x$ to $y$.
\begin{definition}\label{def:erank contrib}
For $n\in\N$ and
a contributing subtree $T$ of $\treex_n(\bel{b})$,
define $\erank_T(x)$ to be $\rank(\belof(x))$ when $x$ is a leaf,
and otherwise to be the weighted sum
\[
\sum_{\text{$y$ is a child of $x$ in $T$}} \influ(x,y)\cdot\erank_T(y) \ .
\]
\end{definition}
By bottom-up induction on the trees, one can prove that, for every $n\in\N$ and every node $x$ of $\treex_n(\bel{b})$,
\[
\erank_n(x) =
\max_{T\in\mathcal{T}_x}\erank_T(x)\ ,
\]
where $\mathcal{T}_x$ is the (finite) set of contributing subtrees of the subtree of $\treex_n(\bel{b})$ rooted at~$x$ (recall that $\treex_n(\bel{b})$ is finite).
Therefore, to prove \Cref{clm:err dec}, we just need to show that, for all contributing subtrees $T$ of $\treex_N(\bel{b})$, we have
$\erank_T(\rootof(\treex_N(\bel{b})))
\leq (1 - c)\cdot\rank(\bel{b})$.

Pick an arbitrary contributing subtree $T$ of $\treex_N(\bel{b})$.
For $k\in\N$, let $D(k)$ be the set of nodes of~$T$ at depth~$k$, together with the leaves of~$T$ at depth smaller than~$k$; in other words, $D(k)$ is the set of leaves of the tree obtained by truncating~$T$ at depth~$k$.
From the definition of $\erank_T$, it follows by induction on~$k$ that
\begin{equation} \label{eq:erank contrib}
\erank_T(\rootof(\treex_N(\bel{b}))) = \sum_{y\in D(k)} \influ(\rootof(\treex_N(\bel{b})),y)\cdot\erank_T(y) \ .
\end{equation}
Observe that, by construction of $\treex_N(\bel{b})$, the sum of the influences at every fixed depth is $1$ (by induction on~$k$, since the weights of the edges from a sum node to its children sum to~$1$ by \Cref{lem:rank children}, and cut and max nodes have a single child in~$T$, with weight~$1$), that is, for all $k\ge 0$,
\begin{equation} \label{eq:influ sum to 1}
\sum_{y\in D(k)} \influ(\rootof(\treex_N(\bel{b})),y) = 1\ .
\end{equation}
Recall that $c = \frac{(\pmin\cdot\thr)^N}{2^{|\states|}}$.
We prove three claims which we will then combine to conclude the proof of \Cref{clm:err dec}.
\begin{claim} \label{clm:rank rank}
  Let $x$ and $y$ be nodes of $T$ such that $y$ is a descendant of $x$.
We have $\rank(\belof(y)) \leq \rank(\belof(x))$.
\end{claim}
\begin{claimproof}
By \Cref{lem:rank children}, for every child~$y'$ of~$x$, we have both $\oneNorm{\belof(y')}\le\oneNorm{\belof(x)}$ and $\rank(\belof(y'))\leq \rank(\belof(x))$.
By induction, the claim follows for all descendants $y$ of $x$.
\end{claimproof}

\begin{claim}\label{clm:erank rank}
Let $x$ and $y$ be nodes of $T$ such that $y$ is a descendant of $x$.
We have $\erank_T(y) \leq \rank(\belof(x))$.
\end{claim}
Note that we allow $x = y$ in the statement of the claim, which implies that \begin{equation} \label{eq:erank rank same node}
  \erank_T(x) \leq \rank(\belof(x))
\end{equation}
for all nodes $x$ of $T$.
\begin{claimproof}[Proof of \Cref{clm:erank rank}]
For all leaves $z$ of $T$ in the subtree rooted at $x$, we have by definition that $\erank_T(z) = \rank(\belof(z))$.
Using \Cref{clm:rank rank}, we thus have $\erank_T(z) \leq \rank(\belof(x))$.
Using~\Cref{def:erank contrib} and~\eqref{eq:influ sum to 1}, we observe that the $\erank_T$ of a node is a convex combination of the $\erank_T$ of the leaves in its subtree.
Since $y$ is in the subtree rooted at~$x$ and all the leaves $z$ in the subtree rooted at $y$ are such that $\erank_T(z) \leq \rank(\belof(x))$, we obtain that $\erank_T(y) \leq \rank(\belof(x))$.
\end{claimproof}
\begin{claim}\label{clm:smaller node}
There is a node $z$ in $T$ such that
\begin{enumerate}
\item $\influ(\rootof(T),z) \geq (\pmin\cdot\thr)^N$, and
\item $\erank_T(z) \leq \left(1 - \frac{1}{2^{|\states|}}\right)\cdot\rank(\bel{b})$.
\end{enumerate}
\end{claim}
\begin{claimproof}
We prove a slightly stronger statement.
Namely, we show that there is a node $z$ in $T$ such that
\begin{enumerate}
\item no node on the path from $\rootof(T)$ to $z$, except possibly $z$ itself, is a cut node, and
\item $\erank_T(z) \leq \rank(\bel{b}) - \oneNorm{\bel{b}}$.
\end{enumerate}
In this proof, the nodes \emph{strictly before}~$z$ are the nodes of the path from $\rootof(T)$ to~$z$ other than~$z$ itself (so $\rootof(T)$ is strictly before~$z$, unless $z = \rootof(T)$).

\smallskip
\noindent\textit{The statement is stronger than \Cref{clm:smaller node}.}
Let $x$ be a node strictly before~$z$, and let $y$ be the child of~$x$ on the path to~$z$.
Since $x$ has a child, it is not unfolded by \ref{unr:bot}, and since it is not a cut node, it is not unfolded by \ref{unr:delmin}.
We show that $\oneNorm{\belof(x)} = \oneNorm{\bel{b}}$, that all non-zero entries of $\belof(x)$ are at least~$\thr$, and that the weight of the edge from~$x$ to~$y$ is at least $\pmin\cdot\thr$.
\begin{itemize}
\item The first property follows from \Cref{lem:rank children}, since no node strictly before~$x$ is a cut node.
\item If $x$ is a belief node, it is unfolded by \ref{unr:dist comp}, \ref{unr:undist comp}, or \ref{unr:rest}; since \ref{unr:delmin} has priority over these rules, all non-zero entries of $\belof(x)$ are at least~$\thr$.
If $x$ is a belief-action node, labelled by $(\bel{c},\action)$, then $x$ is not the root of~$T$ (which is a belief node), and its parent~$x'$ is a belief node strictly before~$z$, unfolded by \ref{unr:undist comp} or \ref{unr:rest}, so all non-zero entries of $\belof(x')$ are at least~$\thr$.
If \ref{unr:rest} was used, then $\bel{c} = \belof(x')$.
If \ref{unr:undist comp} was used, then $\bel{c}\in\reach{\sccFunc}(\belof(x'))$ for a non-distinguishing \adscc{}~$\sccFunc$, so, by \Cref{lem:non-dist reach}, $\bel{c}$ is obtained from $\belof(x')$ by an injective renaming of the states of its support, and the multisets of non-zero entries of $\bel{c}$ and $\belof(x')$ coincide.
In both cases, all non-zero entries of $\bel{c}$ are at least~$\thr$.
\item If $x$ is unfolded by \ref{unr:undist comp} or \ref{unr:rest}, the weight of the edge from~$x$ to~$y$ is~$1$.
If $x$ is unfolded by \ref{unr:dist comp} and $y$ corresponds to the equivalence class~$C$, the weight is $\frac{\oneNorm{\res{\belof(x)}{C}}}{\oneNorm{\belof(x)}}\ge\frac{\thr}{\oneNorm{\belof(x)}}\ge\thr$.
If $x$ is unfolded by \ref{unr:act}, labelled by $(\bel{c},\action)$, and $y$ is labelled by $\bel{c}_{\action,\obs}$, there is $\state_0\in\supp{\bel{c}}$ with $\ptran(\obs\mid\state_0,\action)>0$, hence $\ptran(\obs\mid\state_0,\action)\ge\pmin$, and the weight is
\[
\ptran(\obs\mid\normalised{\bel{c}},\action) = \sum_{\state\in\supp{\bel{c}}}\frac{\bel{c}(\state)}{\oneNorm{\bel{c}}}\cdot\ptran(\obs\mid\state,\action) \ge \bel{c}(\state_0)\cdot\pmin\ge\thr\cdot\pmin\ .
\]
\end{itemize}
Since $z$ has depth at most~$N$ and $\pmin\cdot\thr\le 1$, we obtain $\influ(\rootof(T),z) \geq (\pmin\cdot\thr)^N$, so the first item of \Cref{clm:smaller node} is satisfied.

By definition, $\rank(\bel{b}) = \oneNorm{\bel{b}}\cdot\rank(\supp{\bel{b}})$.
The number of equivalence classes of~$\sim$ (\Cref{def:reach qo}) is at most $|\bsup{\pdp}| = 2^{|\states|} - 1$.
Therefore, $\rank(\supp{\bel{b}}) \leq 2^{|\states|}$ and $\oneNorm{\bel{b}} \geq \frac{\rank(\bel{b})}{2^{|\states|}}$.
Finally, $\erank_T(z) \leq \rank(\bel{b}) - \oneNorm{\bel{b}} \leq \left(1 - \frac{1}{2^{|\states|}}\right)\cdot\rank(\bel{b})$, so the second item of \Cref{clm:smaller node} is satisfied as well.

\smallskip
\noindent\textit{Proof of the stronger statement.}
We say that a node of~$T$ is \emph{decisive} if it is a leaf of~$T$ unfolded by \ref{unr:bot}, a cut node, or a node unfolded by \ref{unr:dist comp} (a split node).
We first show that a decisive node~$z$ of minimal depth satisfies the two items above, and then that $T$ contains at least one decisive node.

Let $z$ be a decisive node of minimal depth, and let $\bel{b}_z = \belof(z)$.
By minimality, no node strictly before~$z$ is a cut node, so the first item holds and $\oneNorm{\bel{b}_z} = \oneNorm{\bel{b}}$ by \Cref{lem:rank children}; moreover, $\rank(\bel{b}_z)\le\rank(\bel{b})$ by \Cref{clm:rank rank}.
\begin{itemize}
\item If $z$ is a leaf unfolded by \ref{unr:bot}, then $\bel{b}_z\neq 0$ (its mass is $\oneNorm{\bel{b}}>0$), so $\supp{\bel{b}_z}\in\set{\set{\target},\set{\bad}}$ and $\erank_T(z) = \rank(\bel{b}_z) = 0\le\rank(\bel{b})-\oneNorm{\bel{b}}$, since $\rank(\supp{\bel{b}})\ge 1$.
\item If $z$ is a cut node, let $z'$ be its child and $\bel{b}_{z'} = \delmin_\thr(\bel{b}_z) = \belof(z')$ be its label.
Since \ref{unr:bot} does not apply to~$z$, $\rank(\supp{\bel{b}_z})\ge 1$.
If $\bel{b}_{z'}\neq 0$, then $\supp{\bel{b}_{z'}}\subsetneq\supp{\bel{b}_z}$, so the class of $\supp{\bel{b}_{z'}}$ is strictly below the class of $\supp{\bel{b}_z}$ for~$\ttot$, and $\rank(\supp{\bel{b}_{z'}})\le\rank(\supp{\bel{b}_z})-1$.
In all cases, we have
\[
\rank(\bel{b}_{z'})\le\oneNorm{\bel{b}_{z'}}\cdot(\rank(\supp{\bel{b}_z}) - 1)\ ,
\]
and
\begin{align*}
\erank_T(z) &= \erank_T(z') \quad \text{by definition of $\erank_T$ for cut operations}\\
&\leq\rank(\bel{b}_{z'}) \quad \text{by \eqref{eq:erank rank same node}}\\
&\leq \oneNorm{\bel{b}_{z'}}\cdot(\rank(\supp{\bel{b}_z}) - 1) \\
&\leq \oneNorm{\bel{b}_z}\cdot(\rank(\supp{\bel{b}_z}) - 1) \\
&= \rank(\bel{b}_z) - \oneNorm{\bel{b}_z}
\leq \rank(\bel{b}) - \oneNorm{\bel{b}}\ .
\end{align*}
\item If $z$ is unfolded by \ref{unr:dist comp}, let $S_z = \supp{\bel{b}_z}$ and $\sccFunc = \sccFunc_{S_z}$, which is distinguishing.
By \Cref{lem:dist all supports}, every equivalence class~$C$ of $\equiv_{(\sccFunc,S_z)}$ is a strict subset of~$S_z$.
Hence, for every child~$z'$ of~$z$, we have $\oneNorm{\belof(z')} = \oneNorm{\bel{b}_z}$ and $\rank(\supp{\belof(z')})\le\rank(S_z)-1$, so, by \Cref{clm:erank rank}, $\erank_T(z')\le\rank(\belof(z'))\le\oneNorm{\bel{b}_z}\cdot(\rank(S_z)-1)$.
Since the weights of the edges from~$z$ to its children sum up to~$1$, we conclude that
\[
\erank_T(z) \leq \oneNorm{\bel{b}_z}\cdot(\rank(S_z) - 1)
= \rank(\bel{b}_z) - \oneNorm{\bel{b}_z} \le \rank(\bel{b}) - \oneNorm{\bel{b}}\ .
\]
\end{itemize}

It remains to show that $T$ contains a decisive node.
Assume, towards a contradiction, that it does not.
Then every inner node of~$T$ is unfolded by \ref{unr:act}, \ref{unr:undist comp}, or \ref{unr:rest}, and every leaf of~$T$ is at depth~$N$ (a node of depth less than~$N$ that is not unfolded by \ref{unr:bot} has at least one child, see the remarks after \Cref{def:unfold}).
Along every branch of~$T$, belief nodes (unfolded by \ref{unr:undist comp} or \ref{unr:rest}) and belief-action nodes (unfolded by \ref{unr:act}) alternate, starting with the root.
In~$T$, every belief node has exactly one child, and every belief-action node labelled by $(\bel{c},\action)$ has one child labelled by $\bel{c}_{\action,\obs}$ for every observation~$\obs$ such that $\transition(\supp{\bel{c}},\action,\obs)\neq\emptyset$ (\Cref{lem:possible hist}).
For a belief node~$x$, let $S_x = \supp{\belof(x)}$; for a belief-action node~$w$ labelled by $(\bel{c},\action)$, let $S_w = \supp{\bel{c}}$ and $\action_w = \action$.
Since no cut occurs in~$T$, all labels have mass $\oneNorm{\bel{b}}>0$, so these belief supports are non-empty.

\emph{Folding $T$ into a finite graph.}
Let $M = |\bsup{\pdp}| = 2^{|\states|}-1$, so that $N = 2M+2$.
Call a belief node~$y$ of~$T$ a \emph{repetition} if some belief node~$x$ that is a proper ancestor of~$y$ satisfies $S_{x} = S_{y}$.
Let $K$ be the set of nodes~$u$ of~$T$ such that no node on the path from $\rootof(T)$ to~$u$ (including~$u$) is a repetition.
On the path from the root to a belief node $u\in K$, the belief nodes have pairwise distinct supports and lie at depths $0,2,\dots$, up to the depth of~$u$; hence $u$ has depth at most $2M-2$, and every node of~$K$ has depth at most $2M-1<N$.
In particular, every node of~$K$ is an inner node of~$T$.
We define a finite directed graph~$G$ with set of vertices~$K$: for every $u\in K$ and every child~$v$ of~$u$ in~$T$, if $v\in K$, then $G$ has an edge from~$u$ to~$v$; otherwise, $v$ is a repetition (as $u\in K$), hence a belief node, and $G$ has an edge from~$u$ to a belief node~$x$ that is a proper ancestor of~$v$ such that $S_{x} = S_{v}$ (note that $x\in K$, since $x$ lies on the path from the root to~$u$).
Every vertex of~$G$ has at least one successor, and, by the rules of \Cref{def:unfold}:
\begin{enumerate}[(E1)]
\item if $u$ is a belief vertex unfolded by \ref{unr:rest}, its unique successor~$w$ is a belief-action vertex with $S_{w} = S_{u}$;
\item if $u$ is a belief vertex unfolded by \ref{unr:undist comp}, its unique successor~$w$ is a belief-action vertex with $S_{w}\in\dom{\sccFunc_{S_{u}}}$ and $\action_{w}\notin\sccFunc_{S_{u}}(S_{w})$;
\item if $w$ is a belief-action vertex, with $S = S_{w}$ and $\action = \action_{w}$, then every successor of~$w$ is a belief vertex~$x$ such that $S_{x} = \transition(S,\action,\obs)$ for some observation~$\obs$, and, for every observation~$\obs$ with $\transition(S,\action,\obs)\neq\emptyset$, $w$ has such a successor.
\end{enumerate}
Since $G$ is finite and every vertex has a successor, $G$ has a \emph{bottom strongly connected component}, i.e., a non-empty set~$\mathcal{B}$ of vertices such that every successor of a vertex of~$\mathcal{B}$ is in~$\mathcal{B}$, and there is a path inside~$\mathcal{B}$ between any two vertices of~$\mathcal{B}$.
By (E1)--(E3), $\mathcal{B}$ contains both belief vertices and belief-action vertices.

\emph{An \adscc{} built from~$\mathcal{B}$.}
Let $g_0$ be the partial function whose domain is the set of belief supports $S_{w}$ for the belief-action vertices~$w\in \mathcal{B}$, defined by
\[
g_0(S) = \setof{\action_{w}}{w\in \mathcal{B} \text{ is a belief-action vertex with } S_{w} = S}\ .
\]
Let $x_1,\dots,x_m$ be the belief vertices of~$\mathcal{B}$ unfolded by \ref{unr:undist comp}, and let $g'$ be the union (as defined before \Cref{lem:maximal disjoint}) of $g_0$ and of the maximal \adsccs{} $\sccFunc_{S_{x_1}},\dots,\sccFunc_{S_{x_m}}$.
We show that $g'$ is an \adscc.
First, $g'(S)\neq\emptyset$ for all $S\in\dom{g'}$.
Second, every belief vertex $x\in \mathcal{B}$ satisfies $S_{x}\in\dom{g'}$: if $x$ is unfolded by \ref{unr:rest}, then $S_{x} = S_{w}\in\dom{g_0}$ for its successor $w\in \mathcal{B}$ by (E1); if $x$ is unfolded by \ref{unr:undist comp}, then $S_{x}\in\dom{\sccFunc_{S_{x}}}$.
\begin{itemize}
\item Closure. Let $S\in\dom{g'}$, $\action\in g'(S)$, and $\obs\in\sig$ such that $\transition(S,\action,\obs)\neq\emptyset$.
If $S\in\dom{\sccFunc_{S_{x_i}}}$ and $\action\in\sccFunc_{S_{x_i}}(S)$ for some~$i$, then $\transition(S,\action,\obs)\in\dom{\sccFunc_{S_{x_i}}}\subseteq\dom{g'}$.
Otherwise, $\action\in g_0(S)$, so there is a belief-action vertex $w\in \mathcal{B}$ with $S_{w} = S$ and $\action_{w} = \action$; by (E3), $w$ has a successor~$x$ with $S_{x} = \transition(S,\action,\obs)$, and $x\in \mathcal{B}$ since $\mathcal{B}$ is bottom, so $\transition(S,\action,\obs) = S_{x}\in\dom{g'}$.
\item Strong connectivity. For every edge from~$u$ to~$v$ of~$G$ with $u,v\in \mathcal{B}$, we have $S_{v}\in\reach{g'}(S_{u})$: in case (E1), $S_{v} = S_{u}$; in case (E2), $S_{u}$ and $S_{v}$ belong to the domain of the \adscc{} $\sccFunc_{S_{u}}$, which is contained in~$g'$, so $S_{v}\in\reach{\sccFunc_{S_{u}}}(S_{u})\subseteq\reach{g'}(S_{u})$; in case (E3), $S_{v} = \transition(S_{u},\action_{u},\obs)\neq\emptyset$ for some~$\obs$, with $\action_{u}\in g_0(S_{u})\subseteq g'(S_{u})$, so the history $\action_{u}\obs$ is inside~$g'$ from~$S_{u}$.
Since $\mathcal{B}$ is strongly connected and $\reach{g'}$ is transitive (\Cref{lem:sec basic}), $S_{v}\in\reach{g'}(S_{u})$ for all $u,v\in \mathcal{B}$.
Moreover, every $S\in\dom{g'}$ is mutually reachable inside~$g'$ with $S_{u}$ for some $u\in \mathcal{B}$: if $S\in\dom{g_0}$, then $S = S_{w}$ for some $w\in \mathcal{B}$; if $S\in\dom{\sccFunc_{S_{x_i}}}$, then $S$ and $S_{x_i}$ are mutually reachable inside $\sccFunc_{S_{x_i}}$, hence inside~$g'$.
Hence all elements of $\dom{g'}$ are mutually reachable inside~$g'$.
\end{itemize}
Therefore, $g'$ is an \adscc.

\emph{Contradiction.}
Let $x$ be a belief vertex of~$\mathcal{B}$, and let $w\in \mathcal{B}$ be its successor.
If $x$ is unfolded by \ref{unr:rest}, then $S_{x}$ belongs to the domain of no \adscc, which contradicts $S_{x}\in\dom{g'}$.
If $x$ is unfolded by \ref{unr:undist comp}, then $x = x_i$ for some~$i$, and the maximal \adscc{} $\sccFunc = \sccFunc_{S_{x}}$ is contained in the \adscc{}~$g'$, so $g' = \sccFunc$.
However, $\action_{w}\in g_0(S_{w})\subseteq g'(S_{w}) = \sccFunc(S_{w})$, which contradicts (E2).
Hence $T$ contains a decisive node, which concludes the proof of \Cref{clm:smaller node}.
This argument only uses the existence of a bottom strongly connected component in a finite directed graph, and no limits of infinite trees.
\end{claimproof}

To conclude the proof of \Cref{clm:err dec}, we use \Cref{clm:erank rank} and \Cref{clm:smaller node}.
Let $z$ be the node given by \Cref{clm:smaller node}.
Let $k$ be the depth of $z$ in $\treex_N(\bel{b})$.
We have
\begin{align*}
&\erank_T(\rootof(\treex_N(\bel{b})))\\
&= \sum_{y\in D(k)} \influ(\rootof(\treex_N(\bel{b})),y)\cdot\erank_T(y)\\
&= \influ(\rootof(\treex_N(\bel{b})),z)\cdot\erank_T(z) + \sum_{\substack{y\in D(k)\\y\neq z}} \influ(\rootof(\treex_N(\bel{b})),y)\cdot\erank_T(y) \\
&\leq \influ(\rootof(\treex_N(\bel{b})),z)\cdot\left(1 - \frac{1}{2^{|\states|}}\right)\cdot\rank(\bel{b}) + \rank(\bel{b})\cdot\sum_{\substack{y\in D(k)\\y\neq z}} \influ(\rootof(\treex_N(\bel{b})),y)\\
&= \rank(\bel{b})-\frac{\influ(\rootof(\treex_N(\bel{b})),z)}{2^{|\states|}}\cdot \rank(\bel{b})\\
&= \left(1 - \frac{\influ(\rootof(\treex_N(\bel{b})),z)}{2^{|\states|}}\right)\cdot\rank(\bel{b})\ ,
\end{align*}
where the first equality holds by~\eqref{eq:erank contrib}, the second equality is just a splitting of the sum, and the following inequality holds by \Cref{clm:smaller node} for $z$ and by \Cref{clm:erank rank} for all other $y$'s.
The following equalities are just rearrangements of the sum and using that the influences at a given level sum to~$1$ by~\eqref{eq:influ sum to 1}.

Since $\influ(\rootof(\treex_N(\bel{b})),z) \geq (\pmin\cdot\thr)^N$ by \Cref{clm:smaller node}, we conclude that
\[
\erank_T(\rootof(\treex_N(\bel{b}))) \leq \left(1 - \frac{(\pmin\cdot\thr)^N}{2^{|\states|}}\right)\cdot\rank(\bel{b}) = (1 - c)\cdot\rank(\bel{b})\ ,
\]
as desired to prove \Cref{clm:err dec}.
\end{proof}

\section{Support end components in \deterministic POMDPs}\label{sec:ec}
Recall that we have defined \adsccs in \Cref{def:sec}.
In this section, we study the properties of SECs in \deterministic POMDPs.
When we are in an \adscc{}~$\sccFunc$, by only taking actions prescribed by~$\sccFunc$, we guarantee that we stay in the \adscc.
In \Cref{sec:basic ec}, we prove that the union of two \adsccs{} with intersecting domains is an \adscc, which gives the proof of \Cref{lem:maximal disjoint}, and we describe the bottom \adsccs{} of POMDPs satisfying the assumptions of \Cref{lem:single target}.
\Cref{sec:distinguishing ec basics} gathers properties of the indistinguishability relation, and \Cref{sec:non-distinguishing,subsec:aPosteriori,sec:distinguishing} characterise the \adsccs{} in which it is possible to eventually gain information (at the limit) about the current state while staying in the \adscc{}, and prove \Cref{thm:dis adscc,thm:nd adscc}.

\subsection{Basic properties of SECs} \label{sec:basic ec}
Let $\pdp = \pdpFull$ be a \deterministic POMDP.
Observable histories inside an \adscc, the union of partial functions, and the basic properties of \adsccs{} (\Cref{lem:sec basic}) were defined and established in \Cref{sec:justifying SECs}, where they are first used.

\begin{lemma} \label{lem:adscc union}
  The union of two \adsccs{} is also an \adscc{} if and only if their domains
  have a non-empty intersection.
\end{lemma}
\begin{proof}
Let $\sccFunc$ and $\sccFunc'$ be two \adsccs, and let $g$ be their union.
Assume first that there is $S_0\in\dom{\sccFunc}\cap\dom{\sccFunc'}$.
Condition~(1) of \Cref{def:sec} clearly holds for~$g$.
For condition~(2), let $S\in\dom{g}$, $\action\in g(S)$, and $\obs\in\sig$ with $\transition(S,\action,\obs)\neq\emptyset$; then $S\in\dom{\sccFunc}$ and $\action\in\sccFunc(S)$, or $S\in\dom{\sccFunc'}$ and $\action\in\sccFunc'(S)$; in the first case, $\transition(S,\action,\obs)\in\dom{\sccFunc}\subseteq\dom{g}$, and similarly in the second case.
For condition~(3), let $S,T\in\dom{g}$.
Since $\sccFunc$ and $\sccFunc'$ are contained in~$g$, histories inside $\sccFunc$ or $\sccFunc'$ are inside~$g$.
As $S$ and $S_0$ belong to the domain of the same \adscc{} ($\sccFunc$ or $\sccFunc'$), we have $S_0\in\reach{g}(S)$; similarly, $T\in\reach{g}(S_0)$; hence $T\in\reach{g}(S)$ by \Cref{lem:sec basic}.

Conversely, assume that $\dom{\sccFunc}\cap\dom{\sccFunc'} = \emptyset$, and let $S\in\dom{\sccFunc}$ and $T\in\dom{\sccFunc'}$.
For every $U\in\dom{\sccFunc}$, we have $U\notin\dom{\sccFunc'}$, so $g(U) = \sccFunc(U)$, and condition~(2) for~$\sccFunc$ implies that every history inside~$g$ from~$U$ leads to a belief support in $\dom{\sccFunc}$.
By induction, every belief support in $\reach{g}(S)$ belongs to $\dom{\sccFunc}$, so $T\notin\reach{g}(S)$, and $g$ is not an \adscc.
\end{proof}

With this in mind, we can now prove \Cref{lem:maximal disjoint} (restated here for clarity).
\maximalDisjoint*
\begin{proof}
    For~(2), let $S$ belong to the domain of some \adscc, and let $\sccFunc_S$ be the union of all the (finitely many) \adsccs{} whose domain contains~$S$.
    By \Cref{lem:adscc union} and induction, $\sccFunc_S$ is an \adscc{}, and by construction it contains every \adscc{} whose domain contains~$S$.
    It is maximal, since every \adscc{} containing~$\sccFunc_S$ has $S$ in its domain, hence is contained in~$\sccFunc_S$.
    It is the unique maximal \adscc{} whose domain contains~$S$, since every such \adscc{} is contained in~$\sccFunc_S$, hence equal to it by maximality.
    Item~(1) follows from~(2) applied to any belief support in the domain of the given \adscc, and item~(3) follows from the uniqueness in~(2).
\end{proof}

Recall that an \adscc{} is bottom if all actions remain in the \adscc{} (\Cref{sec:justifying SECs}).

\begin{lemma} \label{lem:trivial secs}
Assume that $\pdp$ satisfies the assumptions of \Cref{lem:single target}.
\begin{enumerate}
\item The partial functions with domain $\set{\set{\target}}$ (resp.\ $\set{\set{\bad}}$) mapping it to~$\actions$ are bottom maximal \adsccs.
\item If $\sccFunc$ is an \adscc{} whose domain contains a belief support~$S$ such that $\target\in S$ (resp.\ $\bad\in S$), then $\dom{\sccFunc} = \set{\set{\target}}$ (resp.\ $\dom{\sccFunc} = \set{\set{\bad}}$).
\item Every maximal \adscc{} whose domain is different from $\set{\set{\target}}$ and $\set{\set{\bad}}$ is non-bottom, and none of the belief supports in its domain contains $\target$ or~$\bad$.
\end{enumerate}
\end{lemma}
\begin{proof}
We use that $\target$ is absorbing and observable with observation $\obs_\target$: for every $S\subseteq\states$ containing~$\target$ and every action~$\action$, the only observation~$\obs$ such that $\transition(\set{\target},\action,\obs)\neq\emptyset$ is $\obs_\target$, and $\transition(S,\action,\obs_\target) = \set{\target}$.

For~(2), let $S\in\dom{\sccFunc}$ with $\target\in S$, and let $\action\in\sccFunc(S)$.
Then $\transition(S,\action,\obs_\target) = \set{\target}$ is in $\dom{\sccFunc}$, so $|S| = 1$ by \Cref{lem:sec basic}, i.e., $S = \set{\target}$.
Moreover, $\reach{\sccFunc}(\set{\target}) = \set{\set{\target}}$, so by strong connectivity $\dom{\sccFunc} = \set{\set{\target}}$.
The case of~$\bad$ is identical.

For~(1), the partial function with domain $\set{\set{\target}}$ mapping it to~$\actions$ is clearly a bottom \adscc, and by~(2), every \adscc{} containing it has domain $\set{\set{\target}}$, hence is contained in it.

For~(3), let $\sccFunc$ be a maximal \adscc{} whose domain is not $\set{\set{\target}}$ or $\set{\set{\bad}}$.
By~(2), no belief support of $\dom{\sccFunc}$ contains $\target$ or~$\bad$.
Assume towards a contradiction that $\sccFunc$ is bottom, and let $S\in\dom{\sccFunc}$ and $\state\in S$.
Since $\state\neq\bad$, we have $\preach(\dirac{\state})>0$, so there is a finite history from~$\state$ to~$\target$ with positive probability; let $\history$ be its observable projection.
Then $\target\in\transitionIter(S,\history)$, so $\history$ is possible from~$S$, and, as $\sccFunc(U) = \actions$ for all $U\in\dom{\sccFunc}$, $\history$ is inside $\sccFunc$ from~$S$.
Hence $\transitionIter(S,\history)\in\dom{\sccFunc}$ contains~$\target$, a contradiction.
\end{proof}

\subsection{Distinguishability in \adsccs} \label{sec:distinguishing ec basics}
Let $\pdp = \pdpFull$ be a \deterministic POMDP, and let $\transition$ be the deterministic transition function of $\aut{\pdp}$.
We study in more detail the indistinguishability relation defined in \Cref{def:indistinguishable}.

We say that an observable history $\history$ inside~$\sccFunc$ from $\bsupp$ \emphdef{distinguishes $\state$ and $\state'$} if $\history = \history'\action\obs$ and $\ptran(\obs \mid \state_{\history'}, \action) \neq \ptran(\obs \mid \state_{\history'}', \action)$.
Two states $\state,\state'\in\bsupp$ are $(\sccFunc,\bsupp)$-distinguishable if and only if some history inside $\sccFunc$ from~$\bsupp$ distinguishes them.
Indeed, if $\ptran(\obs \mid \state_{\history'}, \action) \neq \ptran(\obs \mid \state_{\history'}', \action)$ for some $\history'$ inside $\sccFunc$ from~$\bsupp$ and $\action\in\sccFunc(\transitionIter(\bsupp,\history'))$, then one of these probabilities is positive, so $\transitionIter(\bsupp,\history'\action\obs)\neq\emptyset$ and $\history'\action\obs$ is inside $\sccFunc$ from~$\bsupp$.
We say that an observable run $\run$ inside $\sccFunc$ from $\bsupp$ \emphdef{distinguishes $\state$ and $\state'$ infinitely often} if there are infinitely many prefixes $\partialRun{n}$ of $\run$ that $(\sccFunc, \bsupp)$-distinguish $\state$ and~$\state'$. 

We show that, in \deterministic POMDPs, both indistinguishability and distinguishability are preserved by taking deterministic transitions.
The statement of this lemma is illustrated in \Cref{fig:dist preserve}.

\begin{lemma}\label{lem:dist preserve}
    Let $\aoscc$ be an \adscc of $\pdp$, and let $\bsupp$ be a belief support in $\dom{\sccFunc}$.
    Let $\state, \state' \in \bsupp$, $\action\in \sccFunc(\bsupp)$, and $\obs \in \sig$ such that $\transition(\bsupp,\action,\obs)\neq\emptyset$.
    Let $\state_{\action\obs} = \transition(\state, \action, \obs)$, $\state'_{\action\obs} = \transition(\state', \action, \obs)$, and $\bsupp_{\action\obs} = \transition(\bsupp, \action, \obs)$.
    \begin{enumerate}
        \item If $\state$ and $\state'$ are $(\aoscc, \bsupp)$-indistinguishable, then $\state_{\action\obs}$ and $\state'_{\action\obs}$ are $(\aoscc, \bsupp_{\action\obs})$-indistinguishable. \label{item:indist preserve}
        \item If $\state$ and $\state' \in \bsupp$ are $(\aoscc, \bsupp)$-distinguishable, then $\state_{\action\obs}$ and $\state'_{\action\obs}$ are $(\aoscc, \bsupp_{\action\obs})$-distinguishable. \label{item:dist preserve}%
    \end{enumerate}%
\end{lemma}

\begin{proof}
    By \Cref{lem:sec basic}, the states $\state_{\action\obs}$ and $\state'_{\action\obs}$ are well defined, and $\bsupp_{\action\obs}\in\dom{\sccFunc}$.
    Statement~\eqref{item:indist preserve} follows directly from the definition of indistinguishability: if $\state$ and $\state'$ are $(\aoscc, S)$-indistinguishable, then for every history $\history'$ from $S_{\action\obs}$ inside $\aoscc$, the history $\history = \action\obs\history'$ is a history from $S$ inside $\aoscc$ (\Cref{lem:sec basic}) that satisfies the indistinguishability condition, since $\transitionIter(\state,\action\obs\history') = \transitionIter(\state_{\action\obs},\history')$ and similarly for~$\state'$.

\begin{figure}
    \centering
    \begin{tikzpicture}[scale=1.05]
        \node (q) at (0,0) {$\state'$};
        \node (indist) at (1,0) {$\indistinguishable$};
        \node (q') at (2,0) {$\state$};
        \node (S) at (2.6,0) {$\bsupp$};
        \node (in) at (2.3,0) {$\in$};

        \node (q2) at (0,-1.6) {$\state'_{\action\obs}$};
        \node (indist2) at (1,-1.6) {$\equiv_{(\aoscc, \bsupp_{\action\obs})}$};
        \node (q2') at (2,-1.6) {$\state_{\action\obs}$};
        \node (S2) at (2.8,-1.6) {$\bsupp_{\action\obs}$};
        \node (in) at (2.4,-1.6) {$\in$};

        \node[rotate=-90,scale=2] (implies) at (.9,-.8) {$\Rightarrow$};
        \draw[->] (q) -- (q2) node[midway, left] {$(\action, \obs)$};
        \draw[->] (q') -- (q2') node[midway, right] {$(\action, \obs)$};
    \end{tikzpicture}%
    \hspace{15pt}%
    \begin{tikzpicture}[scale=1.05]
        \node (q) at (0,0) {$\state'$};
        \node (indist) at (1,0) {$\distinguishable$};
        \node (q') at (2,0) {$\state$};
        \node (S) at (2.6,0) {$\bsupp$};
        \node (in) at (2.3,0) {$\in$};

        \node (q2) at (0,-1.6) {$\state'_{\action\obs}$};
        \node (indist2) at (1,-1.6) {$\not\equiv_{(\aoscc, \bsupp_{\action\obs})}$};
        \node (q2') at (2,-1.6) {$\state_{\action\obs}$};
        \node (S2) at (2.8,-1.6) {$\bsupp_{\action\obs}$};
        \node (in) at (2.4,-1.6) {$\in$};

        \node[rotate=-90,scale=2] (implies) at (.9,-.8) {$\Rightarrow$};
        \draw[->] (q) -- (q2) node[midway, left] {$(\action, \obs)$};
        \draw[->] (q') -- (q2') node[midway, right] {$(\action, \obs)$};
    \end{tikzpicture}
    \caption{Illustration of the statement of \Cref{lem:dist preserve} (\eqref{item:indist preserve} on the left and \eqref{item:dist preserve} on the right).}
    \label{fig:dist preserve}
\end{figure}

    For statement~\eqref{item:dist preserve}, assume that $\state$ and $\state'$ are $(\aoscc, S)$-distinguishable.
    Let $\hat{\history}$ be a history distinguishing $\state$ and $\state'$.
    We show that $\state_{\action\obs}$ and $\state'_{\action\obs}$ are $(\aoscc, S_{\action\obs})$-distinguishable, that is, there exists a history from $S_{\action\obs}$ inside $\aoscc$ that distinguishes $\state_{\action\obs}$ and $\state'_{\action\obs}$.
    Since $\aoscc$ is an \adscc, there exists a history~$\history'$ from $S_{\action\obs}$ inside $\aoscc$ that leads back to $S$.
    Hence, the history $\history = \action\obs\history'$ is a history from $S$ to $S$ inside $\aoscc$.
    Therefore, $\history$ acts as a permutation on the belief support~$S$.
    This implies that there is $k \ge 1$ such that $\transitionIter(\state, \history^k) = \state$ and $\transitionIter(\state', \history^k) = \state'$ (\Cref{lem:sec basic}).
    By construction, the history $\history'\history^{k-1}\hat{\history}$ distinguishes $\state_{\action\obs}$ and $\state'_{\action\obs}$: it is inside $\sccFunc$ from $S_{\action\obs}$ by \Cref{lem:sec basic}, and $\transitionIter(\state_{\action\obs},\history'\history^{k-1}) = \transitionIter(\state,\history^k) = \state$, and similarly for~$\state'$.
\end{proof}
Distinguishability is preserved by containment of \adsccs.

\begin{lemma} \label{lem:distinguishability preserved}
    Let $\sccFunc$ and $\sccFunc'$ be two \adsccs such that $\sccFunc$ is contained in $\sccFunc'$.
    If $\sccFunc$ is distinguishing, then $\sccFunc'$ is distinguishing.
\end{lemma}
\begin{proof}
    Let $\bsupp$ be a belief support in the domain of $\sccFunc$ and $\state,\state'\in \bsupp$ be two states that are $(\sccFunc, \bsupp)$-distinguishable.
    Since~$\sccFunc$ is contained in $\sccFunc'$, we have that $\bsupp$ is also in the domain of~$\sccFunc'$.
    Moreover, any distinguishing history from $\bsupp$ inside $\sccFunc$ is also a distinguishing history from $\bsupp$ inside $\sccFunc'$, as all actions along the history are also available through function $\sccFunc'$.
\end{proof}

In a distinguishing \adscc, distinguishable states exist in every belief support of the domain, and not only in one of them.
\begin{lemma}\label{lem:dist all supports}
    Let $\sccFunc$ be a distinguishing \adscc{} of~$\pdp$ and $\bsupp\in\dom{\sccFunc}$.
    Then $\bsupp$ contains two $(\sccFunc,\bsupp)$-distinguishable states; in other words, every equivalence class of $\indistinguishable$ is a strict subset of~$\bsupp$.
\end{lemma}
\begin{proof}
    Since $\sccFunc$ is distinguishing, there are $\bsupp'\in\dom{\sccFunc}$ and two $(\sccFunc,\bsupp')$-distinguishable states $\state,\state'\in\bsupp'$.
    By strong connectivity, there is a history $\history$ inside $\sccFunc$ from~$\bsupp'$ such that $\transitionIter(\bsupp',\history) = \bsupp$.
    By \Cref{lem:sec basic}, the states $\transitionIter(\state,\history)$ and $\transitionIter(\state',\history)$ are distinct elements of~$\bsupp$, and, by \Cref{lem:dist preserve}\eqref{item:dist preserve} applied along~$\history$, they are $(\sccFunc,\bsupp)$-distinguishable.
\end{proof}

In the rest of \Cref{sec:ec}, we will study both distinguishing and non-distinguishing \adsccs.
In \Cref{sec:non-distinguishing}, we show that staying for a long time in a non-distinguishing \adscc{} brings no information, so that the best strategy is to find the best action to exit it.
In \Cref{sec:distinguishing}, we show that, by staying long enough in a distinguishing \adscc, the agent can tell distinguishable states apart with probability arbitrarily close to~$1$.

We will frequently use the following strategy in order to ``explore'' an \adscc.
Let $\aoscc$ be an \adscc{} and $\bel{b}$ be an initial belief such that $\supp{\bel{b}}\in\dom{\sccFunc}$.
Let \emphdef{$\unif{\aoscc}$} be the strategy that tracks the current belief support $\bsupp$ (i.e., after observing some history $\history$ inside $\sccFunc$ from~$\bel{b}$ such that $\transitionIter(\supp{\bel{b}},\history) = \bsupp$) and chooses an action from $\sccFunc(\bsupp)$ uniformly at random.
We do not need to define $\unif{\aoscc}$ for $\history\in(\actions\times\sig)^*$ such that $\transitionIter(\supp{\bel{b}},\history) \notin\dom{\sccFunc}$ since such histories are never observed as long as we follow $\unif{\aoscc}$ starting from~$\bel{b}$ (for definiteness, $\unif{\sccFunc}$ chooses an action uniformly in~$\actions$ after such histories).
By~\eqref{eq:prob obs hist}, every observable history~$\history$ inside $\sccFunc$ from a belief~$\bel{b}$ satisfies $\pdist{\bel{b}}{\unif{\sccFunc}}(\cyl{\history})\ge\prod_{k=1}^{|\history|}\frac{1}{|\actions|}\cdot\ptran(\obs_k\mid\bel{b}_{\history_{<k}},\action_k)>0$, and almost every run (with respect to $\pdist{\bel{b}}{\unif{\sccFunc}}$) is inside $\sccFunc$ from~$\bel{b}$, since the histories that are not inside~$\sccFunc$ but whose strict prefixes are have probability~$0$.
Moreover, by~\eqref{eq:shift}, for every history~$\history$ inside $\sccFunc$ from~$\bel{b}$ and every observable history~$\history'$, $\pdist{\bel{b}}{\unif{\sccFunc}}(\cyl{\history\history'}) = \pdist{\bel{b}}{\unif{\sccFunc}}(\cyl{\history})\cdot\pdist{\bel{b}_\history}{\unif{\sccFunc}}(\cyl{\history'})$, because the strategy $\unif{\sccFunc}$ shifted by~$\history$ coincides with $\unif{\sccFunc}$ from $\bel{b}_\history$ (it only depends on the current belief support).

A useful lemma is that, by restricting a (possibly distinguishing) \adscc to the supports reachable from an equivalence class of indistinguishable states inside the \adscc, we obtain a non-distinguishing \adscc.
\begin{lemma}\label{lem:ndist form scc}
    Let $\sccFunc$ be an \adscc{} of~$\pdp$, $\bsupp\in\dom{\sccFunc}$, $\state\in \bsupp$, and $C = [\state]_{\indistinguishable}$.
    Let
    \[
    \reach{\state, \bsupp}^{\sccFunc} = \setof{\transitionIter(C,\history)}{\text{$\history$ is an observable history inside $\sccFunc$ from $\bsupp$}}
    \]
    be the set of supports reachable from the equivalence class of $\state$ inside $\sccFunc$ from $\bsupp$.
    Let $\sccFunc_{\state, \bsupp}$ be the function such that $\dom{\sccFunc_{\state, \bsupp}} = \reach{\state, \bsupp}^{\sccFunc}$ and, for all $\bsuppBis \in \reach{\state, \bsupp}^{\sccFunc}$,
    \[
    \sccFunc_{\state, \bsupp}(\bsuppBis) =
    \bigcup_{\history\in H(\bsuppBis)} \sccFunc(\transitionIter(\bsupp, \history))\ ,
    \]
    where $H(\bsuppBis)$ is the set of observable histories~$\history$ inside $\sccFunc$ from $\bsupp$ such that $\transitionIter(C, \history) = \bsuppBis$.
    The function $\sccFunc_{\state, \bsupp}$ is a non-distinguishing \adscc.
    Moreover, every observable history inside $\sccFunc$ from~$\bsupp$ is inside $\sccFunc_{\state,\bsupp}$ from~$C$.
\end{lemma}
\begin{proof}
    By \Cref{lem:sec basic}, for every history $\history$ inside $\sccFunc$ from $\bsupp$, the map $\state'\mapsto\transitionIter(\state',\history)$ is a bijection from $\bsupp$ onto $\transitionIter(\bsupp,\history)$; in particular, $\transitionIter(C,\history)$ is a non-empty subset of $\transitionIter(\bsupp,\history)$, of cardinality~$|C|$.
    We call a pair $(\bsupp',\bsuppBis)$ \emph{reachable} if $\bsupp' = \transitionIter(\bsupp,\history)$ and $\bsuppBis = \transitionIter(C,\history)$ for some history~$\history$ inside $\sccFunc$ from~$\bsupp$.
    By definition, $\sccFunc_{\state,\bsupp}(\bsuppBis)$ is the union of the sets $\sccFunc(\bsupp')$ for all reachable pairs $(\bsupp',\bsuppBis)$.

    We first show that $\sccFunc_{\state, \bsupp}$ is an \adscc.
    Its domain contains $C = \transitionIter(C,\emptyhist)$, and its values are non-empty.

    Let $\bsuppBis \in \reach{\state, \bsupp}^{\sccFunc}$ and $(\action,\obs)\in\actions\times\sig$ such that $\action \in \sccFunc_{\state, \bsupp}(\bsuppBis)$ and $\transition(\bsuppBis,\action,\obs)\neq\emptyset$.
    We show that $\transition(\bsuppBis, \action,\obs) \in \reach{\state, \bsupp}^{\sccFunc}$.
    Since $\action \in \sccFunc_{\state, \bsupp}(\bsuppBis)$, there exists a history $\history$ inside $\sccFunc$ from $\bsupp$ such that, writing $\bsupp' = \transitionIter(\bsupp, \history)$, we have $\bsuppBis = \transitionIter(C, \history)$ and $\action \in \sccFunc(\bsupp')$.
    Since $\bsuppBis\subseteq\bsupp'$, we have $\transition(\bsupp',\action,\obs)\supseteq\transition(\bsuppBis,\action,\obs)\neq\emptyset$, so $\history\action\obs$ is inside $\sccFunc$ from~$\bsupp$.
    Thus, $\transition(\bsuppBis, \action, \obs) = \transitionIter(C, \history\action\obs)$ is indeed an element of $\reach{\state, \bsupp}^{\sccFunc}$.

    We now show the strong connectivity property: for $\bsuppBis, \bsuppBis' \in \reach{\state, \bsupp}^{\sccFunc}$, we show the existence of an observable history inside $\sccFunc_{\state, \bsupp}$ from $\bsuppBis$ to $\bsuppBis'$.
    Let $\history, \history'$ be histories inside $\sccFunc$ from~$\bsupp$ such that $\bsuppBis = \transitionIter(C, \history)$ and $\bsuppBis' = \transitionIter(C, \history')$.
    Let $\bsupp_\history = \transitionIter(\bsupp, \history)$.
    Since $\sccFunc$ is an \adscc, there exists an observable history $\hat{\history}$ inside $\sccFunc$ from $\bsupp_\history$ to $\bsupp$.
    Hence, the history $\history\hat{\history}$ is an observable history inside $\sccFunc$ from $\bsupp$ to $\bsupp$.
    By \Cref{lem:sec basic}, there exists $k \ge 1$ such that $(\history\hat{\history})^k$ is inside~$\sccFunc$ from~$\bsupp$ and acts as the identity on $\bsupp$; in particular, $\transitionIter(C, (\history\hat{\history})^k) = C$.
    Let $\history'' = \hat{\history}(\history\hat{\history})^{k-1}\history'$; by \Cref{lem:sec basic}, $\history\history''$ is inside~$\sccFunc$ from~$\bsupp$, and $\transitionIter(C,\history\history'') = \transitionIter(C,\history') = \bsuppBis'$.
    It remains to show that $\history''$ is inside $\sccFunc_{\state,\bsupp}$ from~$\bsuppBis$; this follows from the following more general observation, applied to $\history\history''$: if $\history^\star$ is a history inside $\sccFunc$ from~$\bsupp$, then $\history^\star$ is inside $\sccFunc_{\state,\bsupp}$ from~$C$.
    Indeed, write $\history^\star = \action_1\obs_1\cdots\action_m\obs_m$ and $\history^\star_{<j} = \action_1\obs_1\cdots\action_{j-1}\obs_{j-1}$.
    For every $j$, the pair $(\transitionIter(\bsupp,\history^\star_{<j}),\transitionIter(C,\history^\star_{<j}))$ is reachable and $\action_j\in\sccFunc(\transitionIter(\bsupp,\history^\star_{<j}))$, so $\action_j\in\sccFunc_{\state,\bsupp}(\transitionIter(C,\history^\star_{<j}))$; moreover, $\transitionIter(C,\history^\star)\neq\emptyset$.
    This also proves the last statement of the lemma.
    Since $\history\history''$ is inside $\sccFunc_{\state,\bsupp}$ from~$C$ and $\transitionIter(C,\history) = \bsuppBis$, \Cref{lem:sec basic} shows that $\history''$ is inside $\sccFunc_{\state,\bsupp}$ from~$\bsuppBis$, and it leads to~$\bsuppBis'$.

    Finally, we show that $\sccFunc_{\state, \bsupp}$ is non-distinguishing.
    We first prove that, for every reachable pair $(\bsupp',\bsuppBis)$, any two states of~$\bsuppBis$ are $(\sccFunc,\bsupp')$-indistinguishable.
    Let $\history$ be a history inside $\sccFunc$ from~$\bsupp$ with $\bsupp' = \transitionIter(\bsupp,\history)$ and $\bsuppBis = \transitionIter(C,\history)$, and let $p,p'\in\bsuppBis$.
    By \Cref{lem:sec basic}, $p = \transitionIter(r,\history)$ and $p' = \transitionIter(r',\history)$ for some $r,r'\in C$.
    Since $r\indistinguishable r'$, applying \Cref{lem:dist preserve}\eqref{item:indist preserve} along~$\history$ shows that $p$ and $p'$ are $(\sccFunc,\bsupp')$-indistinguishable.

    Now let $\bsuppBis\in\dom{\sccFunc_{\state,\bsupp}}$ and $p,p'\in\bsuppBis$; we show that $p$ and $p'$ are $(\sccFunc_{\state,\bsupp},\bsuppBis)$-indistinguishable.
    Let $\history^\star$ be a history inside $\sccFunc_{\state,\bsupp}$ from~$\bsuppBis$, let $\bsuppBis_{\history^\star} = \transitionIter(\bsuppBis,\history^\star)$, which belongs to $\dom{\sccFunc_{\state,\bsupp}}$, let $p_{\history^\star} = \transitionIter(p,\history^\star)$ and $p'_{\history^\star} = \transitionIter(p',\history^\star)$, which belong to $\bsuppBis_{\history^\star}$ by \Cref{lem:sec basic} applied to the \adscc{} $\sccFunc_{\state,\bsupp}$, and let $\action\in\sccFunc_{\state,\bsupp}(\bsuppBis_{\history^\star})$.
    By definition of $\sccFunc_{\state,\bsupp}$, there is a reachable pair $(\bsupp'',\bsuppBis_{\history^\star})$ with $\action\in\sccFunc(\bsupp'')$.
    As shown above, $p_{\history^\star}$ and $p'_{\history^\star}$ are $(\sccFunc,\bsupp'')$-indistinguishable; applying \Cref{def:indistinguishable} with the empty history, we obtain $\ptran(\obs\mid p_{\history^\star},\action) = \ptran(\obs\mid p'_{\history^\star},\action)$ for every observation~$\obs$.
    Hence $p$ and $p'$ are $(\sccFunc_{\state,\bsupp},\bsuppBis)$-indistinguishable.
    This argument does not require $\history^\star$ to be a history inside $\sccFunc$ from a fixed belief support: at each step, it only uses that the action is allowed by $\sccFunc$ in some belief support~$\bsupp''$, which may depend on the step.
\end{proof}

\subsection{Non-distinguishing \adsccs{}} \label{sec:non-distinguishing}
We show here that, in non-distinguishing \adsccs{}, it is not possible to gain information about the current state while staying in the \adscc: the best move in such an \adscc{} is to find the best action to exit it.
Let $\pdp = \pdpFull$ be a \deterministic POMDP.

\begin{lemma}\label{lem:bel up aodscc}
    Let $\aoscc$ be an \adscc of $\pdp$.
    For every belief $\bel{b}$ such that $\supp{\bel{b}} \in \dom{\sccFunc}$, action $\action\in\sccFunc(\supp{\bel{b}})$, observation $\obs\in\sig$ such that $\ptran(\obs\mid\bel{b},\action)>0$, and state $\state\in\supp{\bel{b}}$, the state $\transition(\state,\action,\obs)$ is defined and we have
    \[
    \transition(\bel{b},\action,\obs)(\transition(\state,\action,\obs)) =
    \frac{\ptran(\obs \mid\state,\action)\cdot \bel{b}(\state)}{
    \ptran(\obs \mid \bel{b},\action)
    }\ .
    \]
\end{lemma}
\begin{proof}%
    Let $S = \supp{\bel{b}}$.
    Since $\ptran(\obs\mid\bel{b},\action)>0$, we have $\transition(S,\action,\obs)\neq\emptyset$ (\Cref{lem:possible hist}), so the history $\action\obs$ is inside $\sccFunc$ from~$S$.
    By \Cref{lem:sec basic}, $\transition(\state',\action,\obs)$ is defined for every $\state'\in S$, and $\transition(\state,\action,\obs) \neq \transition(\state',\action,\obs)$ for all $\state \neq \state'\in S$.
    Hence, for every $\state\in S$, the only state of~$S$ from which $\transition(\state,\action,\obs)$ can be reached by taking action $\action$ and observing $\obs$ is~$\state$, and, by posterior-determinism, $\ptran(\obs,\transition(\state,\action,\obs)\mid\state,\action) = \ptran(\obs\mid\state,\action)$.
    By~\eqref{eq:belief update},
    \[
    \transition(\bel{b},\action,\obs)(\transition(\state,\action,\obs))
    = \frac{\bel{b}(\state)\cdot\ptran(\obs,\transition(\state,\action,\obs)\mid\state,\action)}{\ptran(\obs\mid \bel{b},\action)}
    = \frac{\ptran(\obs\mid \state,\action)\cdot\bel{b}(\state)}{\ptran(\obs\mid \bel{b},\action)} \ .
    \qedhere
    \]
\end{proof}
\begin{lemma}\label{lem:non-dist reach}
Let $\ascc$ be a non-distinguishing \adscc of $\pdp$, let $\bel{b}$ be a non-zero sub-belief such that $\supp{\bel{b}}\in \dom{\ascc}$, and let $\history$ be a history inside $\ascc$ from $\bel{b}$.
Then,
\[
\transitionIter(\bel{b},\history) =
\sum_{\state\in\supp{\bel{b}}} \bel{b}(\state)\cdot \transitionIter(\state,\history) \ ,
\]
where, as in \Cref{sec:prelim}, a state $\transitionIter(\state,\history)$ stands for the Dirac belief $\dirac{\transitionIter(\state,\history)}$.
In other words, $\transitionIter(\bel{b},\history)$ is obtained from~$\bel{b}$ by renaming each state $\state\in\supp{\bel{b}}$ into $\transitionIter(\state,\history)$, which is an injective renaming by \Cref{lem:sec basic}.
\end{lemma}
\begin{proof}%
By the remarks on the mass-preserving update in \Cref{sec:prelim} (namely, $(\lambda\bel{c})_{\action,\obs} = \lambda\cdot\bel{c}_{\action,\obs}$) and induction on the length of~$\history$, $\transitionIter(\bel{b},\history) = \oneNorm{\bel{b}}\cdot\transitionIter(\normalised{\bel{b}},\history)$, so it suffices to consider the case where $\bel{b}$ is a belief.
We proceed by induction on the length of $\history$.
The property is clearly true when $\length{\history} = 0$.
Let $\history\action\obs$ be a history inside $\ascc$ from $\bel{b}$, let $\bel{b}_\history = \transitionIter(\bel{b},\history)$, $S_\history = \supp{\bel{b}_\history} = \transitionIter(\supp{\bel{b}},\history)$, and, for $\state\in\supp{\bel{b}}$, $\state_\history = \transitionIter(\state,\history)$.
Assume, by induction hypothesis, that $\bel{b}_\history(\state_\history) = \bel{b}(\state)$ for all $\state\in\supp{\bel{b}}$.
We have $\action\in\ascc(S_\history)$ and $\ptran(\obs\mid\bel{b}_\history,\action)>0$.
Since $\ascc$ is non-distinguishing, the states of $S_\history$ are pairwise $(\ascc,S_\history)$-indistinguishable, so, by \Cref{def:indistinguishable} applied with the empty history, there is a number $t$ such that $\ptran(\obs \mid \state_\history,\action) = t$ for all $\state\in\supp{\bel{b}}$.
Then $\ptran(\obs\mid\bel{b}_\history,\action) = \sum_{\state\in\supp{\bel{b}}}\bel{b}_\history(\state_\history)\cdot t = t$, as $\bel{b}_\history$ is a belief, and $t>0$.
By \Cref{lem:bel up aodscc}, for all $\state\in\supp{\bel{b}}$,
\[
\transition(\bel{b}_\history,\action,\obs)(\transition(\state_\history,\action,\obs)) = \frac{t\cdot\bel{b}_\history(\state_\history)}{t} = \bel{b}(\state)\ .
\]
This finishes the induction step.
\end{proof}

\begin{corollary}\label{cor:reach non-dist scc}
    For every non-distinguishing \adscc $\sccFunc$, non-zero sub-belief $\bel{b}$ such that $\supp{\bel{b}}\in \dom{\sccFunc}$, and $\bel{b}' \in \reach{\sccFunc}(\bel{b})$,
    there exists a history $\history'$ inside $\sccFunc$ from $\bel{b}'$ such that
    \[
    \transitionIter(\bel{b}',\history') = \bel{b} \ .
    \]
    In particular, $\reach{\sccFunc}(\bel{b}') = \reach{\sccFunc}(\bel{b})$.
\end{corollary}
\begin{proof}%
    Let $\history$ be a history inside $\sccFunc$ from $\bel{b}$ such that $\transitionIter(\bel{b},\history) = \bel{b}'$.
    Since $\sccFunc$ is an \adscc,
    there must exist a history $\overline{\history}$ inside $\sccFunc$ from $\bel{b}'$ such that $\transitionIter(\supp{\bel{b}'},\overline{\history}) = \supp{\bel{b}}$.
    Thus, $\history \overline{\history}$ is inside $\sccFunc$ from $\bel{b}$ and $\transitionIter(\supp{\bel{b}},\history \overline{\history}) = \supp{\bel{b}}$.
    By \Cref{lem:sec basic}, there exists $k \ge 1$ such that $(\history \overline{\history})^{k}$ is inside $\sccFunc$ from $\bel{b}$ and
    \(
    \transitionIter(\state,(\history \overline{\history})^{k}) = \state
    \)
    for every $\state\in\supp{\bel{b}}$.
    Now we can finish the proof by taking $\history' = \overline{\history} (\history \overline{\history})^{k - 1}$ and using \Cref{lem:non-dist reach}.
    The last statement follows from \Cref{lem:sec basic}\eqref{item:concat}, since $\bel{b}'$ is reachable from~$\bel{b}$ and conversely.
\end{proof}

We recall that $\reach{\sccFunc}(\bel{b})$ is the set of sub-beliefs reachable inside $\sccFunc$ from $\bel{b}$ (\Cref{def:sec}).

\finBeliefSpace*
\begin{proof}%
Follows from \Cref{lem:non-dist reach}, since it shows that a belief reachable from $\bel{b}$ inside $\sccFunc$ is the same as $\bel{b}$ up to an injective renaming of the states in $\supp{\bel{b}}$, and there are finitely many such renamings (at most $|\states|!$).
\end{proof}
\begin{lemma}\label{lem:val in non-dist}
    For every non-distinguishing \adscc $\sccFunc$ of $\pdp$, belief~$\bel{b}$ such that $\supp{\bel{b}}\in \dom{\sccFunc}$, and $\bel{b}' \in \reach{\sccFunc}(\bel{b})$,
    \[\preach(\bel{b}') = \preach(\bel{b})\ .\]
\end{lemma}
\begin{proof}%
We prove that all the beliefs inside $\reach{\sccFunc}(\bel{b})$ have the same value.

\Cref{cor:non dist fin reach} says that $\reach{\sccFunc}(\bel{b})$ is finite, so let
\[
    \bel{b}_{\min} = \argmin_{\bel{c}\in \reach{\sccFunc}(\bel{b})} \preach(\bel{c})
    \quad\text{and}\quad
    \bel{b}_{\max} = \argmax_{\bel{c}\in \reach{\sccFunc}(\bel{b})} \preach(\bel{c}) \ .
\]
It suffices to show that $\preach(\bel{b}_{\min}) = \preach(\bel{b}_{\max})$.

Let $\history_{\min}$ (resp.\ $\history_{\max}$) be a history inside $\sccFunc$ from $\bel{b}$ such that $\transitionIter(\bel{b},\history_{\min}) = \bel{b}_{\min}$ (resp.\ $\transitionIter(\bel{b},\history_{\max}) = \bel{b}_{\max}$).
By \Cref{cor:reach non-dist scc}, there exists a history $\history_{\min}'$ inside $\sccFunc$ from $\bel{b}_{\min}$ such that
\[
\transitionIter(\bel{b}_{\min},\history_{\min}') = \bel{b} \ .
\]
Hence, the history $\history = \history_{\min}' \history_{\max}$ inside $\sccFunc$ from $\bel{b}_{\min}$ is such that
\[\transitionIter(\bel{b}_{\min},\history) = \bel{b}_{\max} \ .\]

Let $n = \length{\history}$.
Let $\reach{\sccFunc}^n(\bel{b}_{\min})$ be the set of beliefs reachable from $\bel{b}_{\min}$ in exactly $n$ steps while staying inside~$\sccFunc$; in particular, $\bel{b}_{\max}\in \reach{\sccFunc}^n(\bel{b}_{\min})$.
For $\bel{c}\in \reach{\sccFunc}^n(\bel{b}_{\min})$, let $p_{\bel{c}}$ be the probability of reaching $\bel{c}$ in exactly $n$ steps from $\bel{b}_{\min}$ following $\unif{\sccFunc}$.
Since almost every run is inside $\sccFunc$ under $\unif{\sccFunc}$, we have $\sum_{\bel{c} \in \reach{\sccFunc}^n(\bel{b}_{\min})} p_{\bel{c}} = 1$; moreover, $p_{\bel{b}_{\max}}\ge\pdist{\bel{b}_{\min}}{\unif{\sccFunc}}(\cyl{\history}) > 0$.
By \Cref{cor:reach non-dist scc}, $\reach{\sccFunc}^n(\bel{b}_{\min})\subseteq\reach{\sccFunc}(\bel{b}_{\min}) = \reach{\sccFunc}(\bel{b})$, so $\preach(\bel{c})\ge\preach(\bel{b}_{\min})$ for all $\bel{c}\in\reach{\sccFunc}^n(\bel{b}_{\min})$.
By~\eqref{eq:val iterated} applied to $\bel{b}_{\min}$ and $\unif{\sccFunc}$, we have
\begin{align*}
\preach(\bel{b}_{\min})
\quad &\geq\quad
\sum_{\bel{c} \in \reach{\sccFunc}^n(\bel{b}_{\min})} p_{\bel{c}}\cdot \preach(\bel{c}) \\
&=\quad
p_{\bel{b}_{\max}}\cdot\preach(\bel{b}_{\max}) + 
\sum_{\bel{c} \in \reach{\sccFunc}^n(\bel{b}_{\min})\setminus\set{\bel{b}_{\max}}} p_{\bel{c}}\cdot \preach(\bel{c}) \\
&\ge\quad
p_{\bel{b}_{\max}}\cdot\preach(\bel{b}_{\max}) + 
(1 - p_{\bel{b}_{\max}})\cdot\preach(\bel{b}_{\min})\ .
\end{align*}
Hence $p_{\bel{b}_{\max}}\cdot\preach(\bel{b}_{\min}) \ge p_{\bel{b}_{\max}}\cdot\preach(\bel{b}_{\max})$, and since $p_{\bel{b}_{\max}} > 0$, $\preach(\bel{b}_{\min}) \ge \preach(\bel{b}_{\max})$.
By definition of $\bel{b}_{\min}$ and $\bel{b}_{\max}$, this implies that $\preach(\bel{b}_{\min}) = \preach(\bel{b}_{\max})$.
\end{proof}

We can now prove \Cref{thm:nd adscc}, the main result about non-distinguishing \adsccs.
\ndSECs*
\begin{proof}
    We first assume that $\bel{b}$ is a belief, and write $S = \supp{\bel{b}}$ and
    \[
    v^* = \max_{\bel{b}'\in\reach{\sccFunc}(\bel{b})}\ \max_{\action\notin \sccFunc(\supp{\bel{b}'})} \preach(\bel{b}',\action)\ .
    \]
    Since $\sccFunc$ is maximal and non-bottom, by \Cref{lem:trivial secs}, no belief support in $\dom{\sccFunc}$ contains $\target$, and there is $S'\in\dom{\sccFunc}$ with $\sccFunc(S')\neq\actions$.
    By strong connectivity, there is a history~$\history$ inside $\sccFunc$ from~$S$ such that $\transitionIter(S,\history) = S'$, and $\bel{b}_\history\in\reach{\sccFunc}(\bel{b})$ has support~$S'$; hence the maximum defining~$v^*$ ranges over a non-empty set, which is finite by \Cref{cor:non dist fin reach}.
    For a belief $\bel{b}'$ such that $\supp{\bel{b}'}\in \dom{\sccFunc}$, we call an \emph{exiting action} any action $\action\notin \sccFunc(\supp{\bel{b}'})$.

    We first show that $\preach(\bel{b})\ge v^*$.
    For every $\bel{b}'\in\reach{\sccFunc}(\bel{b})$ and every action~$\action$, we have $\preach(\bel{b}',\action)\le\preach(\bel{b}') = \preach(\bel{b})$, by~\eqref{eq:val bel upd} and \Cref{lem:val in non-dist}.

    We now show that $\preach(\bel{b})\le v^*$.
    Let $\strategy$ be a strategy.
    Let $H$ be the set of pairs $(\history,\action)$ such that $\history$ is an observable history inside $\sccFunc$ from~$\bel{b}$ and $\action$ is an exiting action for $\bel{b}_\history$, and, for $(\history,\action)\in H$, let $E_{\history,\action}$ be the set of runs whose observable projection starts with $\history$ followed by action~$\action$.
    The sets $E_{\history,\action}$ are pairwise disjoint: they correspond to the first time an exiting action is played.
    Let $\run = \state_0\action_1\obs_1\state_1\ldots$ be a run that visits~$\target$, at some step~$i$, and all of whose finite prefixes have positive probability.
    For all $k\le i$, we have $\state_k\in\transitionIter(S,\partialRun{k})$.
    If none of the actions $\action_1,\dots,\action_i$ were exiting, then $\partialRun{i}$ would be inside $\sccFunc$ from~$S$, and $\target\in\transitionIter(S,\partialRun{i})\in\dom{\sccFunc}$, which is impossible.
    Hence $\run\in E_{\history,\action}$ for some $(\history,\action)\in H$ with $|\history| < i$, and $\run$ visits~$\target$ after step~$|\history|$.
    Using~\eqref{eq:shift} with $n = |\history|$, and then~\eqref{eq:prob obs hist} and~\eqref{eq:shift} with $n = 1$ as in the proof of~\eqref{eq:val bel upd}, we get, for every $(\history,\action)\in H$,
    \[
    \pdist{\bel{b}}{\strategy}(\reach{\target}\cap E_{\history,\action}) \le \pdist{\bel{b}}{\strategy}(\cyl{\history})\cdot\strategy(\history)(\action)\cdot\preach(\bel{b}_\history,\action) \le v^*\cdot\pdist{\bel{b}}{\strategy}(E_{\history,\action})\ ,
    \]
    since $\bel{b}_\history\in\reach{\sccFunc}(\bel{b})$ and $\action$ is an exiting action for~$\bel{b}_\history$.
    Summing over the pairwise disjoint sets $E_{\history,\action}$, we obtain $\pdist{\bel{b}}{\strategy}(\reach{\target})\le v^*$.
    As this holds for every strategy~$\strategy$, we conclude that $\preach(\bel{b})\le v^*$.

    Finally, if $\bel{b}$ is a non-zero sub-belief, then $\reach{\sccFunc}(\bel{b}) = \setof{\oneNorm{\bel{b}}\cdot\bel{c}}{\bel{c}\in\reach{\sccFunc}(\normalised{\bel{b}})}$ (since the conditions on histories only depend on supports, and updates preserve the mass), and the statement follows from the case of beliefs applied to $\normalised{\bel{b}}$ and from \Cref{lem:subbelief val}, since $\preach(\oneNorm{\bel{b}}\cdot\bel{c},\action) = \oneNorm{\bel{b}}\cdot\preach(\bel{c},\action)$.
\end{proof}
\subsection{The posterior probability of starting in a given state} \label{subsec:aPosteriori}

Let $\pdp = \pdpFull$ be a POMDP.\footnote{Results of \Cref{subsec:aPosteriori} hold for general POMDPs, not only \deterministic ones.
However, we will only use them in the \deterministic case.}
For every $n\in\N$, we define a \sigal{} \emphdef{$\filtr{n}$} on $\Runs(\pdp)$ (which, we recall, is a subset of $(\states\times\ac\times\sig)^\omega$) that corresponds to the information available to the agent after the first $n$ action-observation pairs.
We use the cylinders of \emph{observable} histories (\Cref{sec:prelim}): for an observable history $\history \in (\ac \times \sig)^*$, the \emphdef{cylinder} $\cyl{\history}$ is the set of runs in $\Runs(\pdp)$ whose projection to $(\ac \times \sig)^{\length{\history}}$ matches $\history$.
For $n\in\N$, the cylinders $\cyl{\history}$ for $\history\in(\actions\times\sig)^n$ form a finite partition of $\Runs(\pdp)$ (some of its blocks may be empty), and we let $\filtr{n}$ be the \sigal{} generated by this partition, i.e.,
\[
\filtr{n} = \setof{\bigcup_{\history\in X}\cyl{\history}}{X\subseteq(\actions\times\sig)^n} \ .
\]
Since every cylinder $\cyl{\history}$ with $|\history| = n$ is the union of the cylinders $\cyl{\history\action\obs}$, we have $\filtr{0} \subseteq \filtr{1} \subseteq \filtr{2} \subseteq \dots$, so this sequence defines a \emph{filtration}.

Let $\bel{b}$ be an initial belief, and $\strategy$ be a strategy; since $\bel{b}$ is a belief, $\pdist{\bel{b}}{\strategy}$ is a probability measure.
In this section and in \Cref{sec:distinguishing}, we only consider beliefs (not sub-beliefs); results are transferred to sub-beliefs at the end of \Cref{sec:distinguishing}, using \Cref{lem:subbelief val}.
For every $q\in\supp{\bel{b}}$, define a sequence of random variables \emphdef{$(\randv{q}{n})_{n\in\N}$} that track the probability to have started in $q$ given the first $n$ actions and observations, in addition to the initial belief $\bel{b}$.

Formally, consider the event $\cyl{q} \in \filtr{}$ (i.e., ``the initial state is $q$'').
For $n\in\N$, we define the random variable $\randv{q}{n}\colon \Runs(\pdp) \to [0,1]$ as
\[
\randv{\state}{n}(\run) = \pdist{\bel{b}}{\strategy}(\cyl{\state} \mid \partialRun{n}) \ ,
\]
representing the a posteriori probability to have started in $\state$ given that we have observed the first $n$ actions and observations of run~$\run$.
The random variables considered in this section and in \Cref{sec:distinguishing} only depend on the observable projection of the run; accordingly, we say that a run is inside an \adscc{} $\sccFunc$ from~$\bel{b}$ if its observable projection is, and we write $\partialRun{n}$ for the observable history made of the first $n$ actions and observations of~$\run$.
Formally, $\randv{\state}{n}(\run) = \frac{\pdist{\bel{b}}{\strategy}(\cyl{\state}\cap\cyl{\partialRun{n}})}{\pdist{\bel{b}}{\strategy}(\cyl{\partialRun{n}})}$ if $\pdist{\bel{b}}{\strategy}(\cyl{\partialRun{n}})>0$, and $\randv{\state}{n}(\run) = 0$ otherwise (which only happens on a set of probability~$0$).
Since $\filtr{n}$ is generated by a finite partition, this is a version of the conditional probability of $\cyl{\state}$ given~$\filtr{n}$:
\[
\randv{\state}{n} = \pdist{\bel{b}}{\strategy}(\cyl{\state} \mid \filtr{n}) \ .
\]
In particular, $\randv{\state}{n}$ is $\filtr{n}$-measurable for every $n\in\N$.
Notice that $\randv{\state}{0} = \bel{b}(\state)$.
This sequence of random variables is a \emph{martingale}~\cite[Section~4.2]{Durrett19} for all possible initial states~$\state$: it is the Doob martingale (also called L\'evy martingale) $\expect{\bel{b}}{\strategy}{\ind{\cyl{\state}}\mid\filtr{n}}$ associated with the random variable $\ind{\cyl{\state}}$ and the filtration $(\filtr{n})_{n\in\N}$, see, e.g.,~\cite[Chapter~14]{Wil91} or~\cite[Theorem~4.6.8]{Durrett19}.
Observe that this value can be computed inductively using Bayes' rule, and depends on the initial belief $\bel{b}$ but not on the strategy $\strategy$ (as long as $\strategy$ takes the actions along the history with non-zero probability).

\begin{lemma}\label{lem:is mart}
For every $\state\in\supp{\bel{b}}$, the sequence $(\randv{\state}{n})_{n\in\N}$ defines a martingale with respect to the filtration $(\filtr{n})_{n\in\N}$.
\end{lemma}
\begin{proof}
    Let $\state\in\supp{\bel{b}}$.
    First of all, for every $n\in\N$, $\randv{\state}{n}$ is $\filtr{n}$-measurable by definition.
    Moreover, for every $n\in\N$, it holds that $0 \leq \randv{\state}{n} \leq 1$.
    Hence, $\expect{\bel{b}}{\strategy}{|\randv{\state}{n}|} \leq 1 < \infty$ for every $n\in\N$.

    To show that $(\randv{\state}{n})_{n\in\N}$ is a martingale with respect to the filtration $(\filtr{n})_{n\in\N}$,
    it remains to show that for every $n\in\N$,
    $
    \expect{\bel{b}}{\strategy}{\randv{\state}{n+1} \mid \filtr{n}} =
    \randv{\state}{n}
    $.

    By definition, for every $n\in\N$,
    \[
    \randv{\state}{n} = \pdist{\bel{b}}{\strategy}(\cyl{\state} \mid \filtr{n}) = \expect{\bel{b}}{\strategy}{\ind{\cyl{\state}} \mid \filtr{n}} \ .
    \]
    Hence, for every $n\in\N$,
    \begin{align*}
    \expect{\bel{b}}{\strategy}{\randv{\state}{n+1} \mid \filtr{n}}
    &= \expect{\bel{b}}{\strategy}{\expect{\bel{b}}{\strategy}{\ind{\cyl{\state}} \mid \filtr{n+1}} \mid \filtr{n}} \\
    &= \expect{\bel{b}}{\strategy}{\ind{\cyl{\state}} \mid \filtr{n}} \text{ by the law of iterated expectations}\\
    &= \randv{\state}{n} \ .
    \qedhere
    \end{align*}
\end{proof}

The following is an obvious but useful corollary of \Cref{lem:is mart}.
\begin{corollary}\label{cor:exp leq b(q)}
$\expect{\bel{b}}{\strategy}{\randv{\state}{n}} = \bel{b}(\state) \leq 1$ for all $\state\in\supp{\bel{b}}$ and $n\in\N$.
\end{corollary}
Using Doob's convergence theorem~\cite[Theorem 4.6.7]{Durrett19}, given that the martingales $(\randv{\state}{n})_{n\in\N}$ are bounded and thus uniformly integrable, we get the following; this is also a direct consequence of L\'evy's upward theorem~\cite[Theorem~4.6.8]{Durrett19},~\cite[Chapter~14]{Wil91}, which moreover identifies the limit as $\pdist{\bel{b}}{\strategy}(\cyl{\state}\mid\filtr{\infty})$, where $\filtr{\infty}$ is the \sigal{} generated by $\bigcup_n\filtr{n}$.
\begin{corollary}\label{cor:mart convg}
For every $\state\in\supp{\bel{b}}$,
there exists a random variable \emphdef{$\randv{\state}{\infty}$} such that the sequence $(\randv{\state}{n})_{n\in\N}$ converges almost surely and in $L^1$ to $\randv{\state}{\infty}$.
\end{corollary}

\subsection{Distinguishing \adsccs} \label{sec:distinguishing}
\Cref{thm:dis adscc} is our main result about distinguishing \adsccs: by staying in such an \adscc{} and choosing actions uniformly at random, the agent eventually learns, with probability arbitrarily close to~$1$, which equivalence class contains the current state, and the belief concentrates on this class.
\distSECs*
We fix for this section a \deterministic POMDP $\pdp = \pdpFull$, a distinguishing maximal \adscc $\sccFunc$, and a belief~$\bel{b}$ such that $\bsupp = \supp{\bel{b}}$ belongs to $\dom{\sccFunc}$.
The maximality of~$\sccFunc$ is not used in the proof, so \Cref{thm:dis adscc} holds for every distinguishing \adscc{} (and trivially for non-distinguishing ones, which have a single equivalence class); it is stated for maximal \adsccs{} because this is how it is used by the algorithm.
We first prove \Cref{thm:dis adscc} for the belief~$\bel{b}$; the case of sub-beliefs is deduced at the end of the section.

We first prove the ``$\le$'' inequality of \Cref{thm:dis adscc}, which follows directly from the definition of $\preach$.

\begin{lemma}\label{lem:dis adscc leq}
\begin{equation*}
\preach(\bel{b}) \le
\sum_{C \in \bsupp/{\indistinguishable}}
\preach(\res{\bel{b}}{C}) \ .
\end{equation*}
\end{lemma}
\begin{proof}
    By definition,
    $
    \preach(\bel{b}) =
    \sup_{\strategy} \pdist{\bel{b}}{\strategy}(\reach{\target})
    $.
    For a fixed strategy $\strategy$, we have
    \[
        \pdist{\bel{b}}{\strategy}(\reach{\target}) =
        \sum_{C \in \bsupp/{\indistinguishable}}
        \pdist{\res{\bel{b}}{C}}{\strategy}(\reach{\target}) \ .
    \]
    We conclude that
    \begin{align*}
        \preach(\bel{b})
        &= \sup_{\strategy}
             \sum_{C \in \bsupp/{\indistinguishable}}
             \pdist{\res{\bel{b}}{C}}{\strategy}(\reach{\target})\\
        &\le \sum_{C \in \bsupp/{\indistinguishable}}
             \sup_{\strategy}
             \pdist{\res{\bel{b}}{C}}{\strategy}(\reach{\target}) \\
        &= \sum_{C \in \bsupp/{\indistinguishable}}
           \preach(\res{\bel{b}}{C}) \ ,
    \end{align*}
    where the inequality follows from the fact that the supremum of a sum is at most the sum of the suprema, since the strategies maximizing each term may differ.
\end{proof}

The rest of \Cref{sec:distinguishing} is dedicated to the more involved proof of the ``$\ge$'' inequality of \Cref{thm:dis adscc}, which will be stated in \Cref{lem:dis adscc geq}.
We will study the properties of the martingale $(\randv{\state}{n})_{n\in\N}$ (from Section~\ref{subsec:aPosteriori}) for \deterministic POMDPs, for the initial belief~$\bel{b}$ and the strategy $\unif{\sccFunc}$; we write $\mathbb{P}$ for $\pdist{\bel{b}}{\unif{\sccFunc}}$.
Recall that every history inside $\sccFunc$ from~$\bel{b}$ has positive probability, and that almost every run is inside $\sccFunc$ from~$\bel{b}$.
Since~$\pdp$ is \deterministic, for every $\state\in\bsupp = \supp{\bel{b}}$, every $n\in\N$, and every observable history $\history$ of length~$n$ inside $\sccFunc$ from~$\bel{b}$, we have
\begin{equation} \label{eq:dist randv}
    \randv{\state}{n}(\run) =
    \transitionIter(\bel{b},\history)(\transitionIter(\state,\history))
    \quad
    \text{for all runs $\run\in\cyl{\history}$}\ .
\end{equation}
Indeed, in a \deterministic POMDP, there is only one possible state $\transitionIter(\state,\history)$ after observing $\history$ starting from~$\state$.
Moreover, inside an \adscc, the action of $\history$ on $\supp{\bel{b}}$ is injective (\Cref{lem:sec basic}).
Therefore, we just have to compute the probability of that state in the belief $\transitionIter(\bel{b},\history)$ obtained after observing $\history$ starting from~$\bel{b}$.
Formally, write $\history = \action_1\obs_1\cdots\action_n\obs_n$ and $\history_{<k} = \action_1\obs_1\cdots\action_{k-1}\obs_{k-1}$.
Since $\pdist{\bel{b}}{\strategy} = \sum_{\state'}\bel{b}(\state')\cdot\pdist{\dirac{\state'}}{\strategy}$, and since the belief obtained from the Dirac belief $\dirac{\state}$ after $\history_{<k}$ is $\dirac{\transitionIter(\state,\history_{<k})}$, \eqref{eq:prob obs hist} yields
\[
\randv{\state}{n}(\run) = \frac{\bel{b}(\state)\cdot\pdist{\dirac{\state}}{\strategy}(\cyl{\history})}{\pdist{\bel{b}}{\strategy}(\cyl{\history})}
= \bel{b}(\state)\cdot\prod_{k=1}^{n}\frac{\ptran(\obs_k\mid\transitionIter(\state,\history_{<k}),\action_k)}{\ptran(\obs_k\mid\bel{b}_{\history_{<k}},\action_k)}\ ,
\]
where the terms $\strategy(\history_{<k})(\action_k)$ cancel out; by \Cref{lem:bel up aodscc} applied $n$ times, the right-hand side is equal to $\bel{b}_\history(\transitionIter(\state,\history))$.
We deduce the following.
\begin{lemma}\label{lem:randv non-zero}
    For all $\state\in\supp{\bel{b}}$, for all $n$, for all runs~$\run$ inside~$\sccFunc$ from~$\bel{b}$, $\randv{\state}{n}(\run) \neq 0$.
\end{lemma}
Indeed, by~\eqref{eq:dist randv}, $\randv{\state}{n}(\run) = \bel{b}_{\partialRun{n}}(\transitionIter(\state,\partialRun{n}))$, and, by \Cref{lem:possible hist}, the state $\transitionIter(\state,\partialRun{n})$ belongs to $\transitionIter(\bsupp,\partialRun{n}) = \supp{\bel{b}_{\partialRun{n}}}$.
\begin{lemma}\label{lem:ratio mart}
For every $\state,\state'\in\supp{\bel{b}}$, $n\in\N$ and run $\run = \state_0\action_1\obs_1\state_1\action_2\obs_2\ldots$ inside $\sccFunc$ from~$\bel{b}$,
\[
\frac{\randv{\state}{n+1}(\run)}{\randv{\state'}{n+1}(\run)}
=
\frac{\randv{\state}{n}(\run)}{\randv{\state'}{n}(\run)}
\cdot
\frac{\pgiven{\obs_{n+1}}{\transitionIter(\state,\partialRun{n}),\action_{n+1}}}{\pgiven{\obs_{n+1}}{\transitionIter(\state',\partialRun{n}),\action_{n+1}}} \ ,
\]
where all the quantities involved are positive.
\end{lemma}
\begin{proof}%
Let $\state, \state'\in\supp{\bel{b}}$ and $\run$ be a run from $\bel{b}$ inside $\sccFunc$.
Let $\history = \partialRun{n}$ and $(\action,\obs)$ $= (\action_{n+1},\obs_{n+1})$ be the pair following $\history$ in $\run$.
Let $\bel{b}_\history = \transitionIter(\bel{b},\history)$, $\state_\history = \transitionIter(\state,\history)$, and $\state'_\history = \transitionIter(\state',\history)$.
Since $\run$ is inside $\sccFunc$, $\history$ is also inside $\sccFunc$ and $\action\in \sccFunc(\supp{\bel{b}_\history})$.
Moreover, $\history\action\obs$ is inside $\sccFunc$ from $\bel{b}$, hence possible from~$\bel{b}$, so $\ptran(\obs\mid\bel{b}_\history,\action)>0$ by \Cref{lem:possible hist}, and $\ptran(\obs\mid\state_\history,\action)>0$ and $\ptran(\obs\mid\state'_\history,\action)>0$ by \Cref{lem:sec basic}.
Then, $\randv{\state}{n}(\run) = \bel{b}_\history(\state_\history)$ and $\randv{\state'}{n}(\run) = \bel{b}_\history(\state'_\history)$ by~\eqref{eq:dist randv}.
Also,
\begin{align*}
\randv{\state}{n+1}(\run)
\quad &=\quad 
\transition(\bel{b}_\history, \action, \obs)(\transition(\state_\history, \action, \obs)) & \text{by~\eqref{eq:dist randv}}\\
\quad &=\quad
\frac{\ptran(\obs\mid \state_\history,\action)\cdot \bel{b}_\history(\state_\history)}{\ptran(\obs\mid \bel{b}_\history,\action)}
& \text{(\Cref{lem:bel up aodscc})}\\
\quad &=\quad
\frac{\ptran(\obs\mid \state_\history,\action)\cdot \randv{\state}{n}(\run)}{\ptran(\obs\mid \bel{b}_\history,\action)} & \text{by~\eqref{eq:dist randv}}.
\end{align*}
Similarly
\[
\randv{\state'}{n+1}(\run)
\quad =\quad
\frac{\ptran(\obs\mid \state'_\history,\action)\cdot \randv{\state'}{n}(\run)}{\ptran(\obs\mid \bel{b}_\history,\action)} \ .
\]
The equality in the statement of the lemma follows from the two above equalities and \Cref{lem:randv non-zero}.
\end{proof}
\begin{corollary}\label{cor:same ratio}
    Let $\state\indistinguishable \state'$ be two non-distinguishable states in $\bsupp = \supp{\bel{b}}$.
    Let $\run$ be a run inside $\sccFunc$ from~$\bel{b}$ for which $(\randv{\state}{n}(\run))_{n\in\N}$ converges to $\randv{\state}{\infty}(\run)$ and $(\randv{\state'}{n}(\run))_{n\in\N}$ converges to $\randv{\state'}{\infty}(\run)$.
    Then, $\bel{b}(\state)\cdot\randv{\state'}{\infty}(\run) = \bel{b}(\state')\cdot\randv{\state}{\infty}(\run)$.   
\end{corollary}
\begin{proof}%
    Let $r = \frac{\bel{b}(\state')}{\bel{b}(\state)}$.
    Recall that $\randv{\state}{0}(\run) = \bel{b}(\state)$ and $\randv{\state'}{0}(\run) = \bel{b}(\state')$.
    Thus, using \Cref{lem:ratio mart} and the fact that $\state$ and $\state'$ are not $(\sccFunc, \supp{\bel{b}})$-distinguishable,
    we get that
    \[
    \randv{\state'}{n}(\run) = r\cdot\randv{\state}{n}(\run)
    \]
    for every $n\in\N$.
    Since $(\randv{\state}{n}(\run))_{n\in\N}$ converges to $\randv{\state}{\infty}(\run)$ (\Cref{cor:mart convg}), this implies
    \[
    \randv{\state'}{\infty}(\run) = r\cdot\randv{\state}{\infty}(\run)\ .
    \qedhere
    \]
\end{proof}
\begin{lemma}\label{lem:dist => zero}
    Consider two distinguishable states $\state\distinguishable \state'$ in the support $\bsupp = \supp{\bel{b}}$.
    Let $\run$ be a run inside $f$ from $\bel{b}$ such that $(\randv{\state}{n}(\run))_{n\in\N}$ and $(\randv{\state'}{n}(\run))_{n\in\N}$ converge to $\randv{\state}{\infty}(\run)$ and $\randv{\state'}{\infty}(\run)$ (which holds for almost every run by \Cref{cor:mart convg}).
    If $\run$ distinguishes $\state$ and $\state'$ infinitely often,
    then at least one of the values $\randv{\state}{\infty}(\run)$ and $\randv{\state'}{\infty}(\run)$ is equal to $0$.
\end{lemma}
\begin{proof}%
    Let $\run = \action_1\obs_1\action_2\obs_2\ldots$\
    We prove the contrapositive.
    Assume $\randv{\state}{\infty}(\run)$ and $\randv{\state'}{\infty}(\run)$ are both non-zero; we will show that $\run$ distinguishes $\state$ and $\state'$ only finitely many times.

    Let
    \[
    r = \frac{\randv{\state}{\infty}(\run)}{\randv{\state'}{\infty}(\run)}
    \quad\text{and, for $n\in\N$,}\quad
    r_n = \frac{\randv{\state}{n}(\run)}{\randv{\state'}{n}(\run)}
    \ .
    \]
    Note that $r_n$ is well-defined due to \Cref{lem:randv non-zero}, and $r\in(0,\infty)$.
    By assumption, %
    the sequence $(r_n)_{n\in\N}$ converges to~$r$.
    Let
    \[
    w_n = \frac{\pgiven{\obs_n}{\transitionIter(\state,\partialRun{n-1}),\action_n}}{\pgiven{\obs_n}{\transitionIter(\state',\partialRun{n-1}),\action_n}}
    \text{ for $n\ge 1$}\ .
    \]
    Then, \Cref{lem:ratio mart} implies $\frac{r_{n}}{r_{n-1}} = w_n$ for every $n\ge 1$.
    Since $(r_n)_{n\in\N}$ converges to $r\in(0,\infty)$, we get that the sequence $(w_n)_{n\ge 1}$ converges to~$1$.
    Since the sets $\states$, $\actions$, and $\sig$ are finite,
    so is the set $\setof{w_n}{n\ge 1}$.
    Since $(w_n)_{n\ge 1}$ converges to $1$, this implies
    $w_n = 1$ for all but finitely many $n\ge 1$.
    Observe that $w_n \neq 1$ if and only if $\partialRun{n}$ distinguishes $\state$ and $\state'$ (with $\history' = \partialRun{n-1}$ in the definition of distinguishing histories).
    Hence, $\run$ distinguishes $\state$ and $\state'$ only finitely many times.
\end{proof}
\begin{lemma}\label{lem:small dist hist}
    There exist $L\in\N$ and $\smallProb > 0$ such that, for all beliefs $\bel{c}$ such that $\supp{\bel{c}} \in \dom{\sccFunc}$, for all $\state\distinguishableSupp{\supp{\bel{c}}} \state'\in \supp{\bel{c}}$, there exists an observable history $\history$ inside $\sccFunc$ from~$\bel{c}$ of length at most $L$ that distinguishes $\state$ and $\state'$ and such that
    \[
        \pdist{\bel{c}}{\unif{\sccFunc}}(\cyl{\history}) \geq \smallProb \ .
    \]
\end{lemma}
\begin{proof}
    By definition of distinguishability, for every $\state\distinguishableSupp{\bsupp'} \state'\in \bsupp'\in \dom{\ascc}$, there exists a history $\history_{\state,\state'}$ inside $\sccFunc$ from~$\bsupp'$ that distinguishes $\state$ and $\state'$; note that this only depends on the belief support~$\bsupp'$, not on a belief with this support.
    Since there are finitely many pairs of states in the finitely many belief supports of $\dom{\ascc}$, we can take $L$ to be the maximum length of all histories $\history_{\state,\state'}$.

    Recall that $\pmin$ is the smallest non-zero probability occurring in the syntactic description of~$\pdp$ (\Cref{sec:correctness}).
    We take $\smallProb = \left(\frac{\pmin}{|\actions|}\right)^{L}$.

    Let $\bel{c}$ be a belief such that $\supp{\bel{c}} \in \dom{\sccFunc}$, and let $\state\distinguishableSupp{\supp{\bel{c}}} \state'$ be two states in $\supp{\bel{c}}$.
    We take $\history = (a_1,\obs_1)(a_2,\obs_2)\ldots (a_k,\obs_k)$ to be any history inside $\sccFunc$ from $\bel{c}$ of length at most $L$ that distinguishes $\state$ and $\state'$.
    We give a lower bound on $\pdist{\bel{c}}{\unif{\sccFunc}}(\cyl{\history})$ by considering every step of $\history$.
    When an action is chosen at step~$i$, it has probability at least $\frac{1}{|\actions|}$ to be $\action_i$ since $\unif{\sccFunc}$ chooses actions uniformly from a subset of $\actions$.
    The probability to observe $\obs_i$ at step~$i$ given that action $\action_i$ was chosen is then given by $\ptran(\obs_i \mid \bel{c}_{i-1}, \action_i)$ where $\bel{c}_{i-1}$ is the belief after observing the prefix of $\history$ of length $i-1$ (in particular, $\bel{c}_0 = \bel{c}$).
    Thus, by~\eqref{eq:prob obs hist} (here $\bel{c}$ is a belief, so that $\pdist{\bel{c}}{\unif{\sccFunc}}$ is a probability measure),
    \[
    \pdist{\bel{c}}{\unif{\sccFunc}}(\cyl{\history})
    \quad \ge\quad
    \frac{1}{|\actions|^k}\cdot\ptran(\obs_1 \mid \bel{c}_0, a_1)\cdot\ptran(\obs_2 \mid \bel{c}_1, a_2)\cdot \ldots\cdot \ptran(\obs_k \mid \bel{c}_{k-1}, a_k) \ .
    \]
    Since $\history$ is inside $\sccFunc$ from $\bel{c}$, \Cref{lem:sec basic} shows that for all $1\le i \le k$, for all $\state''\in\supp{\bel{c}_{i-1}}$, $\ptran(\obs_i \mid \state'', a_i) > 0$, and thus $\ptran(\obs_i \mid \state'', a_i) \geq \pmin$.
    Hence,
    \begin{align*}
    \ptran(\obs_i \mid \bel{c}_{i-1}, a_i)
    \quad &=\quad
    \sum_{\state''\in\supp{\bel{c}_{i-1}}}
    \bel{c}_{i-1}(\state'')\cdot \ptran(\obs_i \mid \state'', a_i)\\
    \quad &\ge\quad
    \pmin \cdot \sum_{\state''\in\supp{\bel{c}_{i-1}}} \bel{c}_{i-1}(\state'')\\
    \quad &=\quad
    \pmin \ ,
    \end{align*}
    where the last equality holds because $\bel{c}_{i-1}$ is a belief.
    Therefore,
    \[
    \pdist{\bel{c}}{\unif{\sccFunc}}(\cyl{\history})
    \quad \ge\quad
    \frac{1}{|\actions|^k}\cdot (\pmin)^k
    \quad \ge\quad
    \left(\frac{\pmin}{|\actions|}\right)^L
    \quad =\quad
    \smallProb \ .
    \qedhere
    \]
\end{proof}
\begin{lemma}\label{lem:all compl dist}
With respect to the probability measure $\pdist{\bel{b}}{\unif{\sccFunc}}$, for all pairs $\state\distinguishable \state'$,
almost all observable runs distinguish $\state$ and $\state'$ infinitely often.
\end{lemma}
\begin{proof}
We write $\mathbb{P}$ for $\pdist{\bel{b}}{\unif{\sccFunc}}$.
Let $\state, \state'\in\bsupp$ such that $\state\distinguishable \state'$, and let $D$ be the set of runs inside $\sccFunc$ from $\bel{b}$ that distinguish $\state$ and $\state'$ infinitely often.
We show that $\mathbb{P}(D) = 1$.

Let $L\in\N$, $\smallProb > 0$ be the numbers we get from \Cref{lem:small dist hist}.
Let $R$ be the set of runs inside $\sccFunc$ from~$\bel{b}$; recall that $\mathbb{P}(R) = 1$.
For $k\in\N$, let $E_k$ be the set of runs $\run\in R$ such that $\partialRun{m}$ distinguishes $\state$ and $\state'$ for some $m$ with $kL < m \le (k+1)L$, and let $E^c_k = R\setminus E_k$.
Then $D = \bigcap_{j\in\N}\bigcup_{k\ge j}E_k$, so it suffices to show that $\mathbb{P}(\bigcap_{k\ge j}E^c_k) = 0$ for every $j\in\N$.

We first show that, for every $k\in\N$ and every history $\history$ of length $kL$ inside $\sccFunc$ from~$\bel{b}$,
\begin{equation}\label{eq:cond dist}
\mathbb{P}(E_k\cap\cyl{\history}) \ge \smallProb\cdot\mathbb{P}(\cyl{\history})\ .
\end{equation}
Let $\bsupp_\history = \transitionIter(\bsupp,\history)\in\dom{\sccFunc}$.
By \Cref{lem:dist preserve}, applied along~$\history$, the states $\transitionIter(\state,\history)$ and $\transitionIter(\state',\history)$ are $(\sccFunc,\bsupp_\history)$-distinguishable.
By \Cref{lem:small dist hist} applied to the belief $\bel{b}_\history$, there is a history $z_\history$ inside $\sccFunc$ from $\bel{b}_\history$, of length at most~$L$ (and at least~$1$), that distinguishes these two states and satisfies $\pdist{\bel{b}_\history}{\unif{\sccFunc}}(\cyl{z_\history})\ge\smallProb$.
By \Cref{lem:sec basic}, $\history z_\history$ is inside $\sccFunc$ from~$\bel{b}$ and distinguishes $\state$ and $\state'$, so $\cyl{\history z_\history}\cap R\subseteq E_k$.
Moreover, $\mathbb{P}(\cyl{\history z_\history}) = \mathbb{P}(\cyl{\history})\cdot\pdist{\bel{b}_\history}{\unif{\sccFunc}}(\cyl{z_\history})$ by~\eqref{eq:shift} (see the remarks after the definition of $\unif{\sccFunc}$), which proves~\eqref{eq:cond dist}.

Now let $j\le m$.
Whether a run of~$R$ belongs to $E_j,\dots,E_{m-1}$ only depends on its prefix of length~$mL$; hence the set $\bigcap_{k=j}^{m-1}E^c_k$ is the disjoint union of the sets $\cyl{\history}\cap R$ for $\history$ ranging over a set $X_m$ of histories of length $mL$ inside $\sccFunc$ from~$\bel{b}$.
By~\eqref{eq:cond dist}, we obtain
\[
\mathbb{P}\left(\bigcap_{k=j}^{m}E^c_k\right) = \sum_{\history\in X_m}\mathbb{P}(\cyl{\history}\cap E^c_m) \le (1-\smallProb)\cdot\sum_{\history\in X_m}\mathbb{P}(\cyl{\history}) = (1-\smallProb)\cdot\mathbb{P}\left(\bigcap_{k=j}^{m-1}E^c_k\right)\ .
\]
By induction on~$m$, $\mathbb{P}(\bigcap_{k=j}^{m}E^c_k)\le(1-\smallProb)^{m-j+1}$, hence $\mathbb{P}(\bigcap_{k\ge j}E^c_k) = \lim_{m\to\infty}\mathbb{P}(\bigcap_{k=j}^{m}E^c_k) = 0$, as required.
\end{proof}

For an equivalence class $C$ of $\indistinguishable$ and $n\in\N$,
let \emphdef{$\randv{C}{n}$} denote the sum of the random variables
$
\sum_{\state\in C} \randv{\state}{n}
$.%
\begin{corollary}\label{cor:0 1 a.e.}
    With respect to the probability measure $\pdist{\bel{b}}{\unif{\sccFunc}}$, 
    for almost all runs $\run$ inside $\sccFunc$ from $\bel{b}$,
    $\randv{C}{\infty}(\run) = 1$ for exactly one equivalence class $C$ of $\indistinguishable$,
    and for the remaining equivalence classes $D$ of $\indistinguishable$, $\randv{D}{\infty}(\run) = 0$. 
\end{corollary}
\begin{proof} %
    Let $\run$ be a run inside $\sccFunc$ from $\bel{b}$ that distinguishes every pair of distinguishable states (for $\indistinguishable$) infinitely often.
    Such runs happen almost surely under $\pdist{\bel{b}}{\unif{\sccFunc}}$ by \Cref{lem:all compl dist}.

    First, observe that for all $n\in\N$ and runs $\run$ inside $\sccFunc$ from $\bel{b}$, $\sum_{\state\in\bsupp} \randv{\state}{n}(\run) = 1$: by~\eqref{eq:dist randv}, this sum is the mass of the belief $\bel{b}_{\partialRun{n}}$, which is~$1$ since $\bel{b}$ is a belief.
    By \Cref{cor:mart convg}, this implies that, almost surely, 
    \begin{equation} \label{eq:sum randv inf 1}
        \sum_{\state\in\bsupp} \randv{\state}{\infty} = 1\ .
    \end{equation}

    Hence, for almost all observable runs $\run$, we have that $\randv{\state}{\infty}(\run)$ is non-zero for some $\state\in\bsupp$.
    \Cref{lem:dist => zero} implies $\randv{\state'}{\infty}(\run) = 0$ for every $\state'\distinguishable \state$, and \Cref{cor:same ratio} implies $\randv{\state'}{\infty}(\run) \neq 0$ for every $\state' \indistinguishable \state$.
    Now, applying~\eqref{eq:sum randv inf 1}, we get that $\randv{[\state]_{\indistinguishable}}{\infty}(\run) = 1$ and $\randv{[\state']_{\indistinguishable}}{\infty}(\run) = 0$ for every $\state'\distinguishable \state$.
\end{proof}
Define a sequence of random variables \emphdef{$(\valrand{n})_{n\in\N}$} adapted to the filtration $(\filtr{n})_{n\in\N}$ as:
\[
\valrand{n}(\run) = \preach(\transitionIter(\bel{b},\partialRun{n}))
\qquad
\text{for $\run$ a run inside $\sccFunc$ from $\bel{b}$,}
\]
and $\valrand{n}(\run) = 0$ for the other runs (which form a set of probability~$0$).

\begin{lemma}
    The sequence $(\valrand{n})_{n\in\N}$ defines a supermartingale.
    Moreover,
    $
    \expect{\bel{b}}{\unif{\ascc}}{\valrand{n}} \leq 1
    $
    for all $n$.
\end{lemma}
\begin{proof}
To show that $(\valrand{n})_{n \in \mathbb{N}}$ is a supermartingale, we must verify that $\expect{\bel{b}}{\unif{\ascc}}{\valrand{n+1} \mid \filtr{n}} \leq \valrand{n}$.

Let $\run$ be a run inside $\sccFunc$ from $\bel{b}$.
Let $\bel{b}_n = \transitionIter(\bel{b},\partialRun{n})$ be the belief at step $n$ along $\run$.
By definition of $\valrand{n}$, we have $\valrand{n}(\run) = \preach(\bel{b}_n)$.
Given that the filtration~$\filtr{n}$ fixes the history up to step $n$, and that the strategy is $\unif{\ascc}$, the expectation of the next value is
\[
\expect{\bel{b}}{\unif{\ascc}}{\valrand{n+1} \mid \filtr{n}}(\run) = \sum_{\action\in\sccFunc(\supp{\bel{b}_n})} \frac{1}{|\sccFunc(\supp{\bel{b}_n})|}\cdot \preach(\bel{b}_n, \action) \ .
\]
Given $\preach(\bel{b}_n) = \max_{\action\in\actions} \preach(\bel{b}_n, \action)$ by~\eqref{eq:val bel upd}, we have
\[
\expect{\bel{b}}{\unif{\ascc}}{\valrand{n+1} \mid \filtr{n}}(\run)
\le \preach(\bel{b}_n)\cdot \sum_{\action\in\sccFunc(\supp{\bel{b}_n})} \frac{1}{|\sccFunc(\supp{\bel{b}_n})|} = \preach(\bel{b}_n) = \valrand{n}(\run) \ .
\]

For the second part of the lemma, simply observe that the value function $\preach$ is bounded above by $1$ for any belief $\bel{b}$, since it represents a probability. Therefore, for all $n$, we have $\expect{\bel{b}}{\unif{\ascc}}{\valrand{n}} \le 1$.
\end{proof}

By Doob's convergence theorem~\cite[Theorem 4.6.4]{Durrett19}, applied to the submartingale $(-\valrand{n})_{n\in\N}$, using that the random variables $\valrand{n}$ are bounded and thus uniformly integrable, we obtain the following.
\begin{lemma}
    The sequence $\valrand{n}$ converges almost surely and in $L^1$ to some integrable random variable \emphdef{$\valrand{\infty}$}.
\end{lemma}

Recall that, for a non-zero sub-belief $\bel{b}$,
we use $\normalised{\bel{b}}$ to denote the normalised belief $\frac{1}{\oneNorm{\bel{b}}}\cdot\bel{b}$.
\begin{lemma}\label{lem:lim val}
    Let $\run$ be a run inside $\sccFunc$ from $\bel{b}$ and $C$ be an equivalence class of $\indistinguishable$.
    If $\randv{C}{n}(\run)$ converges to $\randv{C}{\infty}(\run)$ and $\randv{C}{\infty}(\run) = 1$,
    then $\valrand{\infty}(\run) = \preach(\normalised{\res{\bel{b}}{C}})$.
\end{lemma}
\begin{proof}
    We use \Cref{lem:subbelief val}, which holds for all sub-beliefs; in particular, $\preach$ is $1$-Lipschitz with respect to the $1$-norm:
    \begin{equation} \label{eq:lip preach}
        |\preach(\bel{c}) - \preach(\bel{c}')|
        \le
        \oneNorm{\bel{c} - \bel{c}'} \quad\text{for all sub-beliefs $\bel{c},\bel{c}'$}\ .
    \end{equation}

    Assume $\randv{C}{n}(\run)$ converges to $\randv{C}{\infty}(\run)$ and $\randv{C}{\infty}(\run) = 1$.
    Let $\state\in C$, so that $C = [\state]_{\indistinguishable}$.
    For every $n\in\N$, let $\bel{b}_n = \transitionIter(\bel{b},\partialRun{n})$ and
    $C_n = \transitionIter(C,\partialRun{n})$.
    Note that $\valrand{n}(\run) = \preach(\bel{b}_n)$, and that $\bel{b}_n$ is a belief.
    By \Cref{lem:sec basic}, the map $\state'\mapsto\transitionIter(\state',\partialRun{n})$ is a bijection from $\bsupp$ onto $\supp{\bel{b}_n}$ mapping $C$ onto~$C_n$, so, by~\eqref{eq:dist randv},
    \[
    \randv{C}{n}(\run) = \sum_{\state'\in C}\bel{b}_n(\transitionIter(\state',\partialRun{n})) = \oneNorm{\res{\bel{b}_n}{C_n}}\ .
    \]
    For every $\eta > 0$, there exists $k_{\eta}\in\N$ such that $\randv{C}{n}(\run) > 1 - \eta$ for every $n\geq k_{\eta}$, which implies, since $\oneNorm{\bel{b}_n} = 1$,
    \begin{equation} \label{eq:epsilons}
        \oneNorm{\bel{b}_n - \res{\bel{b}_n}{C_n}} < \eta \quad\text{and}\quad
        \oneNorm{\res{\bel{b}_n}{C_n}} > 1 - \eta
    \end{equation}
    for every $n\geq k_{\eta}$.

    We now show that $\normalised{\res{\bel{b}_n}{C_n}} = \transitionIter(\normalised{\res{\bel{b}}{C}},\partialRun{n})$ for every $n$.
    Write $\partialRun{n} = \action_1\obs_1\cdots\action_n\obs_n$ and $\run_{<k} = \partialRun{k-1}$.
    As in the justification of~\eqref{eq:dist randv}, applying \Cref{lem:bel up aodscc} $n$ times yields, for every $\state'\in C$,
    \[
    \bel{b}_n(\transitionIter(\state',\partialRun{n})) = \bel{b}(\state')\cdot\prod_{k=1}^{n}\frac{\ptran(\obs_k\mid\transitionIter(\state',\run_{<k}),\action_k)}{\ptran(\obs_k\mid\bel{b}_{k-1},\action_k)}\ .
    \]
    Since the states of~$C$ are pairwise $(\sccFunc,\bsupp)$-indistinguishable and $\run$ is inside $\sccFunc$ from~$\bsupp$, the numerator $\ptran(\obs_k\mid\transitionIter(\state',\run_{<k}),\action_k)$ does not depend on $\state'\in C$ (this is the argument of \Cref{cor:same ratio}, for finite prefixes).
    Hence the restriction $\res{\bel{b}_n}{C_n}$ is proportional to the renaming of $\res{\bel{b}}{C}$ by the bijection $\state'\mapsto\transitionIter(\state',\partialRun{n})$, and $\normalised{\res{\bel{b}_n}{C_n}}$ is the renaming of $\normalised{\res{\bel{b}}{C}}$ by this bijection.
    On the other hand, by the last statement of \Cref{lem:ndist form scc}, the history $\partialRun{n}$ is inside the non-distinguishing \adscc{} $\sccFunc_{\state,\bsupp}$ from~$C$, so, by \Cref{lem:non-dist reach}, $\transitionIter(\normalised{\res{\bel{b}}{C}},\partialRun{n})$ is the same renaming of $\normalised{\res{\bel{b}}{C}}$.
    This proves the claimed equality, and shows that $\normalised{\res{\bel{b}_n}{C_n}}\in\reach{\sccFunc_{\state,\bsupp}}(\normalised{\res{\bel{b}}{C}})$.
    By \Cref{lem:val in non-dist} applied to $\sccFunc_{\state,\bsupp}$ (which is non-distinguishing by \Cref{lem:ndist form scc}), we obtain
    \[
    \preach(\normalised{\res{\bel{b}_n}{C_n}}) =
    \preach(\normalised{\res{\bel{b}}{C}}) \ .
    \]
    Now, \Cref{lem:subbelief val}, \eqref{eq:lip preach}, and~\eqref{eq:epsilons} imply
    \[
     \preach(\res{\bel{b}_n}{C_n}) \leq \preach(\bel{b}_n) = \valrand{n}(\run) \leq
     \preach(\res{\bel{b}_n}{C_n}) + \eta
    \]
    when $n\geq k_{\eta}$.
    Also, for $n\geq k_{\eta}$, by \Cref{lem:subbelief val}\eqref{item:val homogeneous},
    \[
    \preach(\normalised{\res{\bel{b}}{C}})
    =
    \preach(\normalised{\res{\bel{b}_n}{C_n}})
    \geq
    \preach(\res{\bel{b}_n}{C_n})
    =
    \oneNorm{\res{\bel{b}_n}{C_n}}\cdot
    \preach(\normalised{\res{\bel{b}_n}{C_n}})
    \geq (1 - \eta)\cdot
    \preach(\normalised{\res{\bel{b}}{C}})\ .
    \]
    Hence, for $n\geq k_{\eta}$,
    \[
    (1 - \eta)\cdot
     \preach(\normalised{\res{\bel{b}}{C}})
     \leq \valrand{n}(\run) \leq
     \preach(\normalised{\res{\bel{b}}{C}}) + \eta \ .
    \]
    Since $(\valrand{n}(\run))_{n\in\N}$ converges to $\valrand{\infty}(\run)$,
    the above inequalities imply
    \[
    \valrand{\infty}(\run) = \preach(\normalised{\res{\bel{b}}{C}}) \ .
    \qedhere
    \]
\end{proof}
\begin{lemma}\label{lem:dis adscc geq}
\[
\preach(\bel{b})
\quad\geq\quad
\sum_{C \in \bsupp/{\indistinguishable}} \preach(\res{\bel{b}}{C}) \ .
\]
\end{lemma}
\begin{proof}
    We write $\expectSolo[\cdot]$ for $\expect{\bel{b}}{\unif{f}}{\cdot}$ for brevity.
    Recall from \Cref{cor:exp leq b(q)} that $\expectSolo[\randv{\state}{n}] = \bel{b}(\state)$ for all $\state\in\supp{\bel{b}}$ and $n\in\N$.
    Hence, $\expectSolo[\randv{C}{n}] = \bel{b}(C) = \oneNorm{\res{\bel{b}}{C}}$ for every equivalence class $C$ of $\indistinguishable$ and $n\in\N$.
    By the triangle inequality, we have
    \[
    |\expectSolo[\randv{C}{n}] - \expectSolo[\randv{C}{\infty}]|
    \leq
    \expectSolo[|\randv{C}{n} - \randv{C}{\infty}|]\ .
    \]
    Since $(\randv{C}{n})_{n\in\N}$ converges to $\randv{C}{\infty}$ in $L^1$ (by \Cref{cor:mart convg} and the linearity of expectations),
    $(\expectSolo[|\randv{C}{n} - \randv{C}{\infty}|])_{n\in\N}$ converges to $0$.
    Therefore,
    $
    \expectSolo[\randv{C}{\infty}]
    =
    \oneNorm{\res{\bel{b}}{C}}
    $.

    \Cref{cor:0 1 a.e.,lem:lim val} together imply that, almost surely, $\valrand{\infty} = \sum_{C}\randv{C}{\infty}\cdot\preach(\normalised{\res{\bel{b}}{C}})$ (as $\randv{C}{\infty}\in\set{0,1}$ almost surely), hence
    \begin{align*}
    \expectSolo[\valrand{\infty}]
    &=
    \sum_{C \in \bsupp/{\indistinguishable}}
    \expectSolo[\randv{C}{\infty}]\cdot
    \preach(\normalised{\res{\bel{b}}{C}})\\
    &=
    \sum_{C \in \bsupp/{\indistinguishable}}
    \oneNorm{\res{\bel{b}}{C}}\cdot
    \preach(\normalised{\res{\bel{b}}{C}})\\
    &=
    \sum_{C \in \bsupp/{\indistinguishable}}
    \preach(\res{\bel{b}}{C})\ .
    \end{align*}
    Since $(\valrand{n})_{n\in\N}$ is a supermartingale that converges to $\valrand{\infty}$ in~$L^1$,
    \[
    \preach(\bel{b}) =
    \expectSolo[\valrand{0}] \geq
    \lim_{n \to \infty} \expectSolo[\valrand{n}] =
    \expectSolo[\valrand{\infty}] =
    \sum_{C \in \bsupp/{\indistinguishable}}
    \preach(\res{\bel{b}}{C}) \ .
    \qedhere
    \]
\end{proof}

Along with \Cref{lem:dis adscc leq}, this concludes the proof of \Cref{thm:dis adscc} for beliefs.
Finally, let $\bel{c}$ be a non-zero sub-belief such that $\supp{\bel{c}}\in\dom{\sccFunc}$.
The equivalence classes of $\equiv_{(\sccFunc,\supp{\bel{c}})}$ only depend on the support, so, applying the result to the belief $\normalised{\bel{c}}$ and using \Cref{lem:subbelief val},
\[
\preach(\bel{c}) = \oneNorm{\bel{c}}\cdot\preach(\normalised{\bel{c}}) = \sum_{C}\oneNorm{\bel{c}}\cdot\preach(\res{\normalised{\bel{c}}}{C}) = \sum_{C}\preach(\res{\bel{c}}{C})\ ,
\]
since $\oneNorm{\bel{c}}\cdot\res{\normalised{\bel{c}}}{C} = \res{\bel{c}}{C}$.
This proves \Cref{thm:dis adscc} for sub-beliefs.

\section{Conclusion}
We introduced posterior-deterministic POMDPs and showed that, for this class, the reachability value can be approximated up to any prescribed precision.
Our algorithm unfolds the belief MDP into a finite tree, and uses support end components to decide how each node is unfolded.
In this class, infinitely many beliefs may be reachable from the initial belief, so the problem does not reduce to the analysis of a finite MDP.

The complexity of the problem remains open: our algorithm runs in triply exponential time, and we do not know a matching lower bound.
Another direction is to extend the approach to other objectives, for instance $\omega$-regular objectives.

\bibliography{main.bib}
\bibliographystyle{alphaurl}

\end{document}